\documentclass[sigconf]{acmart}

\usepackage{algorithm}
\usepackage{algorithmic}
\usepackage{microtype}
\usepackage{graphicx}
\usepackage{subcaption}
\usepackage{caption}
\usepackage{tcolorbox,enumitem}

\usepackage{amsmath}
\usepackage{mathtools}
\usepackage{amsthm}
\theoremstyle{plain}
\newtheorem{theorem}{Theorem}

\newtheorem{lemma}{Lemma}
\newtheorem{corollary}{Corollary}
\newtheorem{assumption}{Assumption}
\theoremstyle{definition}
\newtheorem{definition}{Definition}
\theoremstyle{remark}
\newtheorem{remark}{Remark}
\usepackage{cleveref}

\AtBeginDocument{%
  \providecommand\BibTeX{{%
    \normalfont B\kern-0.5em{\scshape i\kern-0.25em b}\kern-0.8em\TeX}}}

\copyrightyear{2023}
\acmYear{2023}
\setcopyright{acmlicensed}
\acmConference[KDD '23]{Proceedings of the 29th ACM SIGKDD Conference on Knowledge Discovery and Data Mining}{August 6--10, 2023}{Long Beach, CA, USA}
\acmBooktitle{Proceedings of the 29th ACM SIGKDD Conference on Knowledge Discovery and Data Mining (KDD '23), August 6--10, 2023, Long Beach, CA, USA}
\acmPrice{15.00}
\acmISBN{979-8-4007-0103-0/23/08}
\acmDOI{10.1145/3580305.3599499}

\begin{document}

\title{
Serverless Federated AUPRC Optimization for Multi-Party Collaborative Imbalanced Data Mining
}

\author{Xidong Wu}
\affiliation{
\institution{Department of Electrical and Computer Engineering}
  \institution{University of Pittsburgh}
  \city{Pittsburgh}
  \state{Pennsylvania}
  \country{USA}
}
\email{xidong_wu@outlook.com}

\author{Zhengmian Hu}
\affiliation{
\institution{Department of Electrical and Computer Engineering}
  \institution{University of Pittsburgh}
  \city{Pittsburgh}
  \state{Pennsylvania}
  \country{USA}
}
\email{huzhengmian@gmail.com}

\author{Jian Pei}
\affiliation{
\institution{Department of Computer Science}
  \institution{Duke University}
  \city{Durham}
  \state{North Carolina}
  \country{USA}
}
\email{j.pei@duke.edu}

\author{Heng Huang}
\affiliation{
\institution{Department of Computer Science}
  \institution{University of Maryland}
  \city{College Park}
  \state{Maryland}
  \country{USA}
}
\email{henghuanghh@gmail.com}
\authornote{
This work was partially supported by NSF IIS 1838627, 1837956, 1956002, 2211492, CNS 2213701, CCF 2217003, DBI 2225775.
}







\renewcommand{\shortauthors}{Xidong Wu, Zhengmian Hu, Jian Pei, Heng Huang}

\begin{abstract}
To address the big data challenges, multi-party collaborative training, such as distributed learning and federated learning, has recently attracted  attention. However, traditional multi-party
collaborative training algorithms were mainly designed for balanced data mining tasks and are intended to optimize accuracy (\emph{e.g.}, cross-entropy). The data distribution in many real-world applications is skewed and classifiers, which are trained to improve accuracy, perform poorly when applied to imbalanced data tasks since models could be significantly biased toward the primary class. Therefore, the Area Under Precision-Recall Curve (AUPRC) was introduced as an effective metric. Although single-machine AUPRC maximization methods have been designed, multi-party collaborative algorithm has never been studied. The change from the single-machine to the multi-party setting poses critical challenges. For example, existing single-machine-based AUPRC maximization algorithms maintain an inner state for local each data point, thus these methods are not applicable to large-scale online multi-party collaborative training due to the dependence on each local data point. 

To address the above challenge, we study serverless multi-party collaborative AUPRC maximization problem since serverless multi-party collaborative training can cut down the communications cost by avoiding the
server node bottleneck, and reformulate it  as a conditional stochastic optimization problem in a serverless multi-party collaborative learning setting and propose a new ServerLess biAsed sTochastic gradiEnt (SLATE) algorithm to directly optimize the AUPRC. After that, we use the variance reduction technique and propose ServerLess biAsed sTochastic gradiEnt with Momentum-based variance reduction (SLATE-M) algorithm to improve the convergence rate, which matches the best theoretical convergence result reached by the single-machine online method. To the best of our knowledge, this is the first work to solve the multi-party collaborative AUPRC maximization problem. Finally, extensive experiments show the advantages of directly optimizing the AUPRC with distributed learning methods and also verify the efficiency of our new algorithms (\emph{i.e.}, SLATE and SLATE-M).
\end{abstract}





\maketitle

\section{Introduction}
Multi-party collaborative learning, such as distributed learning \cite{dean2012large, li2014scaling, bao2022doubly} (typically focus on IID data and train learning model using the gradients from different parties) and federated learning \cite{mcmahan2017communication} (focus on non-IID data and train model via periodically averaging model parameters from different parties coordinated by the server), have been actively studied at past decades to train large-scale deep learning models in a variety of real-world applications, such as computer vision \cite{goyal2017accurate, wu2022retrievalguard}, natural language processing \cite{devlin2018bert}, generative modeling \cite{brock2018large} and other areas \cite{li2023referring, ma2022traffic, li2022bridging, zhang2022many}. In literature, multi-party collaborative learning is also often called decentralized learning (compared to centralized learning in the single-machine setting). With different network topology, serverless algorithms could be converted into different multi-party collaborative algorithms (seen in \ref{sec:3:1}). 
On the other hand, although there are many ground-breaking studies with DNN in data classification \cite{goyal2017accurate, wei2021direct, wu2022adversarial, sun2022demystify}, most works focus on balanced data sets, optimize the cross entropy, and use accuracy to measure model performance. From the viewpoint of optimization, the cross entropy between the estimated probability distribution based on the output of deep learning models and encoding ground-truth labels is a surrogate loss function of the misclassification rate/accuracy. However, in many real-world applications, such as healthcare and biomedicine \cite{joachims2005support, davis2006relationship, zhang2022toward}, where patients make up a far smaller percentage of the population than healthy individuals, the data distribution is frequently skewed due to the scarce occurrence of positive samples. The data from the majority class essentially define the result, and the accuracy fails to be an appropriate metric to assess classifiers' performance. As a result,  areas under the curves (AUC), including  area under the receiver operating curve (AUROC) and area under precision-recall curves (AUPRC) are given much attention since it excels at discovering models with strong predictive power in imbalanced binary classification \cite{cortes2003auc, ji2023prediction}. 

The prediction performance of models, which are trained with cross entropy as the loss function for imbalanced binary classification, may be subpar because cross-entropy is not the surrogate function of AUC, which call for the study of AUC maximization. Recent works have achieved remarkable progress in directly optimizing AUROC with single-machine and multi-party training algorithms \cite{zhao2011online, liu2019stochastic}. \citet{liu2019stochastic} constructed deep AUC as a minimax problem and resolved the stochastic AUC maximization problem with a deep neural network as the classifier. Recently, \citet{yuan2021federated} and \citet{guo2020communication} extended the single-machine training to federated learning and proposed a PL-strongly-concave minimax optimization method to maximize AUROC. 

However, AUROC is not suitable for data with a much larger number of negative examples than positive examples, and AUPRC can address this issue because it doesn’t rely on true negatives. Given that an algorithm that maximizes AUROC does not necessarily maximize AUPRC \cite{davis2006relationship} and matching of loss and metric is important \cite{dou2022learning, ma2020statistical, dou2022sampling}, the design of AUPRC maximization algorithms has attracted attention  \cite{qi2021stochastic,wang2021momentum, wu2022fast, jiang2022multi,why2022finite}.
Nonetheless, the  multi-party algorithm for AUPRC maximization problems has not been studied. Existing AUPRC optimization methods cannot be directly applied to  multi-party collaborative training, since they mainly focus on the finite-sum problem and maintain an inner state for each positive data point, which is not permitted in a multi-party online environment. In addition, to improve communication efficiency, serverless multi-party collaborative learning algorithms are needed to avoid the server node bottleneck in model training. Thus, it is desired to develop efficient stochastic optimization algorithms for serverless multi-party AUPRC maximization for deep learning to meet the challenge of large-scale imbalanced data mining.

The challenges to design serverless multi-party collaborative AUPRC maximization algorithm are three-fold. The first difficulty lies in the complicated integral definition. 
To overcome the problem of the continuous integral, we can use some point estimators. 
The average precision (AP) estimator is one of the most popularly used estimators. AP can be directly calculated based on the sample prediction scores and is not subject to sampling bias. It is ideally suited to be used in stochastic optimization problems due to these advantages. 

The second difficulty lies in the nested structure and the non-differential ranking functions in the AP. Traditional gradient-based gradient descent techniques cannot directly be used with the original concept of AP. Most existing optimization works use the surrogate function to replace the ranking function in the AP function \cite{liu2019stochastic, guo2020communication, qi2021stochastic,wang2021momentum, wu2022fast,jiang2022multi}. We can follow these works and substitute a surrogate loss for the ranking function in the AP function. 

The third difficulty is that existing algorithms only focus on finite-sum settings and maintain inner estimators $u_t$ for each positive data point, which is not permitted in multi-party collaborative online learning.
Therefore, despite recent developments, it is still unclear if there is a strategy to optimize AUPRC for multi-party collaborative imbalanced data mining. It is natural to ask the following question: \textbf{Can we design multi-party stochastic optimization algorithms to directly maximize AUPRC with guaranteed convergence?}

In this paper, we provide an affirmative answer to the aforementioned question. We propose the new algorithms for multi-party collaborative AUPRC maximization and provide systematic analysis. Our main contributions can be summarized as follows:
\begin{itemize}[leftmargin=5.5mm]
\item  We cast the AUPRC maximization problem into non-convex conditional stochastic optimization problem by substituting a surrogate loss for the indicator function in the definition of AP. Unlike existing methods that just focus on finite-sum settings, we consider the stochastic online setting.

\item We propose the first multi-party collaborative learning algorithm, ServerLess biAsed sTochastic gradiEnt (SLATE), to solve our new objective. It can be used in an online environment and has no reliance on specific local data points. In addition, with different network topologies, our algorithm can also be used for distributed learning and federated learning.

\item  Furthermore, we propose a stochastic method (\emph{i.e.}, SLATE-M) based on the momentum-based variance-reduced technique to reduce the convergence complexity in multi-party collaborative  learning. Our method can reach iteration complexity of $O\left(1 / \epsilon^{5}\right)$, which matches the lower bound proposed in the single-machine conditional stochastic optimization.

\item Extensive experiments on various datasets compared with baselines verify the effectiveness of our methods.
\end{itemize}

\section{Related work}

\subsection{AUROC Maximization}
There is a long line of research that investigated the imbalanced data mining with AUROC metric \cite{zhao2011online,ying2016stochastic,liu2019stochastic,yuan2021federated,guo2020communication}, which highlight the value of the AUC metric in imbalanced data mining.  Earlier works about AUROC focused on linear models with pairwise surrogate
losses \cite{joachims2005support}. Furthermore, \citet{ying2016stochastic} solved the AUC square surrogate loss using a stochastic gradient descent ascending approach and provided a minimax reformulation of the loss to address the scaling problem of AUC optimization. Later, \citet{liu2019stochastic} studied the application of AUROC in deep learning and reconstructed deep AUC as a minimax problem, which offers a strategy to resolve the stochastic AUC maximization problem with a deep neural network as the predictive model. Furthermore, some methods were proposed for multi-party AUROC maximization. \citet{yuan2021federated} and \citet{guo2020communication} reformulated the federated deep AUROC maximization as non-convex-strongly-concave problem in the federated setting. However, the analyses of methods in \cite{yuan2021federated} and \cite{guo2020communication} rely on the assumption of PL condition on the deep models. Recently, \cite{yuan2022compositional} developed the compositional deep
AUROC maximization model and \cite{zhang2023federated} extend it to federated learning.
\subsection{AUPRC Maximization}
Early works about AUPRC optimization mainly depend on traditional optimization techniques. Recently, \citet{qi2021stochastic} analyzed AUPRC maximization with deep models in the finite-sum setting. They use a surrogate loss to replace the ranking function in the AP function and maintain biased estimators of the surrogate ranking functions for each positive data point. They proposed the algorithm to directly optimize AUPRC and show a guaranteed convergence. 
Afterward, \citet{wang2021momentum} presented adaptive and non-adaptive methods (\emph{i.e.} ADAP and MOAP) with a new strategy to update the biased estimators for each data point. The momentum average is applied to both the outer and inner estimators to track individual ranking scores. More recently, algorithms proposed in \cite{why2022finite} reduce convergence complexity with the parallel speed-up and \citet{wu2022fast, jiang2022multi} introduced the momentum-based variance-reduction technology into AUPRC maximization to reduce the convergence complexity. While we developed distributed AUPRC optimization concurrently with \cite{guo2023fedx}, they pay attention to X-Risk Optimization in federated learning. Because X-Risk optimization is a sub-problem in conditional stochastic optimization and federated learning could be regarded as decentralized learning with a specific network topology (seen \ref{sec:3:1}), our methods could also be applied to their problem.

Overall, existing methods mainly focus on finite-sum single-machine setting \cite{qi2021stochastic,wang2021momentum, wu2022fast, why2022finite, jiang2022multi}. To solve the biased stochastic gradient, they maintain an inner state for local each data point. However, this strategy limits methods to be applied to real-world big data applications because we cannot store an inner state for each data sample in the online environment. In addition, we cannot extend them directly from the single-machine setting to multi-party setting, because under non-IID assumption, the data point on each machine is different and this inner state can only contain the local data information and make it difficult to train a global model. 

In the perspective of theoretical analysis, \citet{hu2020biased} studied the general condition stochastic optimization and proposed two single-machine algorithms with and without using the variance-reduction technique (SpiderBoost) named BSGD and BSpiderboost, and established the lower bound at $\varepsilon^{-5}$ in the online setting.

AUPRC is widely utilized in binary classification tasks. It is simple to adapt it for multi-class classifications. If a task has multiple classes, we can assume that each class has a binary classification task and adopt the one vs. the rest classification strategy. We can then calculate average precision based on all classification results.
\vspace{-6pt}
\subsection{Serverless Multi-Party Collaborative Learning}
Multi-party collaborative learning (i.e, distributed and federated learning) has wide applications in data mining and machine learning problems \cite{wang2023applications, zhang2021low, mei2023mac, zhang2022learning, liu2022fast, luo2022multisource, he2023robust}. Multi-party collaborative learning in this paper has a more general definition that does not rely on the IID assumption of data to guarantee the convergence analysis. 
In the last years, many serverless multi-party collaborative learning approaches have been put out because they avoid the communication bottlenecks or constrained bandwidth between each worker node and the central server, and also provide some level of data privacy \cite{yuan2016convergence}. \citet{lian2017can} offered the first theoretical backing for serverless multi-party collaborative training. Then serverless multi-party collaborative training attracts attention \cite{tang2018d, lu2019gnsd, liu2020decentralized, xin2021hybrid} and the convergence rate has been improved using many different strategies, including variance extension \cite{tang2018d}, variance reduction \cite{pan2020d, zhang2021gt}, gradient tracking \cite{lu2019gnsd}, and many more. 
In addition, serverless multi-party collaborative learning has been applied to various applications, such as reinforcement learning \cite{zhang2021taming, he2020data, }, robust training \cite{xian2021faster}, generative adversarial nets (GAN) \cite{liu2020decentralized},  robust
principal component analysis \cite{wu2023decentralized} and other optimization problems \cite{liu2022interact, zhang2022net, zhang2019compressed,zhang2020private}. However, none of them focus on imbalanced data mining. 
In application, the serverless multi-party collaborative learning setting in this paper is different from federated learning \cite{mcmahan2017communication} which uses a central server with different communication mechanisms to periodically average the model parameters for indirectly aggregating data from numerous devices. 
However, with a specific network topology (seen \ref{sec:3:1}), federated learning could be regarded as multi-party collaborative learning. Thus, our algorithms could be regarded as a general federated AUPRC maximization.

\section{Preliminary}
\subsection{Serverless Multi-Party Collaborative Learning} \label{sec:3:1}
\textbf{Notations}: 
We use $\mathbf{x}$ to denote a collection of all local model parameters $\mathbf{x}_n$, where $n \in [N]$, \emph{i.e.}, $\mathbf{x} = [x_{1}^{\top}, x_{2}^{\top}, \dots, x_{N}^{ \top}]^{\top} \in \mathbb{R}^{N d}$. Similarly, we define $\mathbf{u}, \mathbf{v}$ as the concatenation of $\mathbf{u}_n, \mathbf{v}_n$ for $n \in [N]$. In addition, $\otimes$ denotes the Kronecker product, and
$\|\cdot\|$ denotes the $\ell_2$ norm for vectors, respectively. 
$\mathcal{D}_n^{+}$ denotes the positive dataset on the $N$ worker nodes and $\mathcal{D}$ denotes the whole dataset on the $n$ worker nodes.
In multi-party collaborative
training, the network system of N worker nodes $\mathcal{G} = (\mathcal{V},  \mathcal{E}) $ is represented by double stochastic matrix $\underline{\mathbf{W}} = \{\underline{w}_{ij} \}  \in \mathbb{R}^{N \times N}$, which is defined as follows: (1) if there exists a link between node i and node j, then $\underline{w}_{ij} > 0$, otherwise $\underline{w}_{ij}$ = 0, (2) $\underline{\mathbf{W}} = \underline{\mathbf{W}}^{\top}$ and (3) $\underline{\mathbf{W}} \mathbf{1} = \mathbf{1}$ and $\mathbf{1}^{\top} \underline{\mathbf{W}} = \mathbf{1}^{\top}$. We define the second-largest eigenvalue of $\mathbf{W}$ as $\lambda$ and $\mathbf{W} := \underline{\mathbf{W}} \otimes \mathbf{I}_{d}$. We denote the exact averaging matrix as $\mathbf{J} = \frac{1}{N}(\mathbf{1}_n \mathbf{1}_n^{\top}) \otimes \mathbf{I}_d  $ and $\lambda = \| \mathbf{W} - \mathbf{J} \|$. Taking ring network topology as an example, where each node can only exchange information with its two neighbors. The corresponding W is in the form of
\begin{align}
\small
\underline{\mathbf{W}} =\left(\begin{array}{cccccc}
1 / 3 & 1 / 3 &  & &  & 1 / 3 \\
1 / 3 & 1 / 3 & 1 / 3 & & & \\
& 1 / 3 & 1 / 3 & \ddots & & \\
& & \ddots & \ddots & 1 / 3 & \\
& & & 1 / 3 & 1 / 3 & 1 / 3 \\
1 / 3 & & & & 1 / 3 & 1 / 3
\end{array}\right) \in \mathbb{R}^{N \times N} \nonumber
\end{align}
If we change the network topology, multi-Party collaborative
learning could become different types of multi-party collaborative training. If $\mathbf{W}$ is $\frac{1}{N} \mathbf{1} \mathbf{1}^{\top}$, it is converted to distributed learning with the average operation in each iteration. If we choose $\mathbf{W}$  as the Identity matrix and change it to $\frac{1}{N} \mathbf{1} \mathbf{1}^{\top}$ every q iteration, it would be federated learning.

\subsection{AUPRC}
AUPRC can be defined as the following integral problem \cite{bamber1975area}:
\begin{align} 
\mathrm{AUPRC} =\int_{-\infty}^{\infty} \operatorname{Pr}(y = 1 \mid h(\mathbf{x}; \mathbf{z}) \geq c) d \operatorname{Pr}(h(\mathbf{x}; \mathbf{z}) \leq c \mid y=1) \nonumber
\end{align}
where $h(\mathbf{x}; \mathbf{z})$ is the prediction score function, $\mathbf{x}$ is the model parameter, $\mathbf{\xi} = \left(\mathbf{z}, y\right)$ is the data point, and $Pr(y = 1 | h(\mathbf{x}; \mathbf{z}) \geq c)$ is the precision at the threshold value of c.

To overcome the problem of the continuous integral, we use AP as the estimator to approximate AUPRC, which is given by \cite{boyd2013area}:
\begin{equation}  \label{eq:1}
\mathrm{AP} = \mathbb{E}_{\mathbf{\xi} \sim  \mathcal{D}^{+}} Precision \left(h\left(\mathbf{x}; \mathbf{z} \right)\right) 
= \mathbb{E}_{\mathbf{\xi} \sim  \mathcal{D}^{+}} \frac{\mathrm{r}^{+} \left(\mathbf{x}; \mathbf{z}\right)}{\mathrm{r}\left(\mathbf{x}; \mathbf{z}\right)} 
\end{equation}
where $\mathcal{D}^{+}$ denotes the positive dataset, and samples $\mathbf{\xi} = \left(\mathbf{z}, y\right)$ are drawn from positive dataset $\mathcal{D}^{+}$ where $\mathbf{z} \in \mathcal{Z}$ represents the data features and $y = +1$ is the positive label. $\mathrm{r}^{+}$ denotes the positive data rank ratio of  prediction score (\emph{i.e.}, the number of positive data points with no less prediction score than that of $\mathbf{\xi}$  including itself over total data number)  and $\mathrm{r}$ denotes its prediction score rank ratio among all data points (\emph{i.e.}, the number of data points with no less prediction score than that of $\mathbf{\xi}$  including itself  over total data number).  $\mathcal{D}$ denotes the whole datasets and $\mathbf{\xi}^{\prime} = (\mathbf{z}^{\prime}, y^{\prime}) \sim \mathcal{D}$ denote a random data drawn from an unknown distribution $\mathcal{D}$, where $\mathbf{z}^{\prime} \in \mathcal{Z}$ represents the data features and $y^{\prime} \in \mathcal{Y} = \{-1, +1\}$. Therefore, \eqref{eq:1} is the same as:
\begin{align} 
\mathrm{AP} = \mathbb{E}_{\mathbf{\xi} \sim  \mathcal{D}^{+}} \frac{\mathbb{E}_{\mathbf{\xi}^{\prime} \sim  \mathcal{D}} \mathbf{I} (h(\mathbf{x}; \mathbf{z}^{\prime}) \geq h(\mathbf{x}; \mathbf{z})) \cdot \mathbf{I}\left(y^{\prime} = 1\right)}{\mathbb{E}_{\mathbf{\xi}^{\prime} \sim \mathcal{D}} \mathbf{I} (h(\mathbf{x}; \mathbf{z}^{\prime}) \geq h(\mathbf{x}; \mathbf{z}))} \nonumber
\end{align}
We employ the following squared hinge loss:
\begin{align} \label{eq:2}
\ell\left(\mathbf{x}; \mathbf{z}, \mathbf{z}^{\prime} \right) = (max\{s - h(\mathbf{x}; \mathbf{z}) + h(\mathbf{x}; \mathbf{z}^{\prime}), 0\})^2
\end{align}
as the surrogate for the indicator function $\mathbf{I} (h(\mathbf{x}; \mathbf{z}^{\prime}) \geq h(\mathbf{x}; \mathbf{z}))$, where $s$ is a margin parameter, that is a common choice used by previous studies
\cite{qi2021stochastic, wang2021momentum, jiang2022multi}. As a result, the AUPRC maximization problem can be formulated as:
\begin{align} 
AP = \mathbb{E}_{\mathbf{\xi} \sim  \mathcal{D}^{+}} \frac{\mathbb{E}_{\mathbf{\xi}^{\prime} \sim  \mathcal{D}} \mathbf{I}\left(y^{\prime}=1\right)  \ell \left(\mathbf{x}; \mathbf{z}, \mathbf{z}^{\prime}\right)}{\mathbb{E}_{\mathbf{\xi}^{\prime} \sim \mathcal{D}} \quad \ell\left(\mathbf{x}; \mathbf{z}, \mathbf{z}^{\prime}\right)}  \nonumber
\end{align}
In the finite-sum setting, it is defined as :
\begin{align} 
AP = \frac{1}{| \mathcal{D}^{+}|} \sum_{\mathbf{x}_{i}, y_{i}=1} \frac{\mathrm{r}^{+}\left(\mathbf{x}_{i}\right)}{\mathrm{r}\left(\mathbf{x}_{i}\right)} = \frac{1}{| \mathcal{D}^{+}|} \sum_{\mathbf{\xi} \sim  \mathcal{D}^{+}} \frac{ 
 \frac{1}{| \mathcal{D}|} \sum_{\mathbf{\xi} \sim  \mathcal{D}}
\mathbf{I}\left(y^{\prime}=1\right)  \ell \left(\mathbf{x}; \mathbf{z}, \mathbf{z}^{\prime}\right)}{ \frac{1}{| \mathcal{D}|} \sum_{\mathbf{\xi} \sim  \mathcal{D}} \quad \ell\left(\mathbf{x}; \mathbf{z}, \mathbf{z}^{\prime}\right)}  \nonumber
\end{align}
For convenience, we define the elements in $g(\mathbf{x})$ as the surrogates of the two prediction score ranking function $\mathrm{r}^{+}\left(\mathbf{x}\right)$ and $\mathrm{r} \left(\mathbf{x}\right)$ respectively. Define the following equation: 
\begin{align} 
g\left(\mathbf{x}; \mathbf{\xi}, \mathbf{\xi}^{\prime} \right)=
\begin{bmatrix}
g^1\left(\mathbf{x}; \mathbf{\xi}, \mathbf{\xi}^{\prime} \right) \\
g^2\left(\mathbf{x}; \mathbf{\xi}, \mathbf{\xi}^{\prime} \right)
\end{bmatrix} =
\begin{bmatrix}
\ell\left(\mathbf{x} ; \mathbf{z}, \mathbf{z}^{\prime} \right) \mathbf{I}\left(y^{\prime}=1\right) \\
\ell\left(\mathbf{x} ; \mathbf{z}, \mathbf{z}^{\prime} \right) 
\end{bmatrix} \nonumber
\end{align}
and $g(\mathbf{x}; \mathbf{\xi}) = \mathbb{E}_{\mathbf{\xi}^{\prime} \sim \mathcal{D}} g\left(\mathbf{x} ; \mathbf{\xi}, \mathbf{\xi}^{\prime}\right) \in \mathbb{R}^{2} $, and assume $f(\mathbf{u}) = -\frac{u_1}{u_2}:\mathbb{R}^{2} \mapsto \mathbb{R}$ for any $\mathbf{u} = [u_1, u_2]^{\top} \in  \mathbb{R}^{2} $. 
We reformulate the optimization objective into the following stochastic optimization problem:
\begin{align} \label{eq:3}
\min_{\mathbf{x}} F (\mathbf{x}) &=\mathrm{E}_{\mathbf{\xi} \sim \mathcal{D}^{+}}\left[f(g(\mathbf{x}; \mathbf{\xi}) \right] \nonumber\\
&=\mathrm{E}_{\mathbf{\xi} \sim \mathcal{D}^{+}} \left[f(\mathrm{E}_{\mathbf{\xi}^{\prime} \sim \mathcal{D}} g(\mathbf{x}; \mathbf{\xi}, \mathbf{\xi}^{\prime}))\right]
\end{align}
It is similar to the two-level conditional stochastic optimization \cite{hu2020biased}, where the inner layer function depends on the data points sampled from both inner and outer layer functions. Given that $f(\cdot)$ is a nonconvex function, problem \eqref{eq:3} is a noncvonex optimiztion problem. In this paper, we considers serverless multi-party collaborative non-convex optimization where N worker nodes cooperate to solve the following problem:
\begin{align} \label{eq:4}
\min_{\mathbf{x}} F (\mathbf{x}) &= \min_{\mathbf{x}} \frac{1}{N} \sum_{n=1}^n F_{n}(\mathbf{x})
\end{align}
where $F_{n}(\mathbf{x}) =
\mathbb{E}_{\xi_n \sim \mathcal{D}_n^{+}} f(\mathbb{E}_{\xi^{\prime}_n \sim \mathcal{D}_n} g_n(\mathbf{x}; \xi_n, \xi^{\prime}_n))$ where $\mathbf{\xi}_n^{\prime} = (\mathbf{z}_n^{\prime}, y_n^{\prime}) \sim \mathcal{D}_n$ and $\mathbf{\xi}_n = (\mathbf{z}_n, y_n) \sim \mathcal{D}^{+}$  
We consider heterogeneous data setting in this paper, which refers to a situation where $\mathcal{D}_{i}$ and $\mathcal{D}_{j}$ are different ($i \neq j$ ) on different worker nodes.

In order to design the method, we first consider how to compute the gradient of $F(\mathbf{x})$.
\begin{equation} \label{eq:8}
\begin{aligned}
\nabla F_n(\mathbf{x})= & \mathbb{E}_{\xi_n \sim \mathcal{D}_n^{+}} \nabla g_n(\mathbf{x}; \xi_n)^{\top} \nabla f\left(g_n(\mathbf{x}; \xi_n)\right) \\
= & \mathbb{E}_{\xi_n \sim \mathcal{D}_n^{+}} \nabla g_n(\mathbf{x}; \xi_n)^{\top}\left(\frac{-1}{ g^{2}_n(\mathbf{x}; \xi_n)}, \frac{ g^{1}_n(\mathbf{x}; \xi_n) }{\left( g^{2}_n(\mathbf{x}; \xi_n)\right)^{2}}\right)^{\top} \nonumber
\end{aligned}
\end{equation}
where 
\begin{align} 
\nabla  g_n(\mathbf{x}; \xi_n) &= 
\begin{bmatrix}
\nabla  g^{1}_n(\mathbf{x}; \xi_n) \\
\nabla  g^{2}_n(\mathbf{x}; \xi_n)
\end{bmatrix}
\nonumber\\
&= \begin{bmatrix}
\mathbb{E}_{\xi^{\prime}_n \sim \mathcal{D}_n} \mathbf{I}\left(y_n^{\prime}=1\right) \nabla \ell\left(\mathbf{x} ; \mathbf{z}_n, \mathbf{z}_n^{\prime}\right) \\
\mathbb{E}_{\xi^{\prime}_n \sim \mathcal{D}_n} \nabla \ell\left(\mathbf{x}; \mathbf{z}_n, \mathbf{z}_n^{\prime}\right) \nonumber
\end{bmatrix}
\end{align}

We can notice that it is different from the standard gradient since there are two levels of functions and the inner function also depends on the sample data from the outer layer. Therefore, the stochastic gradient estimator is not an unbiased estimation for the full gradient. Instead of constructing an unbiased stochastic estimator of the gradient \cite{sun2023scheduling}, we consider a biased estimator of $\nabla F_n(x)$ using one sample $\xi$ from $\mathcal{D}^{+}_n$ and $m$ sample $\xi^{\prime}$ from $\mathcal{D}_n$ as $\mathcal{B}_{n}$ in the following form:
\begin{align} \label{eq:5}
&\nabla \hat{F}_n\left(\mathbf{x}; \xi_n, \mathcal{B}_{n}\right) \\
=&(\frac{1}{m} \sum_{\xi^{\prime} \in \mathcal{B}_{n}} \nabla g_n(\mathbf{x}; \xi_n, \xi^{\prime}_n))^{\top} \nabla f (\frac{1}{m} \sum_{\xi^{\prime}_n \in \mathcal{B}_{n}} g_n(\mathbf{x}; \xi_n, \xi^{\prime}_n))   \nonumber 
\end{align}
where $\mathcal{B}_{n} = \left\{\xi^{\prime  j}\right\}_{j=1}^m$.
It is observed that $\nabla \hat{F}_n\left(\mathbf{x}; \xi_n, \mathcal{B}_{n} \right)$ is the gradient of an empirical objective such that
\begin{align}
\hat{F}_n\left(\mathbf{x}; \xi_n, \mathcal{B}_{n} \right):=f_n \left(\frac{1}{m} \sum_{\xi^{\prime}_n \in \mathcal{B}_{n}} g_n(\mathbf{x}; \xi_n, \xi^{\prime}_n)\right) \,.\nonumber
\end{align}

\section{Algorithms}
In this section, we propose the new serverless multi-party collaborative learning algorithms for
solving the problem \eqref{eq:4}. Specifically, we use the gradient tracking technique ( which could be ignored in practice) and propose a ServerLess biAsed sTochastic gradiEnt (SLATE). We further propose an accelerated version of SLATE with momentum-based variance reduction \citep{cutkosky2019momentum} technology (SLATE-M).

\subsection{Serverless Biased Stochastic Gradient (SLATE)}

Based on the above analysis, we design a serverless multi-party collaborative algorithms with biased stochastic gradient and is named SLATE. \Cref{alg:1} shows the algorithmic framework of the SLATE. \textbf{Step 8 could be ignored in practice}.

\setlength{\textfloatsep}{0.15cm}
\begin{algorithm}[!t]
\caption{SLATE Algorithm }
\label{alg:1}
\begin{algorithmic}[1] 
\STATE {\bfseries Input:} $T$, step size $ \eta$ inner batch size $m$ and mini-batch size $b$; 
$\mathbf{u}_{n, 0} = 0$ and $\mathbf{v}_{n, 0} = 0$ for $n \in \{1, \cdots, N\}$
\STATE {\bfseries Initialize:} $x_{n, 0} = \frac{1}{N} \sum_{k=1}^{N} x_{n, 0}$. \\
\FOR{$t = 0, 1, \ldots, T$}
\FOR{$n = 1, 2, \ldots, N$}
\STATE Draw $b$ samples $\mathcal{B}^{+}_{n, t} = \{\xi^i_{n, t}\}_{i = 1}^{b}$ from $\mathcal{D}^{+}_n$ \\
\STATE  Draw $m$ samples $\mathcal{B}_{n, t} = \left\{\xi^{\prime  j}_{n, t}\right\}_{j=1}^m$ from $\mathcal{D}_n$,
\\
\STATE $\mathbf{u}_{n, t} = \frac{1}{b} \sum_{i=1}^{b} \nabla \hat{F}_n(\mathbf{x}_{n,t}; \xi^i_{n,t}, \mathcal{B}_{n, t})$ as in \eqref{eq:6} \\
\STATE $\mathbf{v}_{n, t} = \sum_{r=1}^{N} \underline{w}_{nr} (\mathbf{v}_{t - 1}^r + \mathbf{u}_t^r - \mathbf{u}_{t - 1}^r)$ \\
\STATE $\mathbf{x}_{n, t+1} = \sum_{r = 1}^{N} \underline{w}_{nr} (\mathbf{x}_{t}^{n} -  \eta \mathbf{v}_{n, t})$\\
\ENDFOR
\ENDFOR
\STATE {\bfseries Output:} $x$ chosen uniformly random from $\{\bar{\mathbf{x}}_t\}_{t=1}^{T}$.
\end{algorithmic}
\end{algorithm}

At the beginning of Algorithm \ref{alg:1}, one simply initializes local model parameters $\mathbf{x}$ for all worker nodes. 
Given the couple structure of problem \eqref{eq:4}. We can assign the value of gradient estimator $\mathbf{u}_{n, 0}$ and gradient tracker $\mathbf{v}_{n, 0}$ as 0. 

At the Lines 5-6 of Algorithm \ref{alg:1}, we draw $b$ samples as $\mathcal{B}^{+}_{n, t}$ from positive dataset $\mathcal{D}_n^{+}$ and $m$ samples as $\mathcal{B}_{n, t}$from full data sets $\mathcal{D}_n$ on each node, respectively. We use a biased stochastic gradient to update the gradient estimator $\mathbf{u}_{n,t}$ according to the \eqref{eq:6}.
\begin{align} \label{eq:6}
&\mathbf{u}_{n, t} = \\
& \sum_{\mathbf{\xi} \in \mathcal{B}^{+}_{n,t}} \sum_{\mathbf{\xi}^{\prime} \in \mathcal{B}_{n , t}} \frac{\left(g^1\left(\mathbf{x}_{n,t}; \mathbf{\xi}, \mathbf{\xi}^{\prime} \right) - g^2\left(\mathbf{x}_{n,t}; \mathbf{\xi}, \mathbf{\xi}^{\prime} \right) \mathbf{I}\left(\mathbf{y}^{\prime}=1\right)\right) \nabla \ell\left(\mathbf{x}_{n,t}; \mathbf{z}, \mathbf{z}^{\prime}\right)}{b m \left( g^2\left(\mathbf{x}_{n,t}; \mathbf{\xi}, \mathbf{\xi}^{\prime} \right) \right)^2} \nonumber
\end{align}
where $\mathbf{\xi} = \left(\mathbf{z}, y\right)$ and $\mathbf{\xi}^{\prime} = \left(\mathbf{z}^{\prime}, y^{\prime}\right)$

Afterward, at the Line 8 of Algorithm \ref{alg:1} (optional), we adopt the gradient tracking technique \cite{lu2019gnsd} to reduce network consensus error, where we update the $\mathbf{v}_{n, t}$ and then do the consensus step with double stochastic matrix $\mathbf{W}$ as:
\begin{align}
\mathbf{v}_{n, t} = \sum_{r=1}^{N} \underline{w}_{nr} (\mathbf{v}_{t - 1}^r + \mathbf{u}_t^r - \mathbf{u}_{t - 1}^r) \nonumber
\end{align}

Finally, at the Line 9 of Algorithm \ref{alg:1}, we update the model with gradient tracker $\mathbf{v}_{n, t}$, following the consensus step among worker nodes with double stochastic matrix $\mathbf{W}$:
\begin{align}
\mathbf{x}_{n, t+1} = \sum_{r = 1}^{N} \underline{w}_{nr} (\mathbf{x}_{r, t} -  \eta \mathbf{v}_{r, t}) \nonumber
\end{align}
The output $\bar{\mathbf{x}}_t$ is defined as:
$
\bar{\mathbf{x}}_t = \frac{1}{N} \sum_{n=1}^{N} \mathbf{x}_{n,t}\,. \nonumber
$
\subsection{SLATE-M}
Furthermore, we further propose an accelerated version of SLATE (SLATE-M) based on the momentum-based variance reduced technique, which has the better convergence complexity. The algorithm is shown in \Cref{alg:2}. \textbf{Step 11 could be ignored in practice}.

\begin{algorithm}[!t]
\caption{SLATE-M Algorithm }
\label{alg:2}
\begin{algorithmic}[1] 
\STATE {\bfseries Input:} $T$, step size $ \eta$, momentum coefficient $\alpha$, inner batch size $m$ and mini-batch size $b$, and initial batch size $B$; 
\STATE {\bfseries Initialize:} $\mathbf{x}_{n, 0} = \frac{1}{N} \sum_{k=1}^{N} \mathbf{x}_{n, 0}$ 
\STATE  Draw $B$ samples of $\{\xi^i_{n, 0}\}_{i=1}^B$ from $\mathcal{D}^{+}_n$, and draw $m$ samples $\mathcal{B}_{n, 0} = \left\{\xi_{n,0}^{\prime  j}\right\}_{j=1}^m$ from $\mathcal{D}_n$, $\mathbf{u}_{n, 0} = \frac{1}{B} \sum_{i=1}^{B} \nabla \hat{F}_n(\mathbf{x}_{n, 0}; \xi^i_{n, 0}, \mathcal{B}_{n, 0}) \forall n \in [N]$
\\
\STATE $\mathbf{v}_{n, 0} = \sum_{r=1}^{N} \underline{w}_{nr} \mathbf{u}_{r, 0} \forall n \in [N]$ \\
\STATE $\mathbf{x}_{n, 1} = \sum_{r = 1}^{N} \underline{w}_{nr} (\mathbf{x}_{n, 0} -  \eta \mathbf{v}_{n, 0}) \forall n \in [N]$\\
\FOR{$t = 1, 2, \ldots, T$}
\FOR{$n = 1, 2, \ldots, N$}
\STATE Draw $b$ samples  $\mathcal{B}^{+}_{n, t} = \{\xi^0_{n, 1}, \cdots, \xi^{b}_{n, 1} \}$ from $\mathcal{D}^{+}_n$ \\
\STATE  Draw $m$ samples $\mathcal{B}_{n, t} = \left\{\xi_{n,t}^{\prime  j}\right\}_{j=1}^m$ from $\mathcal{D}_n$,
\\
\STATE $\mathbf{u}_{n, t} = \frac{1}{b} \sum_{i=1}^{b} \nabla \hat{F}_n(\mathbf{x}_{n, t}; \xi^i_{n,t}, \mathcal{B}_{n, t}) + (1 - \alpha) (\mathbf{u}_{n, t - 1}  - \frac{1}{b} \sum_{i=1}^{b} \nabla \hat{F}_n(\mathbf{x}_{n, t - 1}; \xi^i_{n,t}, \mathcal{B}_{n, t}))$ \\
\STATE $\mathbf{v}_{n, t} = \sum_{r=1}^{N} \underline{w}_{nr} (\mathbf{v}_{t - 1}^r + \mathbf{u}_t^r - \mathbf{u}_{t - 1}^r)$ \\
\STATE $\mathbf{x}_{n, t+1} = \sum_{r = 1}^{N} \underline{w}_{nr} (\mathbf{x}_{t}^{n} -  \eta \mathbf{v}_{n, t}) $\\
\ENDFOR
\ENDFOR
\STATE {\bfseries Output:} $x$ chosen uniformly random from $\{\bar{\mathbf{x}}_t\}_{t=1}^{T}$.
\end{algorithmic}
\end{algorithm}

At the beginning, similar to the SLATE,  worker nodes initialize local model parameters $\mathbf{x}$ as seen in Lines 1-2 in \Cref{alg:2}. 

Different from SLATE, we initialize the $\mathbf{u}_{n,0}$ with initial batch size $B$ and $\mathbf{v}_{n,0} \forall n \in [N]$, which can be seen in Lines 3-4 in \Cref{alg:2}. Then we do the consensus step to update the model parameters $\mathbf{x}_n$. The definition of $\hat{F}_n(\mathbf{x}_{n, 0}; \xi^i_{n, 0}, \mathcal{B}_{n, 0}) $ is similar to \eqref{eq:6} as below:
\begin{align} \label{eq:7}
&\frac{1}{|\mathcal{B}^{+}_{n, t} | } \hat{F}_n(\mathbf{x}_{n, t}; \xi^i_{n, t}, \mathcal{B}_{n, t})  =\\
& \sum_{\mathbf{\xi} \in \mathcal{B}^{+}_{n,t}} \sum_{\mathbf{\xi}^{\prime} \in \mathcal{B}_{n , t}} \frac{\left(g^1\left(\mathbf{x}_{n,t}; \mathbf{\xi}, \mathbf{\xi}^{\prime} \right) - g^2\left(\mathbf{x}_{n,t}; \mathbf{\xi}, \mathbf{\xi}^{\prime} \right) \mathbf{I}\left(\mathbf{y}^{\prime}=1\right)\right) \nabla \ell\left(\mathbf{x}_{n,t}; \mathbf{z}, \mathbf{z}^{\prime}\right)}{|\mathcal{B}^{+}_{n, t} | m \left( g^2\left(\mathbf{x}_{n,t}; \mathbf{\xi}, \mathbf{\xi}^{\prime} \right) \right)^2}  \nonumber
\end{align}
where $|\mathcal{B}_{n, t} |$ denotes the size of batch $\mathcal{B}_{n, t} $ and $\mathbf{\xi} = \left(\mathbf{z}, y\right)$ and $\mathbf{\xi}^{\prime} = \left(\mathbf{z}^{\prime}, y^{\prime}\right)$.

Afterwards, similar to SLATE, each iteration, we draw $b$ samples from positive dataset $\mathcal{D}_n^{+}$ and $m$ samples from full data sets $\mathcal{D}_n$ on each worker node, respectively to construct the biased stochastic gradient, seen in Line 8-9 of Algorithm \ref{alg:2}.

The key different between SLATE and SLATE-M is that we update gradient estimator $\mathbf{u}_{n, t}$ in SLATE-M with the following variance reduction method: 
\begin{align}
& \mathbf{u}_{n, t} = \frac{1}{b} \sum_{i=1}^{b} \nabla \hat{F}_n(\mathbf{x}_{n, t}; \xi^i_{n,t}, \mathcal{B}_{n, t}) + (1 - \alpha) (\mathbf{u}_{n, t - 1} \nonumber\\
&- \frac{1}{b} \sum_{i=1}^{b} \nabla \hat{F}_n(\mathbf{x}_{n, t - 1}; \xi^i_{n,t}, \mathcal{B}_{n, t})) \nonumber
\end{align}
where $\frac{1}{b} \sum_{i=1}^{b} \nabla \hat{F}_n(\mathbf{x}_{n, t}; \xi^i_{n,t}, \mathcal{B}_{n, t})$ and $\frac{1}{b} \sum_{i=1}^{b} \nabla \hat{F}_n(\mathbf{x}_{n, t - 1}; \xi^i_{n,t}, \mathcal{B}_{n, t})$ are defined in \eqref{eq:7}

Finally, we update gradient tracker $\mathbf{v}_{n,t}$ and model parameters $\mathbf{x}_{n,t}$ as in Lines 11-12 in \Cref{alg:2}.

\section{Theoretical Analysis}
We will discuss some mild assumptions and present the convergence results of our algorithms (SLATE and SLATE-M).

\subsection{Assumptions}
In this section, we introduce some basic assumptions used for theoretical analysis.
\begin{assumption} \label{ass:1} $\forall n \in [N]$, we assume (i) there is $C ( > 0)$ that $\ell(x; z_n, z_n) > C$; (ii) there is $M ( > 0)$ that $0 < \ell(x; z_n, z_n^{\prime}) < M$; (iii)  $\ell(x; z_n, z_n^{\prime})$ is Lipschitz continuous and smooth with respect to model $\mathbf{x}$ for any $\xi_n = (z_n, y_n) \sim \mathcal{D}^{+}, \xi^{\prime}_n = (z^{\prime}_n, y^{\prime}_n) \sim \mathcal{D}$. \end{assumption}
\begin{assumption} \label{ass:2}
$\forall n \in [N]$, we assume there exists a positive constant $\sigma$, such that $\| \nabla g(\mathbf{x}; \xi , \xi^{\prime})\|^2 \leq \sigma^2,  \forall \xi \sim \mathcal{D}_n^{+}, \xi^{\prime} \sim \mathcal{D}_n$
\end{assumption}
Assumptions \ref{ass:1} and \ref{ass:2} are a widely used assumption in optimization analysis of AUPRC maximization  \cite{qi2021stochastic, wang2021momentum,jiang2022multi}. They can be easily satisfied when we choose a smooth surrogate loss function $\ell(\mathbf{x}; \mathbf{z}, \mathbf{z}^{\prime})$ and a bounded score function model $h(\mathbf{x}; \cdot)$.

Furthermore, based on Assumptions \ref{ass:1} and \ref{ass:2}, we can build the smoothness and lipschitz continuity of objective function in the problem \eqref{eq:4}.

\begin{lemma} \label{lem:1}(Lemma 1 in \cite{wang2021momentum}) Suppose Assumptions  \ref{ass:1} and \ref{ass:2} hold, then $\forall \mathbf{x}, \| g_n(\mathbf{x}; \mathbf{\xi}) \|^2 \leq \sigma_g^2$, $g_n(\mathbf{x}; \mathbf{\xi})$ is $L_g$-Lipschitz and $S_g$-smooth for $\mathbf{xi} \sim \mathcal{D}_n^{+}$, and $\forall u \in \Omega, f(u)$ is $L_f$-Lipschitz and $S_f$-smooth. $\forall \mathbf{x}, F_n(\mathbf{x})$ is $L_F$-Lipschiz and $S_F$-smooth.
\end{lemma}

From the \Cref{lem:1}, we have $f_n$ and $g_n$ are $S_f$-smooth and $S_h$-smooth. This implies that for an samgle $\mathbf{\xi}_n \sim \mathcal{D}_n^{+}$ there exist $S_f > 0$  and $S_g > 0$ such that
\begin{align}
\mathbb{E}\|\nabla f_n(x_1) - \nabla f_n(x_2) \| \leq S_{f} \|x_1 - x_2\| \nonumber\\
\mathbb{E}\|\nabla g_n(y_1, \xi_n) - \nabla g_n(y_2, \xi_n) \| \leq S_{g} \|y_1 - y_2\| \nonumber
\end{align}
And $f_n$ and $g_n$ are $L_f$-Lipchitz continuous and $L_g$-Lipchitz continuous. This implies that there exist $L_f > 0$  and $L_g > 0$ such that
\begin{align}
\mathbb{E}\|\nabla f_n(x)  \|^2 \leq L^2_{f}  \nonumber\\
\mathbb{E}\|\nabla g_n(y_1, \xi_n) \|^2 \leq S^2_{g} \nonumber 
\end{align}
In addition, we also have bounded variance of $g_n$. There exist $\sigma_g > 0$  such that
\begin{align}
\mathbb{E}_{\xi_n \sim \mathcal{D}_n^{+}} \|g_n(\mathbf{x}; \mathbf{\xi}_n, \mathbf{\xi}^{\prime}_n) - \mathbb{E}_{\xi_n^{\prime} \sim \mathcal{D}_n} g_n(\mathbf{x}; \mathbf{\xi}_n, \mathbf{\xi}^{\prime}_n) \|^2  \leq \sigma_g^2  \nonumber
\end{align}
which indicates that the inner function $g_n$ has bounded variance. To control the estimation bias, we follow the analysis in single-machine conditional  stochastic optimization \cite{hu2020biased}. 

\begin{lemma} \label{lem:A1}  (Proposition B.1 in \cite{hu2020biased}) Under Assumptions \ref{ass:1} and \ref{ass:2}, on the $n$-th worker node, for a sample $\mathbf{\xi}_n \sim \mathcal{D}_n^{+}$ and $m$ samples $\mathcal{B}_{n} $ from $\mathcal{D}_n$, \\
(a) $\mathcal{B}_{n} = \left\{\xi^{\prime  j}\right\}_{j=1}^m $ and we have
\begin{align}
\|\mathbb{E} \nabla \hat{F}_n(x; \xi_n, \mathcal{B}_{n}) - \nabla F_n (x) \|^2 \leq \frac{L_g^2 S_f^2 \sigma_g^2}{m}
\end{align} .\\
(b) $\mathbb{E} \nabla \hat{F}_n(x; \xi_n, \mathcal{B}_{n})$ are $S_F$-Lipschitz smooth \\
(c) \begin{align}
\left\|\nabla\left(f(\hat{g}_n(\mathbf{x}, \xi_n))\right) - \nabla \hat{F}(x)\right\|_2^2 \leq L_f^2 L_g^2
\end{align}
\end{lemma}
\Cref{lem:A1} (a) provide a bound of biased stochastic gradient, which will be used in the following theoretical analysis.

\begin{assumption} \label{ass:4}
The function $F_n(x)$ is bounded below, \emph{i.e.,} $\inf_{x} F_n(x) > -\infty$.
\end{assumption}

\subsection{The Communication Mechanism in Serverless Multi-Party Collaborative Training}
The network system of N worker nodes $\mathcal{G} = (\mathcal{V},  \mathcal{E}) $ is represented by double stochastic matrix $\underline{\mathbf{W}} = \{\underline{w}_{ij} \}  \in \mathbb{R}^{N \times N}$ in the analysis. 

For the ease of exposition, we write the $\mathbf{x}_t$ and $\mathbf{v}_t$-update in \Cref{alg:1} and \Cref{alg:2} in the following equivalent matrix form: $\forall t \geq 0$,
\begin{equation}
\mathbf{v}_{t} = \mathbf{W} (\mathbf{v}_{t - 1} + \mathbf{u}_{t} - \mathbf{u}_{t - 1}), \;\;
\mathbf{x}_{t+1} = \mathbf{W} (\mathbf{x}_{t - 1} - \eta \mathbf{v}_{t}) \nonumber
\end{equation}
where $\mathbf{W} := \underline{\mathbf{W}} \otimes \mathbf{I}_{d}$ and $\mathbf{x}_t, \mathbf{u}_t, \mathbf{v}_t$ are random vectors in $\mathbb{R}^{Nd}$ that respectively concatenate the local estimates $\{\mathbf{x}_{n, t} \}_{n=1}^{N}$ of a stationary point of $F$ , gradient trackers $\{\mathbf{v}_{n, t} \}_{n=1}^{N} $ , gradient estimators $\{\mathbf{u}_{n, t} \}_{n=1}^{N} $. With the exact averaging matrix $\mathbf{J}$, 
we have following quantities:
\begin{equation}
\mathbf{1} \otimes \mathbf{\bar{x}}_t  :=  \mathbf{J} \mathbf{x}_t, \;
\mathbf{1} \otimes  \bar{\mathbf{u}}_t := \mathbf{J} \mathbf{u}_t, \;
\mathbf{1} \otimes  \bar{\mathbf{v}}_t  := \mathbf{J} \mathbf{v}_t\, \nonumber
\end{equation}

Next, we enlist some useful results of gradient tracking-based algorithms for serveress multi-party collaborative stochastic optimization
\begin{lemma} \label{lem:A2} (Lemma 1 in  \cite{xin2021hybrid})
For double stochastic
matrix, we have the following:\\
(a) $\|\mathbf{W} \mathbf{x} - \mathbf{J} \mathbf{x}\| \leq \lambda\|\mathbf{x} - \mathbf{J} \mathbf{x}\|, \forall \mathbf{x} \in \mathbb{R}^{n d}$.\\
(b) $\bar{\mathbf{v}}_{t}=\bar{\mathbf{u}}_t, \forall t \geq 0$. 
As the update step in \ref{alg:1} and \ref{alg:2}, we have 
\begin{align}
 \bar{\mathbf{x}}_{t+1} = \bar{\mathbf{x}}_t - \eta  \bar{\mathbf{v}}_{t} = \bar{\mathbf{x}}_t - \eta \bar{\mathbf{u}}_t \nonumber
\end{align}
(c) According to the definition of network $\mathbf{W}$, we have the following inequalities: $\forall k \geq 0$,
\begin{align}
&\left\|\mathbf{x}_{t + 1} - \mathbf{J} \mathbf{x}_{t + 1}\right\|^2 \leq \frac{1 + \lambda^2}{2}\left\|\mathbf{x}_t - \mathbf{J} \mathbf{x}_t\right\|^2 + \frac{2 \eta^2 \lambda^2}{1 - \lambda^2}\left\|\mathbf{v}_{t} - \mathbf{J v}_{t}\right\|^2 \label{eq:16}\\
&\left\|\mathbf{x}_{t+1} - \mathbf{J} \mathbf{x}_{t+1} \right\|^2 \leq 2 \lambda^2\left\|\mathbf{x}_t - \mathbf{J} \mathbf{x}_t\right\|^2 + 2 \eta^2 \lambda^2\left\|\mathbf{v}_{t} - \mathbf{J} \mathbf{v}_{t}\right\|^2  \label{eq:17} \\
&\left\|\mathbf{x}_{t + 1} - \mathbf{J} \mathbf{x}_{t + 1}\right\| \leq \lambda \left\| \mathbf{x}_t - \mathbf{J} \mathbf{x}_t\right\|^2 + \eta \lambda \left\|\mathbf{v}_{t} - \mathbf{J} \mathbf{v}_{t}\right\| \label{eq:18}
\end{align}
\end{lemma}

Then, we study the convergence properties of SLATE and SLATEM. We first discus the metric
to measure convergence of our algorithms. Given that the loss function is nonconvex, we are unable to demonstrate convergence to an global minimum point. Instead, we establish convergence to an approximate stationary point,  defined below:

\begin{definition}
 A point $x$ is called $\epsilon$-stationary point if $\|\nabla f(x)\| \leq \epsilon$.  Generally, a stochastic algorithm is defined to achieve an $\epsilon$-stationary point in $T$ iterations if  $\mathbb{E}\|\nabla f(x_T)\| \leq \epsilon$.
\end{definition}

\subsection{Convergence Analysis of SLATE Algorithm}
First, we study the convergence properties of our SLATE algorithm. The detailed proofs are provided in the supplementary materials.
\begin{theorem} \label{thm:1} Suppose the sequence $\{\mathbf{\bar{x}}_t\}_{t=1}^T$ be generated from Algorithm \ref{alg:1} and  Assumptions \ref{ass:1}, \ref{ass:2}, and \ref{ass:4} hold, $0 < \eta \leq  \min \{\frac{1 - \lambda^2}{24 \lambda^2 S_F}, \frac{1}{6 S_F}\}$, SLATE in \cref{alg:1} has the following
\begin{align}
&\frac{1}{T} \sum_{t=0}^{T - 1} \mathbb{E} \|\nabla F(\bar{\mathbf{x}}_{t})\|^2 \leq  \frac{2  \mathbb{E} [F(\bar{\mathbf{x}}_{0}) - F(\bar{\mathbf{x}}_{T})] }{\eta T}
\nonumber\\
&+ \left(\frac{1}{ \lambda^2} + 5 N\right) \frac{32 \lambda^2 \eta^2 L_f^2 L_g^2 S_F^2}{(1 - \lambda^2)^2 N}  + \frac{2 L_g^2 S_f^2 \sigma_g^2}{m} + \frac{2 \eta S_F L_f^2 L_g^2}{N} \nonumber
\end{align}
\end{theorem}
\begin{corollary} \label{corollary:1} Based on the analysis in \Cref{thm:1}, by setting $\eta = O(\sqrt{\frac{N}{T}})$ , SLATE in \Cref{alg:1} has the
following
\begin{align}
&\frac{1}{T} \sum_{t=0}^{T - 1} \mathbb{E} \|\nabla F(\bar{\mathbf{x}}_{t})\|^2 \leq O(\frac{\mathbb{E} [F(\bar{\mathbf{x}}_{0}) - F(\bar{\mathbf{x}}_{T})] }{(N T)^{1/2}}) \nonumber\\
&+  \left(\frac{1}{ \lambda^2} + 5 \right) \frac{24 \lambda^2 L_f^2 L_g^2 S_F^2}{(1 - \lambda^2)^2} O(\frac{N}{T})  + \frac{2 L_g^2 S_f^2 \sigma_g^2}{m} + O(\frac{2 S_F L_f^2 L_g^2}{(NT)^{1/2}}) \nonumber
\end{align}
\end{corollary}
Based on the result in \Cref{thm:1} and \Cref{corollary:1}, we can get the convergence result of SLATE. 
\begin{remark}
According to \Cref{corollary:1}, without loss of generality, we let $m = O(\varepsilon^{-2})$, $b = O(1)$ and $\sqrt{T} > N$, we know to make $\frac{1}{T} \sum_{t=0}^{T - 1} \mathbb{E} \|\nabla F(\bar{\mathbf{x}}_{t})\|^2 \leq \varepsilon^{2}$, we have iterations $T$ should be as large as $ O( N^{-1} \varepsilon^{-4} )$. 

In \Cref{alg:1}, we sample $b + m$ data points to build the biased stochastic gradients $\mathbf{u}_{n,t}$, and need $T$ iterations. Thus, our SLATE algorithm has a sample complexity of   $m \cdot T = O(N^{-1} \varepsilon^{-6}) $, for finding an $\epsilon$-stationary point. In addition, the result also indicates the linear speedup of our algorithm with respect to the number of worker nodes.
\end{remark}
\subsection{Convergence Analysis of SLATE-M}
In the subsection, we study the convergence properties of our SLATE-M algorithm. The details about proofs are provided in the supplementary materials.

\begin{theorem}  \label{thm:2}
Suppose the sequence $\{\mathbf{\bar{x}}_t\}_{t=1}^T$ be generated from Algorithm \ref{alg:2} and Assumptions \ref{ass:1}, \ref{ass:2}, and \ref{ass:4} hold, $0 < \eta \leq min \{\frac{1}{4}, \\ \frac{\left(1 - \lambda^2\right)^2}{90 \lambda^2}, \frac{\sqrt{1 - \lambda^2}}{12 \sqrt{7} \lambda} \} \frac{1}{S_F}$ and $\alpha = \frac{72 S^2_F \eta^2 }{Nb}$ , SLATE-M in \Cref{alg:2} has the following
\begin{align}
& \frac{1}{T}\sum_{t = 0}^{T - 1} \mathbb{E} \| \nabla \mathbf{F} (\bar{\mathbf{x}}_{t})\|^2 \leq \frac{2(\mathbf{F}(\bar{\mathbf{x}}_{0}) - \mathbf{F} ( \bar{\mathbf{x}}_{T}))}{\eta T}   + \frac{3 L_g^2 S_f^2 \sigma_g^2}{m} + 3 \frac{L_g^2 L_f^2 }{\alpha N B T}  \nonumber\\
&+ \frac{6 \alpha L_g^2 L_f^2}{Nb} + \frac{96 \lambda^2 L_g^2 L_f^2}{\left(1 - \lambda^2 \right)^3 B T} + \frac{256 \lambda^2 \alpha^2 L_f^2 L_g^2}{(1 - \lambda^2)^3} + \frac{64 \lambda^4 \mathbb{E} \left\|\nabla \hat{\mathbf{F}}_0\right\|^2}{(1 - \lambda^2)^3 N T}  \nonumber
\end{align}
\end{theorem}

\begin{corollary} \label{corollary:2}
Based on the analysis in the \cref{thm:2}, we choose $b = O(1), \eta = O(\frac{N^{2/3}}{T^{1/3}}), \alpha = O(\frac{N^{1/3}}{T^{2/3}}), B = O(\frac{T^{1/3}}{N^{2/3}})$, SLATE-M in \Cref{alg:2} has the
following
\begin{align}
&\frac{1}{T}\sum_{t = 0}^{T - 1} \mathbb{E} \| \nabla \mathbf{F} (\bar{\mathbf{x}}_{t})\|^2 \leq O(\frac{2(\mathbf{F}(\bar{\mathbf{x}}_{0}) - \mathbf{F} ( \bar{\mathbf{x}}_{T}))}{(N T)^{2/3}} + \frac{3 L_g^2 S_f^2 \sigma_g^2}{m} \nonumber\\
&+ O(\frac{3 L_g^2 L_f^2 }{(NT)^{2/3}})  + O(\frac{6 L_g^2 L_f^2}{(NT)^{2/3}}) + \frac{352 \lambda^2 L_f^2 L_g^2}{(1 - \lambda^2)^3} O(\frac{N^{2/3}}{T^{4/3}}) \nonumber\\
&+ \frac{64 \lambda^4 \mathbb{E} \left\|\nabla \hat{\mathbf{F}}_0\right\|^2}{(1 - \lambda^2)^3 N T}  
\end{align}
\end{corollary}
Based on the result in \cref{thm:2}, we can get the convergence result of SLATE-M. 
\begin{remark}
According to \Cref{corollary:2}, without loss of generality,  Let $m = O(\varepsilon^{-2})$, $b = O(1)$ $\eta = O(\frac{N^{2/3}}{T^{1/3}}), \alpha = O(\frac{N^{1/3}}{T^{2/3}})$, and $B = O(\frac{T^{1/3}}{N^{2/3}})$, we know to make $\frac{1}{T} \sum_{t=0}^{T - 1} \mathbb{E} \|\nabla F(\bar{\mathbf{x}}_{t})\|^2 \leq \varepsilon^{2}$, we have iterations $T$ should be as large as $ O( N^{-1} \varepsilon^{-3})$.

In \Cref{alg:2}, in each iteration, we sample $b + m$ data points to build the biased stochastic gradients $\mathbf{u}_{n,t}$, and need $T$ iterations. Thus, our SLATE-M algorithm has a sample complexity of   $m \cdot T = O(N^{-1} \varepsilon^{-5}) $, for finding an $\epsilon$-stationary point, which also achieves the linear speedup of our algorithm with respect to the number of worker nodes. 
\end{remark}
\begin{remark}
The sample complexity of  $O(N^{-1} \varepsilon^{-5})$ in SLATE-M matches the best convergence complexity achieved by the single-machine stochastic method for conditional stochastic optimization in the online setting, and also match the lower bound for the online stochastic algorithms \cite{hu2020biased}.
\end{remark}

\section{Experiments}
In this section, we conduct extensive experiments on imbalanced benchmark datasets to show the efficiency of our algorithms. All experiments are run over a machine with AMD EPYC 7513 32-Core Processors and NVIDIA RTX A6000 GPU. The source code is available at \textbf{https://github.com/xidongwu/D-AUPRC}.

The goal of our experiments is two-fold: (1) to verify that \eqref{eq:4} is the surrogate function of AUPRC and illustrate that directly optimizing the AUPRC in the multi-party collaborative training would improve the model performance compared with traditional loss optimization, and (2) to show the efficiency of our methods for AUPRC maximization. 

\begin{table*}[!t]
  \caption{Statistics of benchmark datasets } \label{tb1}
  \setlength{\tabcolsep}{12pt}
  \begin{tabular}{ccccc}
    \toprule
    Data Set &Training examples & Testing examples & Feature Size & Proportion of positive data\\
    \midrule
w7a & 24692 & 25057 & 300 & 2.99\%  \\
w8a & 49749 & 14951 & 300 & 2.97 \% \\
MNIST & 60000 & 10000 & $28\times 28$ & 16.7\% \\
Fashion MNIST  & 60000 & 10000 & $28\times 28$ & 16.7 \% \\
CIFAR-10 & 50000 & 10000 & $3 \times 32 \times 32$ & 16.7 \% \\
Tiny-ImageNet & 100000 & 10000 & $3 \times 64 \times 64$ & 16.7 \% \\
    \bottomrule
  \end{tabular}
\end{table*}

\begin{figure*}[!t]
     \centering
     \begin{subfigure}[b]{0.26\textwidth}
         \centering
         \includegraphics[width=\textwidth]{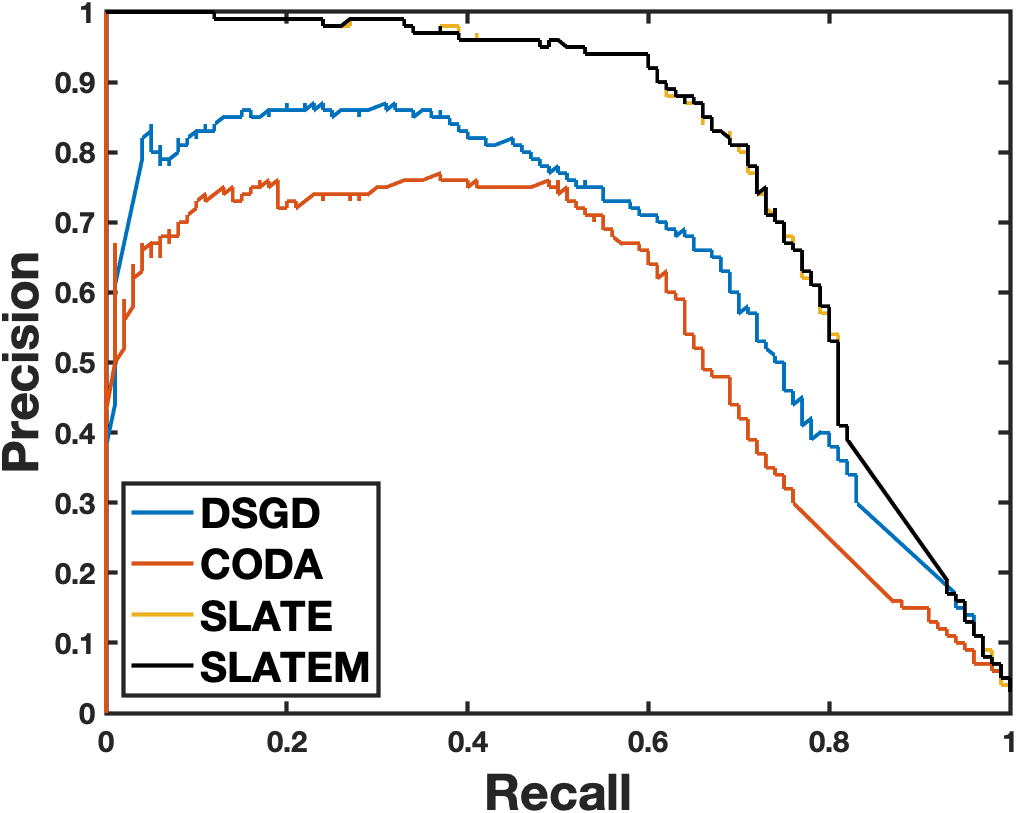}
         \caption{w7a dataset}
         \label{fig:1_1}
     \end{subfigure}
     \qquad
     \begin{subfigure}[b]{0.26\textwidth}
         \centering
         \includegraphics[width=\textwidth]{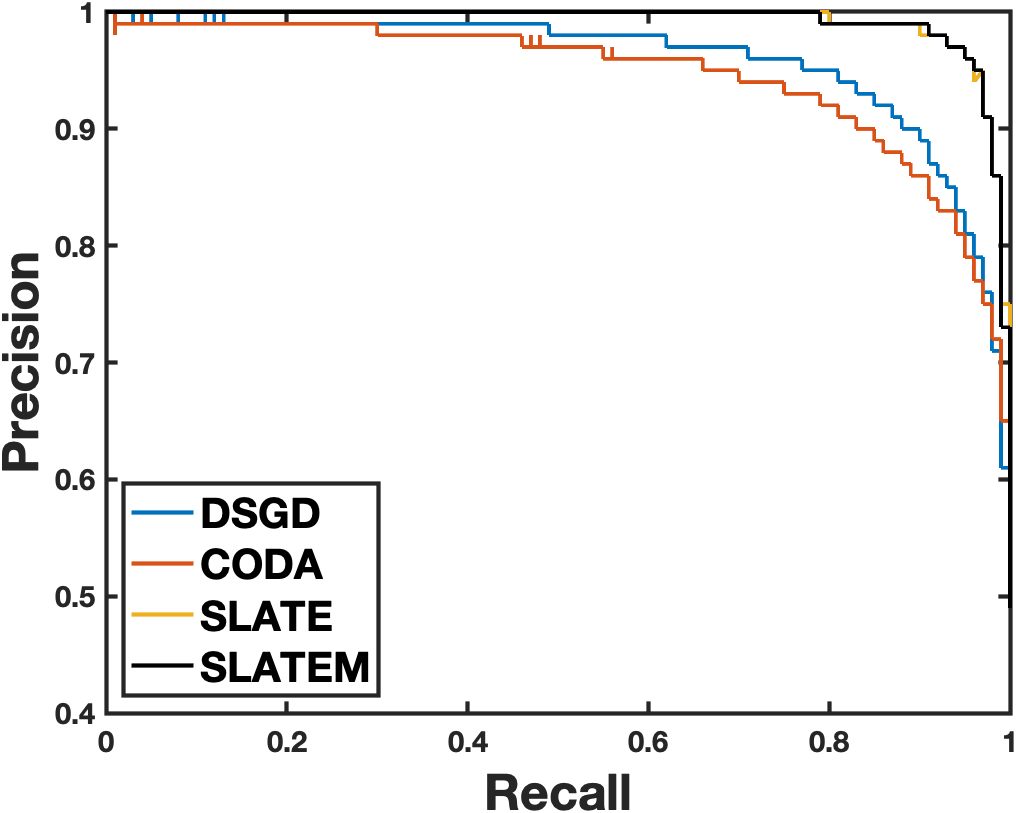}
         \caption{MNIST dataset}
         \label{fig:1:3}
     \end{subfigure}
     \qquad
     \begin{subfigure}[b]{0.26\textwidth}
         \centering
         \includegraphics[width=\textwidth]{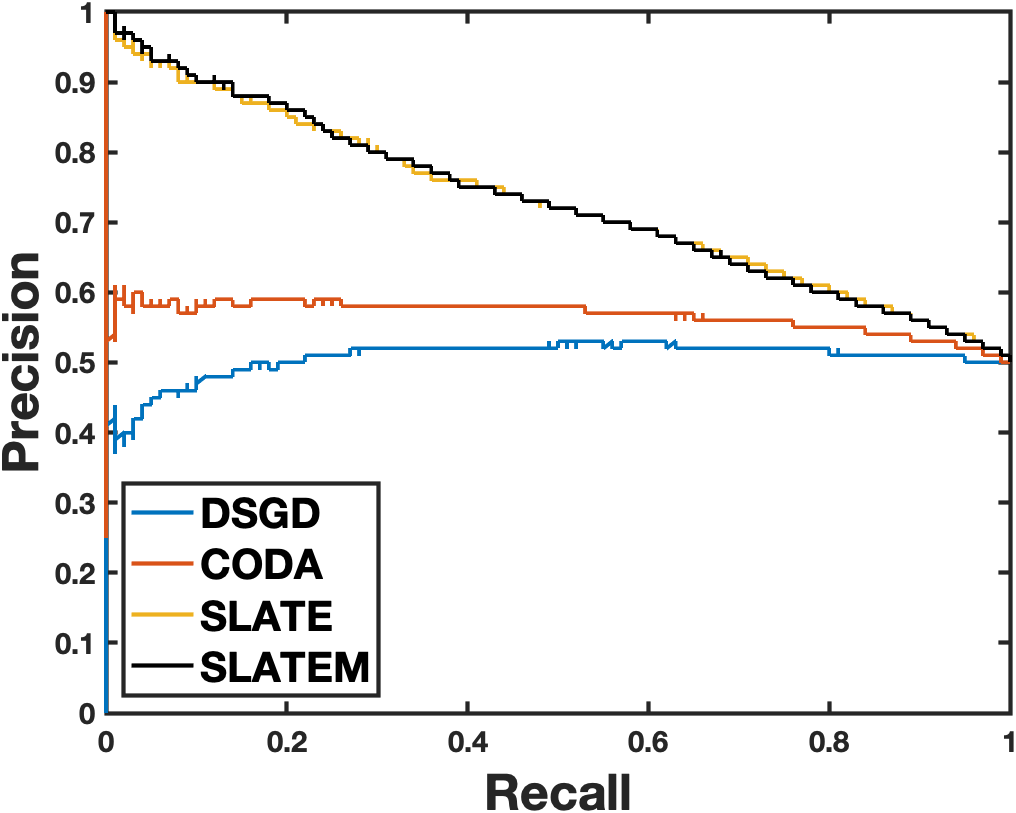}
         \caption{CIFAR-10 dataset}
         \label{fig:1_2}
     \end{subfigure}
     
        \caption{Precision-Recall curves of the models on the testing set}
        \label{fig:1}
\end{figure*}

\begin{figure*}[!t]
     \centering
     \begin{subfigure}[b]{0.26\textwidth}
         \centering
         \includegraphics[width=\textwidth]{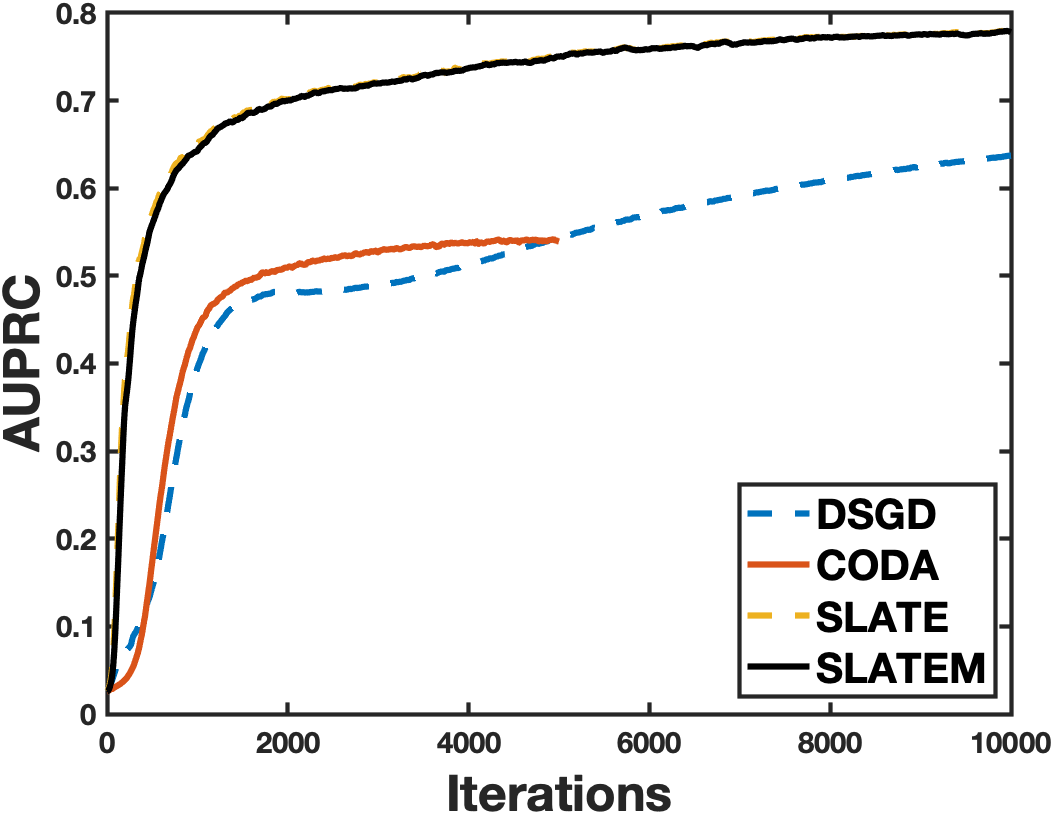}
         \caption{w7a  dataset}
         \label{phishing_1}
     \end{subfigure}
     \qquad
     \begin{subfigure}[b]{0.26\textwidth}
         \centering
         \includegraphics[width=\textwidth]{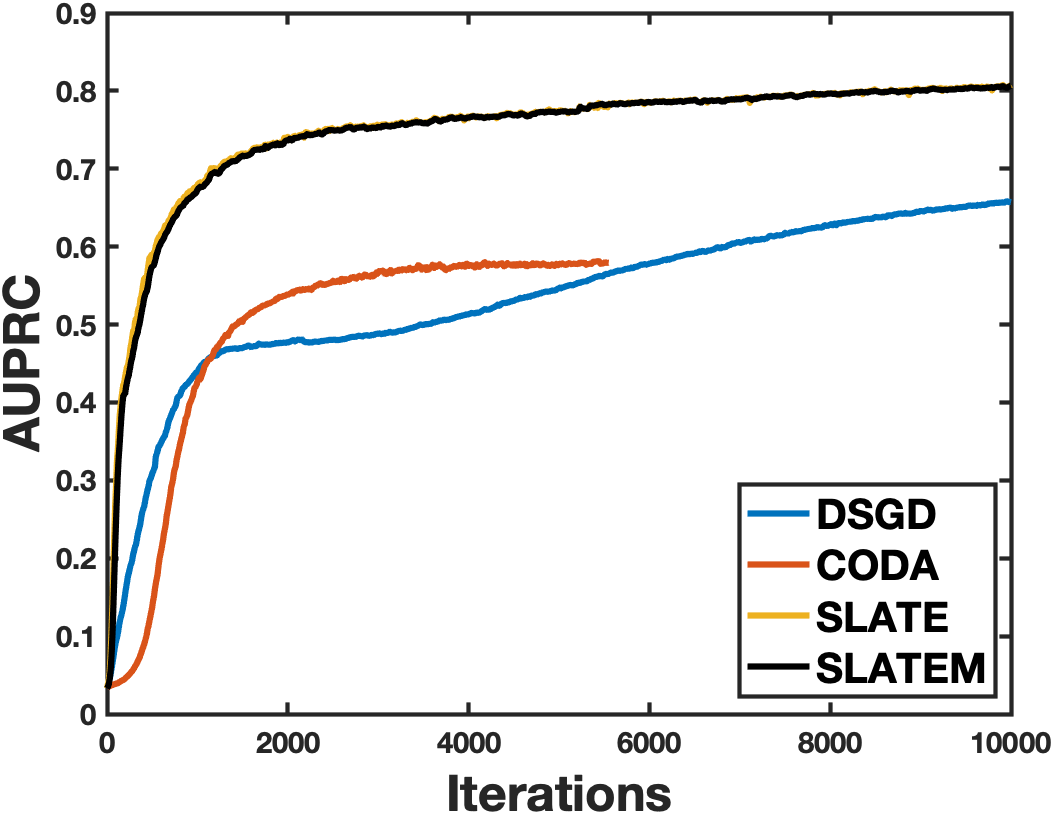}
         \caption{w8a dataset}
         \label{a6a_1}
     \end{subfigure}
     \qquad
     \begin{subfigure}[b]{0.26\textwidth}
         \centering
         \includegraphics[width=\textwidth]{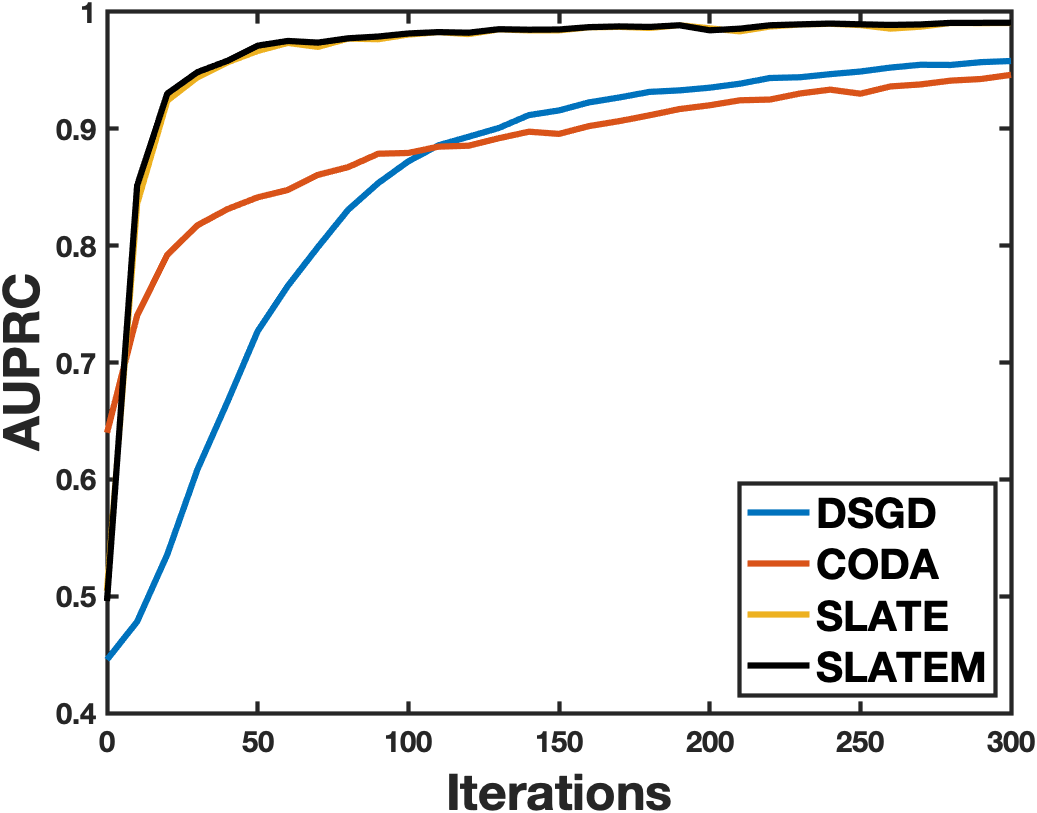}
         \caption{MNIST dataset}
         \label{a7a_1}
     \end{subfigure}
     
    \medskip
     \begin{subfigure}[b]{0.26\textwidth}
         \centering
         \includegraphics[width=\textwidth]{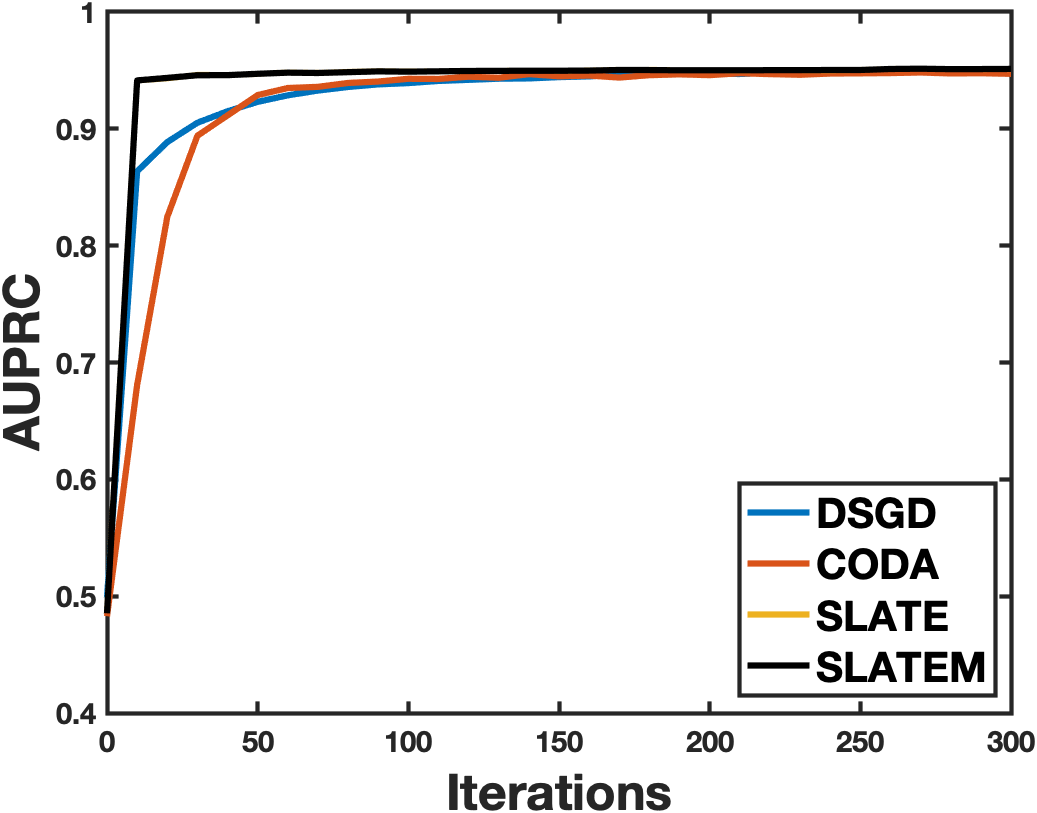}
         \caption{Fashion MNIST dataset}
         \label{w7a_1}
     \end{subfigure}
     \qquad
     \begin{subfigure}[b]{0.26\textwidth}
         \centering
         \includegraphics[width=\textwidth]{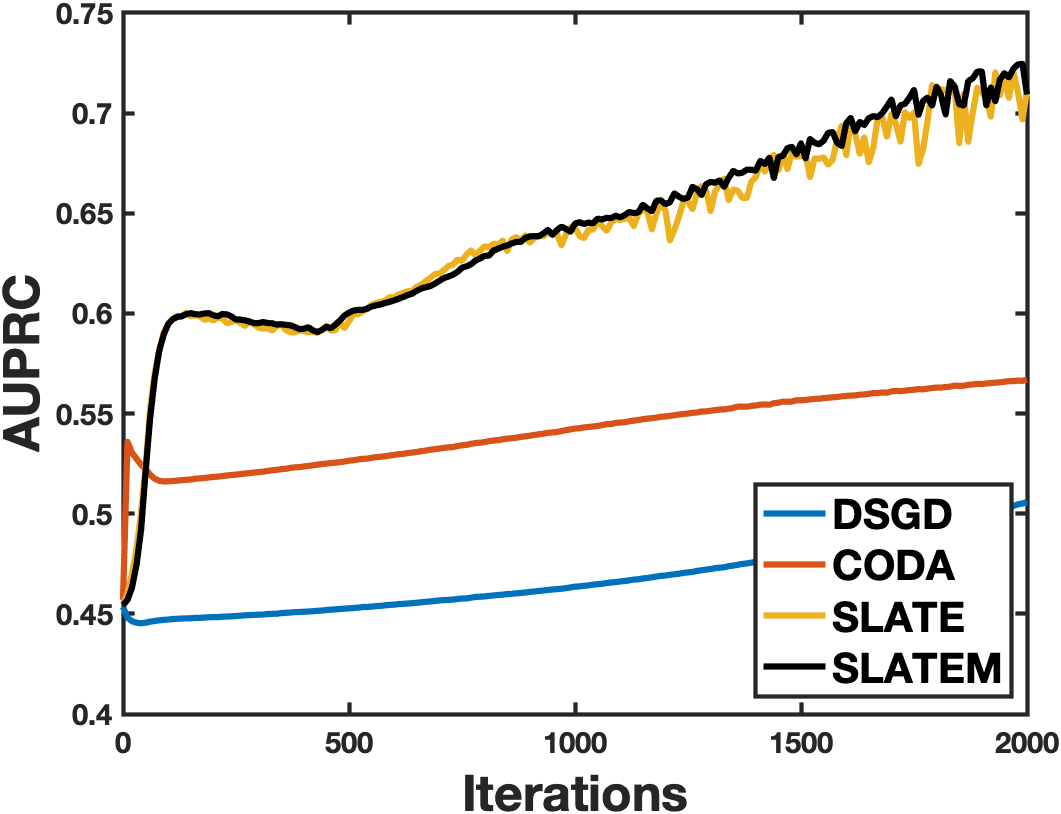}
         \caption{CIFAR-10 dataset}
         \label{w8a_1}
     \end{subfigure}
     \qquad
     \begin{subfigure}[b]{0.26\textwidth}
         \centering
         \includegraphics[width=\textwidth]{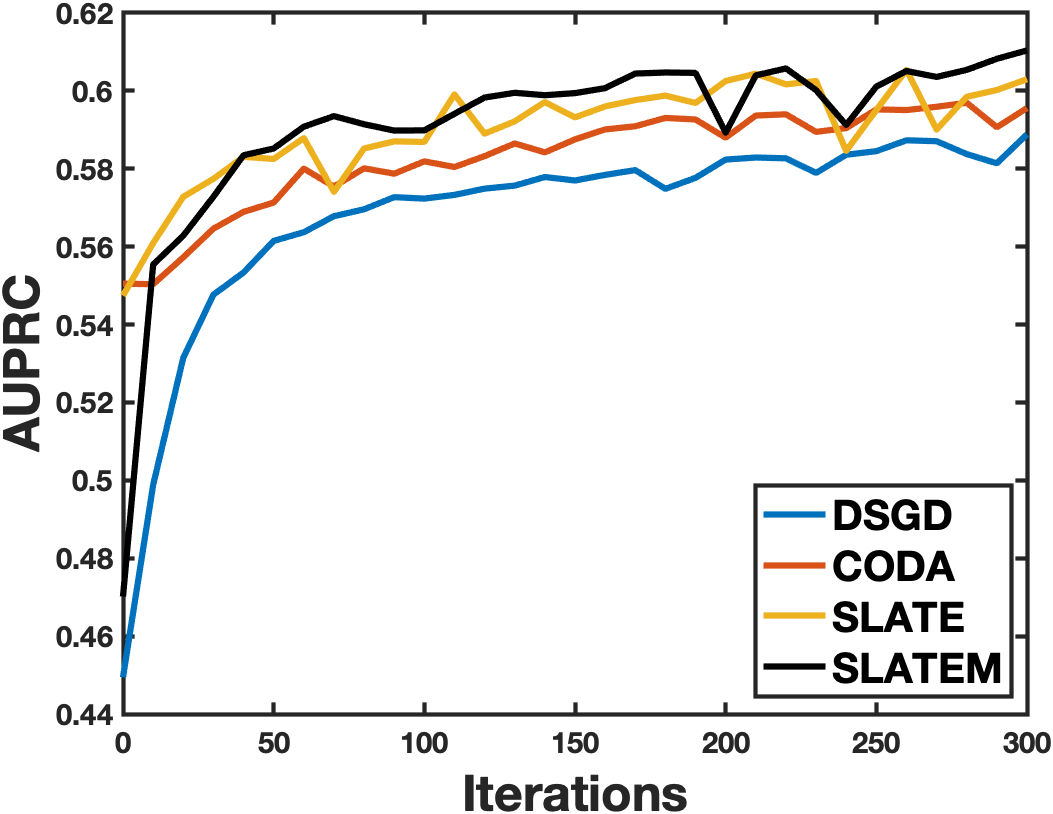}
         \caption{Tiny-ImageNet dataset}
         \label{covtype_1}
     \end{subfigure}
        \caption{AP vs Iterations on the test set } \label{fig:2}
\end{figure*}

\begin{table*}[!t]
  \caption{Final averaged AP scores on the testing data} \label{tb2}
  \setlength{\tabcolsep}{12pt}
  \centering
  \begin{tabular}{lcccccc}
    \hline
    Method & w7a & w8a & MNIST & Fashion MNIST & CIFAR-10 & Tiny-ImageNet\\
    \hline
D-PSGD & 0.6372 & 0.6585 & 0.9592 & 0.9497 & 0.5058 & 0.5906 \\
CODA   & 0.5414 & 0.5786 & 0.9460 & 0.9474 & 0.5668 & 0.5971 \\
SLATE  & 0.7788 & 0.8072 & 0.9911 & 0.9515 & 0.7279 & 0.6032  \\
SLATEM & 0.7778 & 0.8063 & 0.9913 & 0.9515 & 0.7285 & 0.6131 \\
    \hline
  \end{tabular}
\end{table*}
\subsection{Configurations}
\textbf{Datasets}:
We conduct experiments on imbalanced benchmark datasets from LIBSVM data \footnote{https://www.csie.ntu.edu.tw/~cjlin/libsvmtools/datasets/}: w7a and w8a, and four typical image datasets: MNIST dataset, Fashion-MNIST dataset, CIFAR-10, and Tiny-ImageNet dataset (seen in \Cref{tb1}). For w7a and w8a, we scale features to [0, 1]. For image datasets, following \citep{qi2021stochastic, wang2021momentum}, we construct the imbalanced binary-class versions as follows: Firstly, the first half of the classes (0 - 4) in the original MNIST, Fashion-MNIST, and CIFAR-10, and (0 - 99) in Tiny-ImageNet datasets are converted to be the negative class, and another half of classes are considered to be the positive class. Because the original distributions of image datasets are balanced, we randomly drop 80\% of the positive examples in the training set to make them imbalanced and keep test sets of image datasets unchanged. Finally, we evenly partition each datasets into disjoint sets and distribute datasets among worker nodes.

\textbf{Models}:
For w7a and w8a, we use two layers of neural networks with the dimension of the hidden layer as 28. The RELU is used as the activation function. For MNIST, Fashion MNIST and Cifar-10 data sets, we choose model architectures from \cite{ wu2022faster} for our imbalanced binary image classification task. For Tiny-ImageNet, we choose ResNet-18 \cite{he2016deep} as the classifier. In our algorithms, We modify the output of all models to 1 and the sigmoid function is followed since we consider binary classification tasks. 

In the experiments, the number of worker nodes is set as N = 20 and we use the ring-based topology as the communication network structure. \citep{lian2017can}.

\subsection{Comparison with Existing Multi-Party Stochastic Methods}

\textbf{Baselines}: We compare our algorithms with two baselines: 1) D-PSGD \cite{lian2017can}, a SGD-like serverless multi-party collaborative algorithm with the Cross-Entropy loss as the optimization objective. D-PSGD runs SGD locally and then computes the neighborhood weighted average by fetching model parameters from neighbors; 2) CODA, a typical federated learning algorithm for optimizing minimax formulated AUROC loss \cite{guo2020communication, yuan2021federated}. CODA runs local SGDA with the periodic model average in the federated learning setting. We convert it into the serverless multi-party setting and run local SGDA, following the consensus step to update the models. Gradient tracking steps are ignored. In the experiments, we ignore the gradient tracking steps to reduce computation and communication costs.

\noindent\textbf{Parameter tuning}:  We perform a grid search to tune all methods carefully. The total batch size m drawn from $\mathcal{D}$ is chosen in the set $\{20, 20, 20, 20, 60, 200\}$. For SLATE and SLATE-M, the positive batch size b in the total batch size is chosen in the set $\{2, 2, 3, 5, 20, 35\}$, and  m - b negative data points. The squared hinge loss is used as \eqref{eq:2}, $\alpha$ is chosen from $\{0.1, 0.9\}$ and the margin parameter s is selected from $\{0.1, 0.3, 0.5, 0.7, 0.9\}$. The step size is selected from the set $\{0.01, 0.005, 0.001\}$. For the D-PSGD, the step size is chosen in the set $\{0.01, 0.005, 0.001\}$. For the CODA, the 
step size for minimum variable is chose from the set $\{0.01, 0.005, 0.001\}$ and that for the maximum variable is chosen from the set $\{0.0001, 0.0005, 0.001\}$. Moreover, we use Xavier normal to initialize  models.
\begin{table}[!]
  \caption{SLATE test accuracy on CIFAR-10 with margin parameter $s$, and positive batch size b (B=60)} \label{tb3}
  \centering
  \begin{tabular}{cccccc}
    \hline
    Margin & 0.1 & 0.3 & 0.5 & 0.7 & 0.9\\
    \hline
b = 5 &  0.6270 & 0.5971 & 0.5992 & 0.5928 & 0.5765 \\
b = 10 & 0.6910 & 0.5887 & 0.5956 & 0.5986 & 0.5989 \\
b = 15 & 0.7248 & 0.5895 & 0.5902 & 0.5952 & 0.5979 \\
b = 20 & 0.7279 & 0.6001 & 0.5870 & 0.5929 & 0.5962\\
b = 25 & 0.7216 & 0.6038 & 0.5876 & 0.5933 & 0.5962\\
    \hline
  \end{tabular}
\end{table}

\begin{table}[!]
  \caption{SLATE-M test accuracy on CIFAR-10 with margin parameter $s$,  positive batch size b (B = 60), and  $\alpha$} \label{tb4}
  \centering
  \begin{tabular}{lcccc}
    \hline
    Margin & 0.1 ($\alpha$ = 0.1) & 0.1 ($\alpha$ = 0.9) & 0.3 ($\alpha$ = 0.1) & 0.3 ($\alpha$ = 0.9) \\
    \hline
b = 15 &  0.7273 & 0.7272 & 0.5853 & 0.5895 \\
b = 20 &  0.7285 & 0.7289 & 0.5986 & 0.6001 \\
b = 25 &  0.7167 & 0.7206 & 0.6024 & 0.6036 \\
    \hline
  \end{tabular}
\end{table}

\noindent\textbf{Experimental results}:
\Cref{tb2} summarizes the final results on the test sets. In order to present the advantage of optimization of AUPRC, we plot the Precision-Recall curves of final models on testing sets of W7A, MNIST, and CIFAR-10 when training stop in \Cref{fig:1}. Then we illustrate the convergence curve on test sets in \Cref{fig:2}. Results show that our algorithms (\emph{i.e.}, SLATE, SLATE-M) can outperform baselines in terms of AP with a great margin across each benchmark, regardless of model structure. The experiments verify that 1) the objective function \eqref{eq:4} is a good surrogate loss function of AUPRC and directly optimizing the AUPRC in the multi-party collaborative training would improve the model performance compared with traditional loss optimization in the imbalanced data mining. 2) Although CODA, with minimax formulated AUROC loss, has a relatively better performance compared with D-PSGD in large-scale datasets (CIFAR-10 and Tiny-ImageNet), the results verify the previous results that an algorithm that maximizes AUROC does not necessarily maximize AUPRC. Therefore, designing the algorithms for AUPRC in multi-party collaborative training is necessary. 3) our algorithms can efficiently optimize the \eqref{eq:4} and largely improve the performance in terms of AUPRC in multi-party collaborative imbalanced data mining. 4) In datasets CIFAR-10 and Tiny-ImageNet, SLATE-M has better performance compared with SLATE.

\noindent\textbf{Ablation study}: In this part, we study the effect of margin parameters, positive batch size, and $\alpha $ of SLATE-M. The results are listed in \Cref{tb3} and \Cref{tb4}.

\section{Conclusion}
In this paper, we systematically studied how to design serverless multi-party collaborative learning algorithms to directly maximize AUPRC and also provided the theoretical guarantee on algorithm convergence. To the best of our knowledge, this is the first work to optimize AUPRC in the multi-party collaborative training. We cast the AUPRC maximization problem into non-convex two-level stochastic optimization functions under the multi-party collaborative learning settings as the problem \eqref{eq:4}, and proposed the first multi-party collaborative learning algorithm, ServerLess biAsed sTochastic gradiEnt (SLATE). Theoretical analysis shows that SLATE has a sample complexity of $O(\varepsilon^{-6})$ and shows a linear speedup respective to the number of worker nodes. Furthermore, we proposed a stochastic method (\emph{i.e.}, SLATE-M) based on the momentum-based variance-reduced technique to reduce the convergence complexity for maximizing AP in multi-party collaborative optimization. Our methods reach iteration complexity of $O\left(1 / \epsilon^{5}\right)$, which matches the best convergence complexity achieved by the single-machine stochastic method for conditional stochastic optimization in the online setting, and also matches the lower bound for the online stochastic algorithms. Unlike existing single-machine methods that just focus on finite-sum settings and must keep an inner state for each positive data point, we consider the stochastic online setting. The extensive experiments on various data sets compared with previous stochastic multi-party collaborative optimization algorithms validate the effectiveness of our methods. Experimental results also demonstrate that directly optimizing the AUPRC in the multi-party collaborative training would largely improve the model performance compared with traditional loss optimization.


\newpage
\bibliographystyle{ACM-Reference-Format}
\bibliography{sample-base}


\begin{thebibliography}{66}


\ifx \showCODEN    \undefined \def \showCODEN     #1{\unskip}     \fi
\ifx \showDOI      \undefined \def \showDOI       #1{#1}\fi
\ifx \showISBNx    \undefined \def \showISBNx     #1{\unskip}     \fi
\ifx \showISBNxiii \undefined \def \showISBNxiii  #1{\unskip}     \fi
\ifx \showISSN     \undefined \def \showISSN      #1{\unskip}     \fi
\ifx \showLCCN     \undefined \def \showLCCN      #1{\unskip}     \fi
\ifx \shownote     \undefined \def \shownote      #1{#1}          \fi
\ifx \showarticletitle \undefined \def \showarticletitle #1{#1}   \fi
\ifx \showURL      \undefined \def \showURL       {\relax}        \fi
\providecommand\bibfield[2]{#2}
\providecommand\bibinfo[2]{#2}
\providecommand\natexlab[1]{#1}
\providecommand\showeprint[2][]{arXiv:#2}

\bibitem[Bamber(1975)]%
        {bamber1975area}
\bibfield{author}{\bibinfo{person}{Donald Bamber}.}
  \bibinfo{year}{1975}\natexlab{}.
\newblock \showarticletitle{The area above the ordinal dominance graph and the
  area below the receiver operating characteristic graph}.
\newblock \bibinfo{journal}{\emph{Journal of mathematical psychology}}
  \bibinfo{volume}{12}, \bibinfo{number}{4} (\bibinfo{year}{1975}),
  \bibinfo{pages}{387--415}.
\newblock


\bibitem[Bao et~al\mbox{.}(2022)]%
        {bao2022doubly}
\bibfield{author}{\bibinfo{person}{Runxue Bao}, \bibinfo{person}{Xidong Wu},
  \bibinfo{person}{Wenhan Xian}, {and} \bibinfo{person}{Heng Huang}.}
  \bibinfo{year}{2022}\natexlab{}.
\newblock \showarticletitle{Doubly sparse asynchronous learning for stochastic
  composite optimization}. In \bibinfo{booktitle}{\emph{Thirty-First
  International Joint Conference on Artificial Intelligence (IJCAI)}}.
  \bibinfo{pages}{1916--1922}.
\newblock


\bibitem[Boyd et~al\mbox{.}(2013)]%
        {boyd2013area}
\bibfield{author}{\bibinfo{person}{Kendrick Boyd}, \bibinfo{person}{Kevin~H
  Eng}, {and} \bibinfo{person}{C~David Page}.} \bibinfo{year}{2013}\natexlab{}.
\newblock \showarticletitle{Area under the precision-recall curve: point
  estimates and confidence intervals}. In \bibinfo{booktitle}{\emph{Joint
  European conference on machine learning and knowledge discovery in
  databases}}. Springer, \bibinfo{pages}{451--466}.
\newblock


\bibitem[Brock et~al\mbox{.}(2018)]%
        {brock2018large}
\bibfield{author}{\bibinfo{person}{Andrew Brock}, \bibinfo{person}{Jeff
  Donahue}, {and} \bibinfo{person}{Karen Simonyan}.}
  \bibinfo{year}{2018}\natexlab{}.
\newblock \showarticletitle{Large scale GAN training for high fidelity natural
  image synthesis}.
\newblock \bibinfo{journal}{\emph{arXiv preprint arXiv:1809.11096}}
  (\bibinfo{year}{2018}).
\newblock


\bibitem[Cortes and Mohri(2003)]%
        {cortes2003auc}
\bibfield{author}{\bibinfo{person}{Corinna Cortes} {and}
  \bibinfo{person}{Mehryar Mohri}.} \bibinfo{year}{2003}\natexlab{}.
\newblock \showarticletitle{AUC optimization vs. error rate minimization}.
\newblock \bibinfo{journal}{\emph{Advances in neural information processing
  systems}}  \bibinfo{volume}{16} (\bibinfo{year}{2003}).
\newblock


\bibitem[Cutkosky and Orabona(2019)]%
        {cutkosky2019momentum}
\bibfield{author}{\bibinfo{person}{Ashok Cutkosky} {and}
  \bibinfo{person}{Francesco Orabona}.} \bibinfo{year}{2019}\natexlab{}.
\newblock \showarticletitle{Momentum-based variance reduction in non-convex
  sgd}.
\newblock \bibinfo{journal}{\emph{Advances in neural information processing
  systems}}  \bibinfo{volume}{32} (\bibinfo{year}{2019}).
\newblock


\bibitem[Davis and Goadrich(2006)]%
        {davis2006relationship}
\bibfield{author}{\bibinfo{person}{Jesse Davis} {and} \bibinfo{person}{Mark
  Goadrich}.} \bibinfo{year}{2006}\natexlab{}.
\newblock \showarticletitle{The relationship between Precision-Recall and ROC
  curves}. In \bibinfo{booktitle}{\emph{Proceedings of the 23rd international
  conference on Machine learning}}. \bibinfo{pages}{233--240}.
\newblock


\bibitem[Dean et~al\mbox{.}(2012)]%
        {dean2012large}
\bibfield{author}{\bibinfo{person}{Jeffrey Dean}, \bibinfo{person}{Greg
  Corrado}, \bibinfo{person}{Rajat Monga}, \bibinfo{person}{Kai Chen},
  \bibinfo{person}{Matthieu Devin}, \bibinfo{person}{Mark Mao},
  \bibinfo{person}{Marc'aurelio Ranzato}, \bibinfo{person}{Andrew Senior},
  \bibinfo{person}{Paul Tucker}, \bibinfo{person}{Ke Yang}, {et~al\mbox{.}}}
  \bibinfo{year}{2012}\natexlab{}.
\newblock \showarticletitle{Large scale distributed deep networks}.
\newblock \bibinfo{journal}{\emph{Advances in neural information processing
  systems}}  \bibinfo{volume}{25} (\bibinfo{year}{2012}).
\newblock


\bibitem[Devlin et~al\mbox{.}(2018)]%
        {devlin2018bert}
\bibfield{author}{\bibinfo{person}{Jacob Devlin}, \bibinfo{person}{Ming-Wei
  Chang}, \bibinfo{person}{Kenton Lee}, {and} \bibinfo{person}{Kristina
  Toutanova}.} \bibinfo{year}{2018}\natexlab{}.
\newblock \showarticletitle{Bert: Pre-training of deep bidirectional
  transformers for language understanding}.
\newblock \bibinfo{journal}{\emph{arXiv preprint arXiv:1810.04805}}
  (\bibinfo{year}{2018}).
\newblock


\bibitem[Dou et~al\mbox{.}(2022a)]%
        {dou2022learning}
\bibfield{author}{\bibinfo{person}{Jason~Xiaotian Dou}, \bibinfo{person}{Minxue
  Jia}, \bibinfo{person}{Nika Zaslavsky}, \bibinfo{person}{Runxue Bao},
  \bibinfo{person}{Shiyi Zhang}, \bibinfo{person}{Ke Ni},
  \bibinfo{person}{Paul~Pu Liang}, \bibinfo{person}{Haiyi Mao}, {and}
  \bibinfo{person}{Zhihong Mao}.} \bibinfo{year}{2022}\natexlab{a}.
\newblock \showarticletitle{Learning More Effective Cell Representations
  Efficiently}. In \bibinfo{booktitle}{\emph{NeurIPS 2022 Workshop on Learning
  Meaningful Representations of Life}}.
\newblock


\bibitem[Dou et~al\mbox{.}(2022b)]%
        {dou2022sampling}
\bibfield{author}{\bibinfo{person}{Jason~Xiaotian Dou},
  \bibinfo{person}{Alvin~Qingkai Pan}, \bibinfo{person}{Runxue Bao},
  \bibinfo{person}{Haiyi~Harry Mao}, \bibinfo{person}{Lei Luo}, {and}
  \bibinfo{person}{Zhihong Mao}.} \bibinfo{year}{2022}\natexlab{b}.
\newblock \showarticletitle{Sampling through the lens of sequential decision
  making}.
\newblock \bibinfo{journal}{\emph{arXiv preprint arXiv:2208.08056}}
  (\bibinfo{year}{2022}).
\newblock


\bibitem[Goyal et~al\mbox{.}(2017)]%
        {goyal2017accurate}
\bibfield{author}{\bibinfo{person}{Priya Goyal}, \bibinfo{person}{Piotr
  Doll{\'a}r}, \bibinfo{person}{Ross Girshick}, \bibinfo{person}{Pieter
  Noordhuis}, \bibinfo{person}{Lukasz Wesolowski}, \bibinfo{person}{Aapo
  Kyrola}, \bibinfo{person}{Andrew Tulloch}, \bibinfo{person}{Yangqing Jia},
  {and} \bibinfo{person}{Kaiming He}.} \bibinfo{year}{2017}\natexlab{}.
\newblock \showarticletitle{Accurate, large minibatch sgd: Training imagenet in
  1 hour}.
\newblock \bibinfo{journal}{\emph{arXiv preprint arXiv:1706.02677}}
  (\bibinfo{year}{2017}).
\newblock


\bibitem[Guo et~al\mbox{.}(2022)]%
        {guo2023fedx}
\bibfield{author}{\bibinfo{person}{Zhishuai Guo}, \bibinfo{person}{Rong Jin},
  \bibinfo{person}{Jiebo Luo}, {and} \bibinfo{person}{Tianbao Yang}.}
  \bibinfo{year}{2022}\natexlab{}.
\newblock \showarticletitle{FeDXL: Provable Federated Learning for Deep X-Risk
  Optimization}.
\newblock \bibinfo{journal}{\emph{arXiv preprint arXiv:2210.14396}}
  (\bibinfo{year}{2022}).
\newblock


\bibitem[Guo et~al\mbox{.}(2020)]%
        {guo2020communication}
\bibfield{author}{\bibinfo{person}{Zhishuai Guo}, \bibinfo{person}{Mingrui
  Liu}, \bibinfo{person}{Zhuoning Yuan}, \bibinfo{person}{Li Shen},
  \bibinfo{person}{Wei Liu}, {and} \bibinfo{person}{Tianbao Yang}.}
  \bibinfo{year}{2020}\natexlab{}.
\newblock \showarticletitle{Communication-efficient distributed stochastic auc
  maximization with deep neural networks}. In
  \bibinfo{booktitle}{\emph{International Conference on Machine Learning}}.
  PMLR, \bibinfo{pages}{3864--3874}.
\newblock


\bibitem[He et~al\mbox{.}(2016)]%
        {he2016deep}
\bibfield{author}{\bibinfo{person}{Kaiming He}, \bibinfo{person}{Xiangyu
  Zhang}, \bibinfo{person}{Shaoqing Ren}, {and} \bibinfo{person}{Jian Sun}.}
  \bibinfo{year}{2016}\natexlab{}.
\newblock \showarticletitle{Deep residual learning for image recognition}. In
  \bibinfo{booktitle}{\emph{Proceedings of the IEEE conference on computer
  vision and pattern recognition}}. \bibinfo{pages}{770--778}.
\newblock


\bibitem[He et~al\mbox{.}(2023)]%
        {he2023robust}
\bibfield{author}{\bibinfo{person}{Sihong He}, \bibinfo{person}{Songyang Han},
  \bibinfo{person}{Sanbao Su}, \bibinfo{person}{Shuo Han},
  \bibinfo{person}{Shaofeng Zou}, {and} \bibinfo{person}{Fei Miao}.}
  \bibinfo{year}{2023}\natexlab{}.
\newblock \showarticletitle{Robust Multi-Agent Reinforcement Learning with
  State Uncertainty}.
\newblock \bibinfo{journal}{\emph{Transactions on Machine Learning Research}}
  (\bibinfo{year}{2023}).
\newblock


\bibitem[He et~al\mbox{.}(2020)]%
        {he2020data}
\bibfield{author}{\bibinfo{person}{Sihong He}, \bibinfo{person}{Lynn Pepin},
  \bibinfo{person}{Guang Wang}, \bibinfo{person}{Desheng Zhang}, {and}
  \bibinfo{person}{Fei Miao}.} \bibinfo{year}{2020}\natexlab{}.
\newblock \showarticletitle{Data-driven distributionally robust electric
  vehicle balancing for mobility-on-demand systems under demand and supply
  uncertainties}. In \bibinfo{booktitle}{\emph{2020 IEEE/RSJ International
  Conference on Intelligent Robots and Systems (IROS)}}. IEEE,
  \bibinfo{pages}{2165--2172}.
\newblock


\bibitem[Hu et~al\mbox{.}(2020)]%
        {hu2020biased}
\bibfield{author}{\bibinfo{person}{Yifan Hu}, \bibinfo{person}{Siqi Zhang},
  \bibinfo{person}{Xin Chen}, {and} \bibinfo{person}{Niao He}.}
  \bibinfo{year}{2020}\natexlab{}.
\newblock \showarticletitle{Biased stochastic first-order methods for
  conditional stochastic optimization and applications in meta learning}.
\newblock \bibinfo{journal}{\emph{Advances in Neural Information Processing
  Systems}}  \bibinfo{volume}{33} (\bibinfo{year}{2020}),
  \bibinfo{pages}{2759--2770}.
\newblock


\bibitem[Ji et~al\mbox{.}(2023)]%
        {ji2023prediction}
\bibfield{author}{\bibinfo{person}{Yuelyu Ji}, \bibinfo{person}{Yuhe Gao},
  \bibinfo{person}{Runxue Bao}, \bibinfo{person}{Qi Li},
  \bibinfo{person}{Disheng Liu}, \bibinfo{person}{Yiming Sun}, {and}
  \bibinfo{person}{Ye Ye}.} \bibinfo{year}{2023}\natexlab{}.
\newblock \showarticletitle{Prediction of COVID-19 Patients' Emergency Room
  Revisit using Multi-Source Transfer Learning}. In
  \bibinfo{booktitle}{\emph{2023 IEEE 11th International Conference on
  Healthcare Informatics (ICHI)}}. IEEE.
\newblock


\bibitem[Jiang et~al\mbox{.}(2022)]%
        {jiang2022multi}
\bibfield{author}{\bibinfo{person}{Wei Jiang}, \bibinfo{person}{Gang Li},
  \bibinfo{person}{Yibo Wang}, \bibinfo{person}{Lijun Zhang}, {and}
  \bibinfo{person}{Tianbao Yang}.} \bibinfo{year}{2022}\natexlab{}.
\newblock \showarticletitle{Multi-block-Single-probe Variance Reduced Estimator
  for Coupled Compositional Optimization}.
\newblock \bibinfo{journal}{\emph{arXiv preprint arXiv:2207.08540}}
  (\bibinfo{year}{2022}).
\newblock


\bibitem[Joachims(2005)]%
        {joachims2005support}
\bibfield{author}{\bibinfo{person}{Thorsten Joachims}.}
  \bibinfo{year}{2005}\natexlab{}.
\newblock \showarticletitle{A support vector method for multivariate
  performance measures}. In \bibinfo{booktitle}{\emph{Proceedings of the 22nd
  international conference on Machine learning}}.
\newblock


\bibitem[Li et~al\mbox{.}(2022)]%
        {li2022bridging}
\bibfield{author}{\bibinfo{person}{Jizhizi Li}, \bibinfo{person}{Jing Zhang},
  \bibinfo{person}{Stephen~J Maybank}, {and} \bibinfo{person}{Dacheng Tao}.}
  \bibinfo{year}{2022}\natexlab{}.
\newblock \showarticletitle{Bridging composite and real: towards end-to-end
  deep image matting}.
\newblock \bibinfo{journal}{\emph{International Journal of Computer Vision}}
  \bibinfo{volume}{130}, \bibinfo{number}{2} (\bibinfo{year}{2022}),
  \bibinfo{pages}{246--266}.
\newblock


\bibitem[Li et~al\mbox{.}(2023)]%
        {li2023referring}
\bibfield{author}{\bibinfo{person}{Jizhizi Li}, \bibinfo{person}{Jing Zhang},
  {and} \bibinfo{person}{Dacheng Tao}.} \bibinfo{year}{2023}\natexlab{}.
\newblock \showarticletitle{Referring image matting}. In
  \bibinfo{booktitle}{\emph{Proceedings of the IEEE/CVF conference on computer
  vision and pattern recognition}}. \bibinfo{pages}{22448--22457}.
\newblock


\bibitem[Li et~al\mbox{.}(2014)]%
        {li2014scaling}
\bibfield{author}{\bibinfo{person}{Mu Li}, \bibinfo{person}{David~G Andersen},
  \bibinfo{person}{Jun~Woo Park}, \bibinfo{person}{Alexander~J Smola},
  \bibinfo{person}{Amr Ahmed}, \bibinfo{person}{Vanja Josifovski},
  \bibinfo{person}{James Long}, \bibinfo{person}{Eugene~J Shekita}, {and}
  \bibinfo{person}{Bor-Yiing Su}.} \bibinfo{year}{2014}\natexlab{}.
\newblock \showarticletitle{Scaling distributed machine learning with the
  parameter server}. In \bibinfo{booktitle}{\emph{11th $\{$USENIX$\}$ Symposium
  on Operating Systems Design and Implementation ($\{$OSDI$\}$ 14)}}.
  \bibinfo{pages}{583--598}.
\newblock


\bibitem[Lian et~al\mbox{.}(2017)]%
        {lian2017can}
\bibfield{author}{\bibinfo{person}{Xiangru Lian}, \bibinfo{person}{Ce Zhang},
  \bibinfo{person}{Huan Zhang}, \bibinfo{person}{Cho-Jui Hsieh},
  \bibinfo{person}{Wei Zhang}, {and} \bibinfo{person}{Ji Liu}.}
  \bibinfo{year}{2017}\natexlab{}.
\newblock \showarticletitle{Can decentralized algorithms outperform centralized
  algorithms? a case study for decentralized parallel stochastic gradient
  descent}.
\newblock \bibinfo{journal}{\emph{Advances in Neural Information Processing
  Systems}}  \bibinfo{volume}{30} (\bibinfo{year}{2017}).
\newblock


\bibitem[Liu et~al\mbox{.}(2019)]%
        {liu2019stochastic}
\bibfield{author}{\bibinfo{person}{Mingrui Liu}, \bibinfo{person}{Zhuoning
  Yuan}, \bibinfo{person}{Yiming Ying}, {and} \bibinfo{person}{Tianbao Yang}.}
  \bibinfo{year}{2019}\natexlab{}.
\newblock \showarticletitle{Stochastic auc maximization with deep neural
  networks}.
\newblock \bibinfo{journal}{\emph{arXiv preprint arXiv:1908.10831}}
  (\bibinfo{year}{2019}).
\newblock


\bibitem[Liu et~al\mbox{.}(2020)]%
        {liu2020decentralized}
\bibfield{author}{\bibinfo{person}{Mingrui Liu}, \bibinfo{person}{Wei Zhang},
  \bibinfo{person}{Youssef Mroueh}, \bibinfo{person}{Xiaodong Cui},
  \bibinfo{person}{Jarret Ross}, \bibinfo{person}{Tianbao Yang}, {and}
  \bibinfo{person}{Payel Das}.} \bibinfo{year}{2020}\natexlab{}.
\newblock \showarticletitle{A decentralized parallel algorithm for training
  generative adversarial nets}.
\newblock \bibinfo{journal}{\emph{Advances in Neural Information Processing
  Systems}}  \bibinfo{volume}{33} (\bibinfo{year}{2020}),
  \bibinfo{pages}{11056--11070}.
\newblock


\bibitem[Liu et~al\mbox{.}(2022a)]%
        {liu2022fast}
\bibfield{author}{\bibinfo{person}{Weidong Liu}, \bibinfo{person}{Xiaojun Mao},
  {and} \bibinfo{person}{Xin Zhang}.} \bibinfo{year}{2022}\natexlab{a}.
\newblock \showarticletitle{Fast and Robust Sparsity Learning Over Networks: A
  Decentralized Surrogate Median Regression Approach}.
\newblock \bibinfo{journal}{\emph{IEEE Transactions on Signal Processing}}
  \bibinfo{volume}{70} (\bibinfo{year}{2022}), \bibinfo{pages}{797--809}.
\newblock


\bibitem[Liu et~al\mbox{.}(2022b)]%
        {liu2022interact}
\bibfield{author}{\bibinfo{person}{Zhuqing Liu}, \bibinfo{person}{Xin Zhang},
  \bibinfo{person}{Prashant Khanduri}, \bibinfo{person}{Songtao Lu}, {and}
  \bibinfo{person}{Jia Liu}.} \bibinfo{year}{2022}\natexlab{b}.
\newblock \showarticletitle{INTERACT: achieving low sample and communication
  complexities in decentralized bilevel learning over networks}. In
  \bibinfo{booktitle}{\emph{Proceedings of the Twenty-Third International
  Symposium on Theory, Algorithmic Foundations, and Protocol Design for Mobile
  Networks and Mobile Computing}}. \bibinfo{pages}{61--70}.
\newblock


\bibitem[Lu et~al\mbox{.}(2019)]%
        {lu2019gnsd}
\bibfield{author}{\bibinfo{person}{Songtao Lu}, \bibinfo{person}{Xinwei Zhang},
  \bibinfo{person}{Haoran Sun}, {and} \bibinfo{person}{Mingyi Hong}.}
  \bibinfo{year}{2019}\natexlab{}.
\newblock \showarticletitle{GNSD: A gradient-tracking based nonconvex
  stochastic algorithm for decentralized optimization}. In
  \bibinfo{booktitle}{\emph{2019 IEEE Data Science Workshop (DSW)}}. IEEE,
  \bibinfo{pages}{315--321}.
\newblock


\bibitem[Luo et~al\mbox{.}(2022)]%
        {luo2022multisource}
\bibfield{author}{\bibinfo{person}{Xiaoling Luo}, \bibinfo{person}{Xiaobo Ma},
  \bibinfo{person}{Matthew Munden}, \bibinfo{person}{Yao-Jan Wu}, {and}
  \bibinfo{person}{Yangsheng Jiang}.} \bibinfo{year}{2022}\natexlab{}.
\newblock \showarticletitle{A multisource data approach for estimating vehicle
  queue length at metered on-ramps}.
\newblock \bibinfo{journal}{\emph{Journal of Transportation Engineering, Part
  A: Systems}} \bibinfo{volume}{148}, \bibinfo{number}{2}
  (\bibinfo{year}{2022}), \bibinfo{pages}{04021117}.
\newblock


\bibitem[Ma(2022)]%
        {ma2022traffic}
\bibfield{author}{\bibinfo{person}{Xiaobo Ma}.}
  \bibinfo{year}{2022}\natexlab{}.
\newblock \emph{\bibinfo{title}{Traffic Performance Evaluation Using
  Statistical and Machine Learning Methods}}.
\newblock \bibinfo{thesistype}{Ph.\,D. Dissertation}. \bibinfo{school}{The
  University of Arizona}.
\newblock


\bibitem[Ma et~al\mbox{.}(2020)]%
        {ma2020statistical}
\bibfield{author}{\bibinfo{person}{Xiaobo Ma}, \bibinfo{person}{Abolfazl
  Karimpour}, {and} \bibinfo{person}{Yao-Jan Wu}.}
  \bibinfo{year}{2020}\natexlab{}.
\newblock \showarticletitle{Statistical evaluation of data requirement for ramp
  metering performance assessment}.
\newblock \bibinfo{journal}{\emph{Transportation Research Part A: Policy and
  Practice}}  \bibinfo{volume}{141} (\bibinfo{year}{2020}),
  \bibinfo{pages}{248--261}.
\newblock


\bibitem[McMahan et~al\mbox{.}(2017)]%
        {mcmahan2017communication}
\bibfield{author}{\bibinfo{person}{Brendan McMahan}, \bibinfo{person}{Eider
  Moore}, \bibinfo{person}{Daniel Ramage}, \bibinfo{person}{Seth Hampson},
  {and} \bibinfo{person}{Blaise~Aguera y Arcas}.}
  \bibinfo{year}{2017}\natexlab{}.
\newblock \showarticletitle{Communication-efficient learning of deep networks
  from decentralized data}. In \bibinfo{booktitle}{\emph{Artificial
  intelligence and statistics}}. PMLR.
\newblock


\bibitem[Mei et~al\mbox{.}(2023)]%
        {mei2023mac}
\bibfield{author}{\bibinfo{person}{Yongsheng Mei}, \bibinfo{person}{Hanhan
  Zhou}, \bibinfo{person}{Tian Lan}, \bibinfo{person}{Guru Venkataramani},
  {and} \bibinfo{person}{Peng Wei}.} \bibinfo{year}{2023}\natexlab{}.
\newblock \showarticletitle{MAC-PO: Multi-agent experience replay via
  collective priority optimization}.
\newblock \bibinfo{journal}{\emph{arXiv preprint arXiv:2302.10418}}
  (\bibinfo{year}{2023}).
\newblock


\bibitem[Pan et~al\mbox{.}(2020)]%
        {pan2020d}
\bibfield{author}{\bibinfo{person}{Taoxing Pan}, \bibinfo{person}{Jun Liu},
  {and} \bibinfo{person}{Jie Wang}.} \bibinfo{year}{2020}\natexlab{}.
\newblock \showarticletitle{D-SPIDER-SFO: A decentralized optimization
  algorithm with faster convergence rate for nonconvex problems}. In
  \bibinfo{booktitle}{\emph{Proceedings of the AAAI Conference on Artificial
  Intelligence}}, Vol.~\bibinfo{volume}{34}. \bibinfo{pages}{1619--1626}.
\newblock


\bibitem[Qi et~al\mbox{.}(2021)]%
        {qi2021stochastic}
\bibfield{author}{\bibinfo{person}{Qi Qi}, \bibinfo{person}{Youzhi Luo},
  \bibinfo{person}{Zhao Xu}, \bibinfo{person}{Shuiwang Ji}, {and}
  \bibinfo{person}{Tianbao Yang}.} \bibinfo{year}{2021}\natexlab{}.
\newblock \showarticletitle{Stochastic Optimization of Areas Under
  Precision-Recall Curves with Provable Convergence}.
\newblock \bibinfo{journal}{\emph{Advances in Neural Information Processing
  Systems}}  \bibinfo{volume}{34} (\bibinfo{year}{2021}).
\newblock


\bibitem[Sun et~al\mbox{.}(2022)]%
        {sun2022demystify}
\bibfield{author}{\bibinfo{person}{Jianhui Sun}, \bibinfo{person}{Mengdi Huai},
  \bibinfo{person}{Kishlay Jha}, {and} \bibinfo{person}{Aidong Zhang}.}
  \bibinfo{year}{2022}\natexlab{}.
\newblock \showarticletitle{Demystify hyperparameters for stochastic
  optimization with transferable representations}. In
  \bibinfo{booktitle}{\emph{Proceedings of the 28th ACM SIGKDD Conference on
  Knowledge Discovery and Data Mining}}. \bibinfo{pages}{1706--1716}.
\newblock


\bibitem[Sun et~al\mbox{.}(2023)]%
        {sun2023scheduling}
\bibfield{author}{\bibinfo{person}{Jianhui Sun}, \bibinfo{person}{Ying Yang},
  \bibinfo{person}{Guangxu Xun}, {and} \bibinfo{person}{Aidong Zhang}.}
  \bibinfo{year}{2023}\natexlab{}.
\newblock \showarticletitle{Scheduling Hyperparameters to Improve
  Generalization: From Centralized SGD to Asynchronous SGD}.
\newblock \bibinfo{journal}{\emph{ACM Transactions on Knowledge Discovery from
  Data}} \bibinfo{volume}{17}, \bibinfo{number}{2} (\bibinfo{year}{2023}).
\newblock


\bibitem[Tang et~al\mbox{.}(2018)]%
        {tang2018d}
\bibfield{author}{\bibinfo{person}{Hanlin Tang}, \bibinfo{person}{Xiangru
  Lian}, \bibinfo{person}{Ming Yan}, \bibinfo{person}{Ce Zhang}, {and}
  \bibinfo{person}{Ji Liu}.} \bibinfo{year}{2018}\natexlab{}.
\newblock \showarticletitle{$ D^2$: Decentralized training over decentralized
  data}. In \bibinfo{booktitle}{\emph{International Conference on Machine
  Learning}}. PMLR, \bibinfo{pages}{4848--4856}.
\newblock


\bibitem[Wang and Yang(2022)]%
        {why2022finite}
\bibfield{author}{\bibinfo{person}{Bokun Wang} {and} \bibinfo{person}{Tianbao
  Yang}.} \bibinfo{year}{2022}\natexlab{}.
\newblock \showarticletitle{Finite-Sum Coupled Compositional Stochastic
  Optimization: Theory and Applications}. In
  \bibinfo{booktitle}{\emph{International Conference on Machine Learning,
  {ICML} 2022, 17-23 July 2022, Baltimore, Maryland, {USA}}}
  \emph{(\bibinfo{series}{Proceedings of Machine Learning Research},
  Vol.~\bibinfo{volume}{162})}. \bibinfo{publisher}{PMLR},
  \bibinfo{pages}{23292--23317}.
\newblock


\bibitem[Wang et~al\mbox{.}(2021)]%
        {wang2021momentum}
\bibfield{author}{\bibinfo{person}{Guanghui Wang}, \bibinfo{person}{Ming Yang},
  \bibinfo{person}{Lijun Zhang}, {and} \bibinfo{person}{Tianbao Yang}.}
  \bibinfo{year}{2021}\natexlab{}.
\newblock \showarticletitle{Momentum Accelerates the Convergence of Stochastic
  AUPRC Maximization}.
\newblock \bibinfo{journal}{\emph{arXiv preprint arXiv:2107.01173}}
  (\bibinfo{year}{2021}).
\newblock


\bibitem[Wang et~al\mbox{.}(2023)]%
        {wang2023applications}
\bibfield{author}{\bibinfo{person}{Tongnian Wang}, \bibinfo{person}{Yan Du},
  \bibinfo{person}{Yanmin Gong}, \bibinfo{person}{Kim-Kwang~Raymond Choo},
  {and} \bibinfo{person}{Yuanxiong Guo}.} \bibinfo{year}{2023}\natexlab{}.
\newblock \showarticletitle{Applications of Federated Learning in Mobile
  Health: Scoping Review}.
\newblock \bibinfo{journal}{\emph{Journal of Medical Internet Research}}
  \bibinfo{volume}{25} (\bibinfo{year}{2023}), \bibinfo{pages}{e43006}.
\newblock


\bibitem[Wei et~al\mbox{.}(2021)]%
        {wei2021direct}
\bibfield{author}{\bibinfo{person}{Yadi Wei}, \bibinfo{person}{Rishit Sheth},
  {and} \bibinfo{person}{Roni Khardon}.} \bibinfo{year}{2021}\natexlab{}.
\newblock \showarticletitle{Direct loss minimization for sparse gaussian
  processes}. In \bibinfo{booktitle}{\emph{International Conference on
  Artificial Intelligence and Statistics}}. PMLR, \bibinfo{pages}{2566--2574}.
\newblock


\bibitem[Wu et~al\mbox{.}(2023)]%
        {wu2023decentralized}
\bibfield{author}{\bibinfo{person}{Xidong Wu}, \bibinfo{person}{Zhengmian Hu},
  {and} \bibinfo{person}{Heng Huang}.} \bibinfo{year}{2023}\natexlab{}.
\newblock \showarticletitle{Decentralized Riemannian Algorithm for Nonconvex
  Minimax Problems}. In \bibinfo{booktitle}{\emph{Proceedings of the AAAI
  Conference on Artificial Intelligence}}.
\newblock


\bibitem[Wu et~al\mbox{.}(2022c)]%
        {wu2022faster}
\bibfield{author}{\bibinfo{person}{Xidong Wu}, \bibinfo{person}{Feihu Huang},
  \bibinfo{person}{Zhengmian Hu}, {and} \bibinfo{person}{Heng Huang}.}
  \bibinfo{year}{2022}\natexlab{c}.
\newblock \showarticletitle{Faster Adaptive Federated Learning}.
\newblock \bibinfo{journal}{\emph{arXiv preprint arXiv:2212.00974}}
  (\bibinfo{year}{2022}).
\newblock


\bibitem[Wu et~al\mbox{.}(2022b)]%
        {wu2022fast}
\bibfield{author}{\bibinfo{person}{Xidong Wu}, \bibinfo{person}{Feihu Huang},
  {and} \bibinfo{person}{Heng Huang}.} \bibinfo{year}{2022}\natexlab{b}.
\newblock \showarticletitle{Fast Stochastic Recursive Momentum Methods for
  Imbalanced Data Mining}. In \bibinfo{booktitle}{\emph{2021 IEEE International
  Conference on Data Mining (ICDM)}}. IEEE.
\newblock


\bibitem[Wu et~al\mbox{.}(2022a)]%
        {wu2022adversarial}
\bibfield{author}{\bibinfo{person}{Yihan Wu}, \bibinfo{person}{Aleksandar
  Bojchevski}, {and} \bibinfo{person}{Heng Huang}.}
  \bibinfo{year}{2022}\natexlab{a}.
\newblock \showarticletitle{Adversarial Weight Perturbation Improves
  Generalization in Graph Neural Network}.
\newblock \bibinfo{journal}{\emph{arXiv preprint arXiv:2212.04983}}
  (\bibinfo{year}{2022}).
\newblock


\bibitem[Wu et~al\mbox{.}(2022d)]%
        {wu2022retrievalguard}
\bibfield{author}{\bibinfo{person}{Yihan Wu}, \bibinfo{person}{Hongyang Zhang},
  {and} \bibinfo{person}{Heng Huang}.} \bibinfo{year}{2022}\natexlab{d}.
\newblock \showarticletitle{RetrievalGuard: Provably Robust 1-Nearest Neighbor
  Image Retrieval}. In \bibinfo{booktitle}{\emph{International Conference on
  Machine Learning}}. PMLR, \bibinfo{pages}{24266--24279}.
\newblock


\bibitem[Xian et~al\mbox{.}(2021)]%
        {xian2021faster}
\bibfield{author}{\bibinfo{person}{Wenhan Xian}, \bibinfo{person}{Feihu Huang},
  \bibinfo{person}{Yanfu Zhang}, {and} \bibinfo{person}{Heng Huang}.}
  \bibinfo{year}{2021}\natexlab{}.
\newblock \showarticletitle{A faster decentralized algorithm for nonconvex
  minimax problems}.
\newblock \bibinfo{journal}{\emph{Advances in Neural Information Processing
  Systems}}  \bibinfo{volume}{34} (\bibinfo{year}{2021}),
  \bibinfo{pages}{25865--25877}.
\newblock


\bibitem[Xin et~al\mbox{.}(2021)]%
        {xin2021hybrid}
\bibfield{author}{\bibinfo{person}{Ran Xin}, \bibinfo{person}{Usman Khan},
  {and} \bibinfo{person}{Soummya Kar}.} \bibinfo{year}{2021}\natexlab{}.
\newblock \showarticletitle{A hybrid variance-reduced method for decentralized
  stochastic non-convex optimization}. In
  \bibinfo{booktitle}{\emph{International Conference on Machine Learning}}.
  PMLR, \bibinfo{pages}{11459--11469}.
\newblock


\bibitem[Ying et~al\mbox{.}(2016)]%
        {ying2016stochastic}
\bibfield{author}{\bibinfo{person}{Yiming Ying}, \bibinfo{person}{Longyin Wen},
  {and} \bibinfo{person}{Siwei Lyu}.} \bibinfo{year}{2016}\natexlab{}.
\newblock \showarticletitle{Stochastic online AUC maximization}.
\newblock \bibinfo{journal}{\emph{Advances in neural information processing
  systems}}  \bibinfo{volume}{29} (\bibinfo{year}{2016}).
\newblock


\bibitem[Yuan et~al\mbox{.}(2016)]%
        {yuan2016convergence}
\bibfield{author}{\bibinfo{person}{Kun Yuan}, \bibinfo{person}{Qing Ling},
  {and} \bibinfo{person}{Wotao Yin}.} \bibinfo{year}{2016}\natexlab{}.
\newblock \showarticletitle{On the convergence of decentralized gradient
  descent}.
\newblock \bibinfo{journal}{\emph{SIAM Journal on Optimization}}
  \bibinfo{volume}{26}, \bibinfo{number}{3} (\bibinfo{year}{2016}),
  \bibinfo{pages}{1835--1854}.
\newblock


\bibitem[Yuan et~al\mbox{.}(2022)]%
        {yuan2022compositional}
\bibfield{author}{\bibinfo{person}{Zhuoning Yuan}, \bibinfo{person}{Zhishuai
  Guo}, \bibinfo{person}{Nitesh Chawla}, {and} \bibinfo{person}{Tianbao Yang}.}
  \bibinfo{year}{2022}\natexlab{}.
\newblock \showarticletitle{Compositional training for end-to-end deep AUC
  maximization}. In \bibinfo{booktitle}{\emph{International Conference on
  Learning Representations}}.
\newblock


\bibitem[Yuan et~al\mbox{.}(2021)]%
        {yuan2021federated}
\bibfield{author}{\bibinfo{person}{Zhuoning Yuan}, \bibinfo{person}{Zhishuai
  Guo}, \bibinfo{person}{Yi Xu}, \bibinfo{person}{Yiming Ying}, {and}
  \bibinfo{person}{Tianbao Yang}.} \bibinfo{year}{2021}\natexlab{}.
\newblock \showarticletitle{Federated deep AUC maximization for hetergeneous
  data with a constant communication complexity}. In
  \bibinfo{booktitle}{\emph{International Conference on Machine Learning}}.
  PMLR, \bibinfo{pages}{12219--12229}.
\newblock


\bibitem[Zhang et~al\mbox{.}(2022d)]%
        {zhang2022many}
\bibfield{author}{\bibinfo{person}{Hongyang Zhang}, \bibinfo{person}{Yihan Wu},
  {and} \bibinfo{person}{Heng Huang}.} \bibinfo{year}{2022}\natexlab{d}.
\newblock \showarticletitle{How Many Data Are Needed for Robust Learning?}
\newblock \bibinfo{journal}{\emph{arXiv preprint arXiv:2202.11592}}
  (\bibinfo{year}{2022}).
\newblock


\bibitem[Zhang et~al\mbox{.}(2020)]%
        {zhang2020private}
\bibfield{author}{\bibinfo{person}{Xin Zhang}, \bibinfo{person}{Minghong Fang},
  \bibinfo{person}{Jia Liu}, {and} \bibinfo{person}{Zhengyuan Zhu}.}
  \bibinfo{year}{2020}\natexlab{}.
\newblock \showarticletitle{Private and communication-efficient edge learning:
  a sparse differential gaussian-masking distributed SGD approach}. In
  \bibinfo{booktitle}{\emph{Proceedings of the Twenty-First International
  Symposium on Theory, Algorithmic Foundations, and Protocol Design for Mobile
  Networks and Mobile Computing}}. \bibinfo{pages}{261--270}.
\newblock


\bibitem[Zhang et~al\mbox{.}(2022b)]%
        {zhang2022net}
\bibfield{author}{\bibinfo{person}{Xin Zhang}, \bibinfo{person}{Minghong Fang},
  \bibinfo{person}{Zhuqing Liu}, \bibinfo{person}{Haibo Yang},
  \bibinfo{person}{Jia Liu}, {and} \bibinfo{person}{Zhengyuan Zhu}.}
  \bibinfo{year}{2022}\natexlab{b}.
\newblock \showarticletitle{Net-fleet: Achieving linear convergence speedup for
  fully decentralized federated learning with heterogeneous data}. In
  \bibinfo{booktitle}{\emph{Proceedings of the Twenty-Third International
  Symposium on Theory, Algorithmic Foundations, and Protocol Design for Mobile
  Networks and Mobile Computing}}. \bibinfo{pages}{71--80}.
\newblock


\bibitem[Zhang et~al\mbox{.}(2022c)]%
        {zhang2022learning}
\bibfield{author}{\bibinfo{person}{Xin Zhang}, \bibinfo{person}{Jia Liu}, {and}
  \bibinfo{person}{Zhengyuan Zhu}.} \bibinfo{year}{2022}\natexlab{c}.
\newblock \showarticletitle{Learning Coefficient Heterogeneity over Networks: A
  Distributed Spanning-Tree-Based Fused-Lasso Regression}.
\newblock \bibinfo{journal}{\emph{J. Amer. Statist. Assoc.}}
  (\bibinfo{year}{2022}), \bibinfo{pages}{1--13}.
\newblock


\bibitem[Zhang et~al\mbox{.}(2019)]%
        {zhang2019compressed}
\bibfield{author}{\bibinfo{person}{Xin Zhang}, \bibinfo{person}{Jia Liu},
  \bibinfo{person}{Zhengyuan Zhu}, {and} \bibinfo{person}{Elizabeth~S
  Bentley}.} \bibinfo{year}{2019}\natexlab{}.
\newblock \showarticletitle{Compressed distributed gradient descent:
  Communication-efficient consensus over networks}. In
  \bibinfo{booktitle}{\emph{IEEE INFOCOM 2019-IEEE Conference on Computer
  Communications}}. IEEE, \bibinfo{pages}{2431--2439}.
\newblock


\bibitem[Zhang et~al\mbox{.}(2021b)]%
        {zhang2021gt}
\bibfield{author}{\bibinfo{person}{Xin Zhang}, \bibinfo{person}{Jia Liu},
  \bibinfo{person}{Zhengyuan Zhu}, {and} \bibinfo{person}{Elizabeth~Serena
  Bentley}.} \bibinfo{year}{2021}\natexlab{b}.
\newblock \showarticletitle{Gt-storm: Taming sample, communication, and memory
  complexities in decentralized non-convex learning}. In
  \bibinfo{booktitle}{\emph{Proceedings of the Twenty-second International
  Symposium on Theory, Algorithmic Foundations, and Protocol Design for Mobile
  Networks and Mobile Computing}}. \bibinfo{pages}{271--280}.
\newblock


\bibitem[Zhang et~al\mbox{.}(2021c)]%
        {zhang2021low}
\bibfield{author}{\bibinfo{person}{Xin Zhang}, \bibinfo{person}{Jia Liu},
  \bibinfo{person}{Zhengyuan Zhu}, {and} \bibinfo{person}{Elizabeth~Serena
  Bentley}.} \bibinfo{year}{2021}\natexlab{c}.
\newblock \showarticletitle{Low Sample and Communication Complexities in
  Decentralized Learning: A Triple Hybrid Approach}. In
  \bibinfo{booktitle}{\emph{IEEE INFOCOM 2021-IEEE Conference on Computer
  Communications}}. IEEE, \bibinfo{pages}{1--10}.
\newblock


\bibitem[Zhang et~al\mbox{.}(2021a)]%
        {zhang2021taming}
\bibfield{author}{\bibinfo{person}{Xin Zhang}, \bibinfo{person}{Zhuqing Liu},
  \bibinfo{person}{Jia Liu}, \bibinfo{person}{Zhengyuan Zhu}, {and}
  \bibinfo{person}{Songtao Lu}.} \bibinfo{year}{2021}\natexlab{a}.
\newblock \showarticletitle{Taming Communication and Sample Complexities in
  Decentralized Policy Evaluation for Cooperative Multi-Agent Reinforcement
  Learning}.
\newblock \bibinfo{journal}{\emph{Advances in Neural Information Processing
  Systems}}  \bibinfo{volume}{34} (\bibinfo{year}{2021}),
  \bibinfo{pages}{18825--18838}.
\newblock


\bibitem[Zhang et~al\mbox{.}(2023)]%
        {zhang2023federated}
\bibfield{author}{\bibinfo{person}{Xinwen Zhang}, \bibinfo{person}{Yihan
  Zhang}, \bibinfo{person}{Tianbao Yang}, \bibinfo{person}{Richard Souvenir},
  {and} \bibinfo{person}{Hongchang Gao}.} \bibinfo{year}{2023}\natexlab{}.
\newblock \showarticletitle{Federated Compositional Deep AUC Maximization}.
\newblock \bibinfo{journal}{\emph{arXiv preprint arXiv:2304.10101}}
  (\bibinfo{year}{2023}).
\newblock


\bibitem[Zhang et~al\mbox{.}(2022a)]%
        {zhang2022toward}
\bibfield{author}{\bibinfo{person}{Yanfu Zhang}, \bibinfo{person}{Runxue Bao},
  \bibinfo{person}{Jian Pei}, {and} \bibinfo{person}{Heng Huang}.}
  \bibinfo{year}{2022}\natexlab{a}.
\newblock \showarticletitle{Toward Unified Data and Algorithm Fairness via
  Adversarial Data Augmentation and Adaptive Model Fine-tuning}. In
  \bibinfo{booktitle}{\emph{2022 IEEE International Conference on Data Mining
  (ICDM)}}. IEEE.
\newblock


\bibitem[Zhao et~al\mbox{.}(2011)]%
        {zhao2011online}
\bibfield{author}{\bibinfo{person}{Peilin Zhao}, \bibinfo{person}{Steven~CH
  Hoi}, \bibinfo{person}{Rong Jin}, {and} \bibinfo{person}{Tianbo YANG}.}
  \bibinfo{year}{2011}\natexlab{}.
\newblock \showarticletitle{Online AUC maximization}.
\newblock  (\bibinfo{year}{2011}).
\newblock


\end{thebibliography}


\newpage
\appendix
\onecolumn
\section{Supplementary material}



\subsection{Basic Lemma}
We draw one sample $\xi_n$ from $\mathcal{D}^{+}_n$ and $m$ sample $\xi^{\prime}_n$ from $\mathcal{D}_n$ as $\mathcal{B}_{n}$. We define 
\begin{align}
g_n(x; \mathbf{\xi}_n) = \mathbb{E}_{\mathbf{\xi}_n^{\prime} \sim \mathcal{D}_n} g_n \left(x; \mathbf{\xi}_n, \mathbf{\xi}_n^{\prime} \right) \nonumber
\end{align} 
\begin{align}
\hat{g}_n(x, \xi_n) = \frac{1}{m} \sum_{\xi_n^{\prime} \in \mathcal{B}_{n}} g_n(x, \xi_n; \xi_n^{\prime}) \nonumber
\end{align}
\begin{align}
\hat{F}_n(x; \xi_n, \mathcal{B}_{n}) = f(\hat{g}_n(x, \xi_n)) \nonumber
\end{align}

We also define 
\begin{align}
&\nabla F_n (x) = \nabla \mathbb{E}_{\xi_n} f(g_n(x, \xi_n))  = \mathbb{E}_{\xi_n}\left[\nabla\left(f(g_n(x, \xi_n))\right)\right] 
= \mathbb{E}_{\xi_n}\left[\nabla f(g_n(x, \xi_n)) \cdot \nabla g_n(x, \xi_n)\right] \nonumber
\end{align}
\begin{align}
\hat{F}_n(x) = \mathbb{E}\hat{F}_n(x; \xi_n, \mathcal{B}_{n}) = \mathbb{E} \left[f(\hat{g}_n(x, \xi_n))\right] \nonumber
\end{align}
\begin{align}
\nabla \hat{F}_n(x) = \mathbb{E}\nabla\hat{F}_n(x; \xi_n, \mathcal{B}_{n})
\end{align}

For convenience, we denote 
\begin{align} \bar{\mathbf{x}}_t = \frac{1}{N} \sum_{n=1}^N \mathbf{x}_{n, t}, \quad \bar{\mathbf{v}}_t = \frac{1}{N} \sum_{n=1}^N \mathbf{v}_{n,t},  \quad \bar{\mathbf{u}}_t = \frac{1}{N} \sum_{n=1}^N \mathbf{u}_{n,t} \quad F(\bar{\mathbf{x}}) = \frac{1}{N} \sum_{n = 1}^{N} F_n (\bar{\mathbf{x}}), \nabla \hat{F}(\bar{\mathbf{x}}) = \frac{1}{N} \sum_{n = 1}^{N} \nabla \hat{F}_n(\bar{\mathbf{x}}) 
\end{align}
and $\nabla \hat{\mathbf{F}}_t = [\nabla \hat{F}_1(\mathbf{x}_1)^{\top}, \nabla \hat{F}_2(\mathbf{x}_2)^{\top}, \dots, \nabla \hat{F}_N(\mathbf{x}_N)^{\top}]^{\top} \in \mathbb{R}^{n d}$.

\begin{lemma} (Lemma 6 in \cite{xin2021hybrid}) \label{lem:A3}
Let $\{ V_t \}_{t \geq 0}, \{ R_t \}_{t\geq 0}$ and $\{Q_t \}_{t\geq 0}$ be non-negative sequences and $C \geq 0$ be some constant such that $V_t \leq q V_{t-1}+q R_{t-1}+Q_t+C, \forall t \geq 1$, where $q \in(0,1)$ Then the following inequality holds: $\forall T \geq 1$,
\begin{align}
\sum_{t=0}^{T-1} V_t \leq \frac{V_0}{1-q}+\frac{1}{1-q} \sum_{t=0}^{T-2} R_t+\frac{1}{1-q} \sum_{t=1}^{T-1} Q_t + \frac{C T}{1-q}  
\end{align}
\end{lemma}


\section{SLATE}
\subsection{Proofs of the Intermediate Lemmas}
\begin{lemma} \label{lem:B1}
Let Assumptions \ref{ass:1}, \ref{ass:2} hold, and $F$ is $S_F$-smooth, we have 
\begin{align}
\mathbb{E} F(\bar{\mathbf{x}}_{t+1}) &  \leq \mathbb{E} F(\bar{\mathbf{x}}_{t}) - (\frac{\eta}{2} - \eta^2 S_F ) \mathbb{E} \|\nabla \hat{F}(\bar{\mathbf{x}}_t) \|^2 - \frac{\eta}{2} \mathbb{E}\| \nabla F(\bar{\mathbf{x}}_{t})\|^2 \nonumber\\
&+ \frac{\eta S_F^2}{N} \sum_{n=1}^{N}\|x_{n,t} - \bar{\mathbf{x}}_t\|^2 + \frac{\eta L_g^2 S_f^2 \sigma_g^2}{m} + \frac{\eta^2 S_F L_f^2 L_g^2}{N}
\end{align}
\end{lemma}
\begin{proof}
Based on the smoothness of F, we have 
\begin{align}
\mathbb{E} F(\bar{\mathbf{x}}_{t+1}) & \stackrel{(a)}{\leq}  \mathbb{E} F(\bar{\mathbf{x}}_{t}) + \mathbb{E}\langle\nabla F(\bar{\mathbf{x}}_{t}), \bar{\mathbf{x}}_{t+1} - \bar{\mathbf{x}}_{t}\rangle + \frac{S_F}{2}\mathbb{E}\left\|\bar{\mathbf{x}}_{t+1} - \bar{\mathbf{x}}_{t}\right\|^{2}  \nonumber\\
&\stackrel{(b)}{\leq} \mathbb{E}F(\bar{\mathbf{x}}_{t}) - \eta \mathbb{E} \langle\nabla F(\bar{\mathbf{x}}_{t}), \bar{\mathbf{u}}_{t}\rangle + \frac{\eta^2 S_F }{2} \mathbb{E} \|\bar{\mathbf{u}}_{t}\|^{2} \nonumber\\
&\stackrel{(c)}{\leq} \mathbb{E}F(\bar{\mathbf{x}}_{t}) - \eta \mathbb{E} \langle\nabla F(\bar{\mathbf{x}}_{t}), \frac{1}{N} \sum_{n=1}^{N} \nabla \hat{F}_n(\mathbf{x}_{n, t})\rangle + \eta^2 S_F\mathbb{E} [\|\bar{\mathbf{u}}_{t} -  \frac{1}{N} \sum_{n=1}^{N} \nabla \hat{F}_n(\mathbf{x}_{n, t})\|^{2} + \|\frac{1}{N} \sum_{n=1}^{N} \nabla \hat{F}_n(\mathbf{x}_{n, t})\|^2] \nonumber\\
&\stackrel{(d)}{=} \mathbb{E} F(\bar{\mathbf{x}}_{t}) - (\frac{\eta}{2} - \eta^2 S_F ) \mathbb{E}\| \frac{1}{N} \sum_{n=1}^{N} \nabla \hat{F}_n(\mathbf{x}_{n, t})\|^2 - \frac{\eta}{2} \mathbb{E} \| \nabla F(\bar{\mathbf{x}}_{t})\|^2 + \frac{\eta}{2} \mathbb{E}\|\nabla F(\bar{\mathbf{x}}_{t}) - \frac{1}{N} \sum_{n=1}^{N} \nabla \hat{F}_n(\mathbf{x}_{n, t})\|^2 + \eta^2 S_F \mathbb{E} \|\bar{\mathbf{u}}_{t} -  \frac{1}{N} \sum_{n=1}^{N} \nabla \hat{F}_n(\mathbf{x}_{n, t})\|^{2} \nonumber\\
&\leq  \mathbb{E} F(\bar{\mathbf{x}}_{t}) - (\frac{\eta}{2} - \eta^2 S_F ) \| \frac{1}{N} \sum_{n=1}^{N} \nabla \hat{F}_n(\mathbf{x}_{n, t}) \|^2 - \frac{\eta}{2} \mathbb{E}\| \nabla F(\bar{\mathbf{x}}_{t})\|^2 + \eta \mathbb{E}\|\nabla F(\bar{\mathbf{x}}_{t})- \frac{1}{N} \sum_{n=1}^N \nabla F_n(\mathbf{x}_{n, t})\|^2 \nonumber\\
&+ \eta \mathbb{E} \|\frac{1}{N} \sum_{n=1}^N \nabla F_n(\mathbf{x}_{n, t}) - \frac{1}{N} \sum_{n=1}^N \nabla \hat{F}_n(\mathbf{x}_{n, t})\|^2 + \eta^2 S_F \mathbb{E} \|\frac{1}{N} \sum_{n=1}^N \nabla \hat{F}_n(\mathbf{x}_{n, t}) - \bar{\mathbf{u}}_{t}\|^2
\end{align}
where inequality (a) holds by the smoothness of $F$; equality (b) follows from update step in Step 9 of Algorithm \ref{alg:1} and \cref{lem:A1} (b); (c) uses the fact that $\mathbb{E} u_{n,t} = \nabla \hat{F}_n(\mathbf{x}_{n, t})$ and $\|a + b\|^2 \leq 2\|a\|^2 + 2\|b\|^2$; (d) holds since the inequality $ \langle a, b\rangle = \frac{1}{2}[ \|a\|^2 +\|a\|^2 - \|a - b\|^2]$. Taking expectation on both sides and considering the last third term
\begin{align}
\mathbb{E}\|\nabla F(\bar{\mathbf{x}}_{t})- \frac{1}{N} \sum_{n=1}^N \nabla F_n(\mathbf{x}_{n, t})\|^{2} &\leq \frac{1}{N} \sum_{i=1}^{N} \mathbb{E}\|\nabla F_n(\bar{\mathbf{x}}_{t}) - \nabla F_n(\mathbf{x}_{n, t})\|^{2} \nonumber\\
 &\leq \frac{S_F^{2}}{N} \sum_{i=1}^{N} \|x_{n,t} - \bar{\mathbf{x}}_{t}\|^{2} 
\end{align}
Considering the last second term and \Cref{lem:A1} (a), we have
\begin{align}
\mathbb{E}\|\frac{1}{N} \sum_{n=1}^N \nabla F_n(\mathbf{x}_{n, t}) - \frac{1}{N} \sum_{n=1}^N \nabla \hat{F}_n(\mathbf{x}_{n, t})\|^2 &\leq \frac{1}{N} \sum_{n=1}^{N} \mathbb{E} \| \nabla F_n(\mathbf{x}_{n, t}) - \nabla \hat{F}_n(\mathbf{x}_{n, t})\|^2 \nonumber\\
&\leq \frac{L_g^2 S_f^2 \sigma_g^2 }{m}
\end{align}
Given that $\mathbb{E} u_{n,t} = \nabla \hat{F}_n(\mathbf{x}_{n, t})$ and \Cref{lem:A1} (c), for the last term, we have 
\begin{align} \label{eq:21}
\mathbb{E} \|\frac{1}{N} \sum_{n=1}^N \nabla \hat{F}_n(\mathbf{x}_{n, t}) - \bar{\mathbf{u}}_{t}\|^2 \leq \frac{1}{N^2} \sum_{n=1}^N \mathbb{E} \| \nabla \hat{F}_n(\mathbf{x}_{n, t}) - \mathbf{u}_{n, t}\|^2 \leq  \frac{ L_f^2 L_g^2}{N}
\end{align}

Therefore, we obtain
\begin{align}
\mathbb{E} F(\bar{\mathbf{x}}_{t+1}) &  \leq \mathbb{E} F(\bar{\mathbf{x}}_{t}) - (\frac{\eta}{2} - \eta^2 S_F ) \mathbb{E} \|\nabla \hat{F}(\bar{\mathbf{x}}_t) \|^2 - \frac{\eta}{2} \mathbb{E}\| \nabla F(\bar{\mathbf{x}}_{t})\|^2 + \frac{\eta S_F^2}{N} \sum_{n=1}^{N}\|x_{n,t} - \bar{\mathbf{x}}_t\|^2 \nonumber\\
&+ \frac{\eta L_g^2 S_f^2 \sigma_g^2}{m} + \frac{\eta^2 S_F L_f^2 L_g^2}{N}
\end{align}
\end{proof}

\begin{lemma} \label{lem:7}
Let Assumptions \ref{ass:1}, \ref{ass:2}  hold. We have: $\forall t \geq 0$,
\begin{align}
& \mathbb{E}\left[\left\|\mathbf{v}_{t + 1} - \mathbf{J} \mathbf{v}_{t + 1}\right\|^2\right] 
\leq  \frac{1 +\lambda^2}{2} \mathbb{E}\left [\left\|\mathbf{v}_{t} - \mathbf{J} \mathbf{v}_{t}\right\|^2\right] + \frac{6 \lambda^2 \eta^2 S_F^2 N}{1 - \lambda^2} \mathbb{E} \left\|\nabla \hat{F}_t\right\|^2 \nonumber\\
&+ \frac{24 \lambda^2 S_F^2}{1 - \lambda^2} \mathbb{E}\left[\left\|\mathbf{x}_t - \mathbf{J} \mathbf{x}_t\right\|^2\right] + \left(\frac{6 \lambda^2 \eta^2 S_F^2 N}{1 - \lambda^2} + 3 N + 2\right) L_f^2 L_g^2 
\end{align}
\end{lemma}
\begin{proof}
Using the gradient tracking update step in \Cref{alg:1}, and the fact that $\mathbf{W J}=\mathbf{J} \mathbf{W}=\mathbf{J}$, we have: $\forall t \geq 0$,
\begin{align} \label{eq:22}
&\mathbb{E}\left\|\mathbf{v}_{t + 1} - \mathbf{J} \mathbf{v}_{t + 1}\right\|^2 \nonumber\\
=&\mathbb{E}\left\|\mathbf{W}\left(\mathbf{v}_{t} + \mathbf{u}_{t + 1} - \mathbf{u}_{t}\right) - \mathbf{J} \left(\mathbf{v}_{t} + \mathbf{u}_{t + 1} - \mathbf{u}_{t}\right)\right\|^2  \nonumber\\
=&\mathbb{E} \left\|\mathbf{W} \mathbf{v}_{t} - \mathbf{J} \mathbf{v}_{t} + (\mathbf{W} - \mathbf{J})\left(\mathbf{u}_{t + 1} - \mathbf{u}_{t}\right)\right\|^2  \nonumber\\
=&\mathbb{E}\left\|\mathbf{W} \mathbf{v}_{t} - \mathbf{J} \mathbf{v}_{t}\right\|^2 + \mathbb{E} \left\|(\mathbf{W} - \mathbf{J})\left(\mathbf{u}_{t + 1} - \mathbf{u}_{t}\right)\right\|^2  \nonumber\\
&+2\mathbb{E}\left\langle(\mathbf{W}-\mathbf{J}) \mathbf{v}_{t}, (\mathbf{W} - \mathbf{J})\left(\mathbf{u}_{t + 1} - \mathbf{u}_{t}\right)\right\rangle  \nonumber\\
\leq & \lambda^2\mathbb{E}\left\|\mathbf{v}_{t} - \mathbf{J} \mathbf{v}_{t} \right\|^2+\lambda^2\mathbb{E}\left\|\mathbf{u}_{t + 1} - \mathbf{u}_{t}\right\|^2  \nonumber\\
&+2 \mathbb{E}\left\langle(\mathbf{W} - \mathbf{J}) \mathbf{v}_{t}, (\mathbf{W} - \mathbf{J})\left(\mathbf{u}_{t + 1} - \mathbf{u}_{t}\right)\right\rangle
\end{align}

where the last inequality is due to \Cref{lem:A2} (a). since $u_{t+1}$ and $v_{t+1}$ are $\mathcal{F}_{t+1}$ -measurable.

For the second term in \eqref{eq:22}, we have 
\begin{align} \label{eq:23}
& \mathbb{E} \left\|\mathbf{u}_{t + 1} - \mathbf{u}_{t}\right\|^2 \nonumber\\
\stackrel{(a)}{=}& \mathbb{E}[\|\mathbf{u}_{t + 1} - \nabla \hat{\mathbf{F}}_{t + 1}\|^2 + \|\nabla \hat{\mathbf{F}}_{t + 1} - \mathbf{u}_{t}\|^2]  \nonumber\\
=& \mathbb{E}[\|\mathbf{u}_{t + 1} - \nabla \hat{\mathbf{F}}_{t + 1}\|^2 + 2\|\nabla \hat{\mathbf{F}}_{t + 1} - \nabla \hat{\mathbf{F}}_{t}\|^2 + 2\|\mathbf{u}_{t} - \nabla \hat{\mathbf{F}}_{t}\|^2]  \nonumber\\
\leq & 3 N L_f^2 L_g^2  + 2 S_F^2 \mathbb{E}\left[\left\|\mathbf{x}_{t + 1}-\mathbf{x}_{t}\right\|^2\right]
\end{align}
where the equality (a) follows that $\mathbb{E}[\mathbf{u}_{t + 1} \mid \mathcal{F}_{t+1}] = \nabla \hat{\mathbf{F}}_{t + 1}$ and the last inequality uses \Cref{lem:A1} (b) and (c). $\forall t \geq 0$,
\begin{align} \label{eq:24}
\mathbb{E}\left\|\mathbf{x}_{t + 1} - \mathbf{x}_{t}\right\|^2 =& \mathbb{E}\left\|\mathbf{x}_{t + 1} - \mathbf{J} \mathbf{x}_{t + 1} + \mathbf{J} \mathbf{x}_{t + 1} - \mathbf{J} \mathbf{x}_{t} + \mathbf{J} \mathbf{x}_{t} - \mathbf{x}_{t}\right\|^2 \nonumber\\
\stackrel{(a)}{\leq} & 3 \mathbb{E} \left\|\mathbf{x}_{t + 1} - \mathbf{J} \mathbf{x}_{t + 1}\right\|^2 + 3 N \eta^2 \mathbb{E} \left\|\bar{\mathbf{u}}_{t}\right\|^2 + 3 \mathbb{E} \left\|\mathbf{x}_{t} - \mathbf{J} \mathbf{x}_{t}\right\|^2  \nonumber \\
\stackrel{(b)}{\leq} & 9 \mathbb{E}\left[\left\|\mathbf{x}_{t} - \mathbf{J} \mathbf{x}_{t}\right\|^2\right] + 3 N \eta^2 \mathbb{E}\left[\left\|\bar{\mathbf{u}}_{t}\right\|^2\right] + 6 \eta^2 \lambda^2 \mathbb{E}\left[\left\|\mathbf{v}_{t} - \mathbf{J} \mathbf{v}_{t}\right\|^2\right]
\end{align}
where (b) holds due to \Cref{lem:A2} (b), and (c) holds due to \eqref{eq:17} and $\lambda \leq 1$.
Putting the \eqref{eq:24} into \eqref{eq:23}, we have 
\begin{align} \label{eq:25}
& \mathbb{E} \left\|\mathbf{u}_{t + 1} - \mathbf{u}_{t}\right\|^2 \leq 3 N L_f^2 L_g^2  + 18 S_F^2 \mathbb{E} \left\|\mathbf{x}_{t} - \mathbf{J} \mathbf{x}_{t}\right\|^2 + 6 N \eta^2 S_F^2 \mathbb{E} \left\|\bar{\mathbf{u}}_{t}\right\|^2 + 12 \eta^2 \lambda^2 S_F^2 \mathbb{E} \left\|\mathbf{v}_{t} - \mathbf{J} \mathbf{v}_{t}\right\|^2
\end{align}

For the last term in \eqref{eq:22}, we have,
\begin{align} \label{eq:31}
& \mathbb{E}\left[\left\langle(\mathbf{W} - \mathbf{J}) \mathbf{v}_{t}, (\mathbf{W} - \mathbf{J})\left(\mathbf{u}_{t + 1} - \mathbf{u}_{t}\right)\right\rangle \mid \mathcal{F}_{t + 1}\right]  \nonumber\\
=&\mathbb{E}\left\langle(\mathbf{W} - \mathbf{J}) \mathbf{v}_{t}, (\mathbf{W} - \mathbf{J}) \left(\nabla \hat{\mathbf{F}}_{t + 1} - \mathbf{u}_{t}\right)\right\rangle  \nonumber\\
=&\mathbb{E}\left\langle(\mathbf{W} - \mathbf{J}) \mathbf{v}_{t}, (\mathbf{W} - \mathbf{J}) \left(\nabla \hat{\mathbf{F}}_{t} - \mathbf{u}_{t} \right) \right \rangle + \mathbb{E}\left\langle(\mathbf{W} - \mathbf{J}) \mathbf{v}_{t}, (\mathbf{W} - \mathbf{J}) \left( \nabla \hat{\mathbf{F}}_{t + 1} -  \nabla \hat{\mathbf{F}}_{t}\right)\right\rangle
\end{align}

Furthermore, we have 
\begin{align} \label{eq:32} &\mathbb{E}\left[\left\langle(\mathbf{W} - \mathbf{J}) \mathbf{v}_{t}, (\mathbf{W} - \mathbf{J}) \left(\nabla \hat{\mathbf{F}}_{t} - \mathbf{u}_{t}\right)\right\rangle \mid \mathcal{F}_t\right] \nonumber\\
\stackrel{(a)}{=}& \mathbb{E}\left[\left\langle\mathbf{W} \mathbf{v}_{t}, (\mathbf{W} - \mathbf{J}) \left(\nabla \hat{\mathbf{F}}_{t} - \mathbf{u}_{t}\right)\right\rangle \mid \mathcal{F}_t\right] \nonumber\\
\stackrel{(b)}{=}& \mathbb{E}\left[\left\langle\mathbf{W}^2 \left(\mathbf{v}_{t - 1} + \mathbf{u}_{t} - \mathbf{u}_{t-1}\right), (\mathbf{W} - \mathbf{J}) \left(\nabla \hat{\mathbf{F}}_{t} - \mathbf{u}_{t}\right)\right\rangle \mid \mathcal{F}_t\right] \nonumber\\
\stackrel{(c)}{=}& \mathbb{E}\left[\left\langle\mathbf{W}^2 \mathbf{u}_t, (\mathbf{W} - \mathbf{J}) \left(\nabla \hat{\mathbf{F}}_{t} - \mathbf{u}_{t}\right) \right\rangle \mid \mathcal{F}_t\right] \nonumber\\
=& \mathbb{E}\left[\left\langle\mathbf{W}^2\left(\mathbf{u}_t - \nabla \hat{\mathbf{F}}_t\right), (\mathbf{W} - \mathbf{J}) \left(\nabla \hat{\mathbf{F}}_t - \mathbf{u}_t\right)\right\rangle \mid \mathcal{F}_t \right] \nonumber\\
=& \mathbb{E}\left[\left(\mathbf{u}_t - \nabla \hat{\mathbf{F}}_t \right)^{\top}\left(\mathbf{J}-\mathbf{W}^{\top} \mathbf{W}^2\right)\left(\mathbf{u}_t - \nabla \hat{\mathbf{F}}_t\right) \mid \mathcal{F}_t\right] \nonumber\\
=& \mathbb{E}\left[\left(\mathbf{u}_t - \nabla \hat{\mathbf{F}}_t\right)^{\top} \operatorname{diag}\left(\mathbf{J}-\mathbf{W}^{\top} \mathbf{W}^2\right)\left(\mathbf{u}_t - \nabla \hat{\mathbf{F}}_t\right) \mid \mathcal{F}_t\right] \nonumber\\
\leq &  \mathbb{E}\left[\left\|\mathbf{u}_t - \nabla \hat{\mathbf{F}}_t\right\|^2 \mid \mathcal{F}_t\right] / N \leq L_f^2 L_g^2
\end{align}
where (a) holds since $\mathbf{J} (\mathbf{W} - \mathbf{J}) = \mathbf{O}_{np}$, (b) because the update step of $\mathbf{v}_t$; (c) because $\mathbb{E} \mathbf{u}_t = \nabla \hat{F}_t$ and $\mathbf{v}_{t - 1}$ and $\mathbf{u}_{t - 1}$ are independently. 

For the second term in \eqref{eq:31}, we have
\begin{align} \label{eq:31}
&\mathbb{E} \left\langle(\mathbf{W} - \mathbf{J}) \mathbf{v}_{t}, (\mathbf{W}-\mathbf{J}) \left(\nabla \hat{\mathbf{F}}_{t+1} - \nabla \hat{\mathbf{F}}_t \right)\right\rangle \nonumber\\
\stackrel{(a)}{=}& \mathbb{E} \left\langle(\mathbf{W} - \mathbf{J})\left(\mathbf{v}_{t} - \mathbf{J} \mathbf{v}_{t}\right), (\mathbf{W} - \mathbf{J}) \left(\nabla \hat{\mathbf{F}}_t - \nabla \hat{\mathbf{F}}_{t - 1} \right)\right\rangle \nonumber\\
\stackrel{(b)}{\leq} & \lambda^2 S_F \left\|\mathbf{v}_{t}-\mathbf{J} \mathbf{v}_{t}\right\|\left\|\mathbf{x}_{t+1}-\mathbf{x}_t\right\|
\end{align}
where (a) follows  $(\mathbf{W}-\mathbf{J}) \mathbf{J}=\mathbf{O}_{n p}$ and (b) the Cauchy-Schwarz inequality and smooth of $\hat{F}(x) $, we have: $\forall k \geq 0$,
where the last inequality uses $\|\mathbf{W}-\mathbf{J}\|=\lambda$ and the $S_F$ smoothness of $\hat{F}_{n,t}$. 

Furthermore, $\forall t \geq 0$,
\begin{align} \label{eq:29}
&\left\|\mathbf{x}_{t+1} - \mathbf{x}_t\right\| \nonumber\\
=&\left\|\mathbf{x}_{t+1} - \mathbf{J} \mathbf{x}_{t+1} + \mathbf{J} \mathbf{x}_{t+1} - \mathbf{J} \mathbf{x}_t + \mathbf{J} \mathbf{x}_t - \mathbf{x}_t\right\| \nonumber\\
\leq &\left\|\mathbf{x}_{t+1} - \mathbf{J} \mathbf{x}_{t+1}\right\| + \eta \sqrt{N}\left\|\bar{\mathbf{u}}_t\right\| + \left\|\mathbf{x}_t - \mathbf{J} \mathbf{x}_t\right\| \nonumber\\
\leq & 2\left\|\mathbf{x}_t - \mathbf{J} \mathbf{x}_t \right\| + \eta \sqrt{N} \left\|\bar{\mathbf{u}}_t\right\| + \eta \lambda\left\|\mathbf{v}_{t} - \mathbf{J} \mathbf{v}_{t} \right\|
\end{align}

where the last inequality uses \eqref{eq:18}. Combining the above two inequalities, we have

\begin{align} \label{eq:31}
&\left\langle(\mathbf{W} - \mathbf{J}) \mathbf{v}_{t}, (\mathbf{W} - \mathbf{J}) \left(\nabla \hat{\mathbf{F}}_{t+1} - \nabla \hat{\mathbf{F}}_{t}\right)\right\rangle \nonumber\\
\leq & \lambda^3 \eta S_F \left\|\mathbf{v}_{t} - \mathbf{J} \mathbf{v}_{t} \right\|^2 + \left(\lambda\left\|\mathbf{v}_{t} - \mathbf{J v}_{t}\right\|\right)\left(\lambda \eta S_F \sqrt{N}\left\|\bar{\mathbf{u}}_t\right\|\right) + 2\left(\lambda\left\|\mathbf{v}_{t} - \mathbf{J v}_{t}\right\|\right)\left(\lambda S_F\left\|\mathbf{x}_k-\mathbf{J} \mathbf{x}_k\right\|\right) \nonumber \\
\stackrel{(a)}{\leq} & \lambda^3 \eta S_F \left\|\mathbf{v}_{t} - \mathbf{J} \mathbf{v}_{t} \right\|^2 +  \frac{c_0 \lambda^2}{2}\left\|\mathbf{v}_{t} - \mathbf{J v}_{t}\right\|^2 + \frac{ \lambda^2 \eta^2 S_F^2 N}{2c_0}\left\|\bar{\mathbf{u}}_t\right\|^2 + c_1 \lambda^2\left\|\mathbf{v}_{t} - \mathbf{J} \mathbf{v}_{t}\right\|^2 + \frac{\lambda^2 S_F^2}{c_1} \left\|\mathbf{x}_t - \mathbf{J} \mathbf{x}_t\right\|^2 \nonumber\\
\leq& (\lambda^3 \eta S_F + \frac{c_0 \lambda^2}{2} + c_1 \lambda^2) \left\|\mathbf{v}_{t} - \mathbf{J} \mathbf{v}_{t} \right\|^2  + \frac{ \lambda^2 \eta^2 S_F^2 N}{2c_0}\left\|\bar{\mathbf{u}}_t\right\|^2 +  \frac{\lambda^2 S_F^2}{c_1} \left\|\mathbf{x}_t - \mathbf{J} \mathbf{x}_t\right\|^2
\end{align}
where (a) holds due to the Young’s inequality and $c_0, c_1 (> 0)$ are arbitrary.

Then putting \eqref{eq:25}, \eqref{eq:31}, \eqref{eq:32} and \eqref{eq:31} into \eqref{eq:22}, we have
\begin{align}
\mathbb{E} \|\mathbf{v}_{t + 1} - \mathbf{J v}_{t + 1} \|^2  \leq& \lambda^2 (1 + 12 \lambda^2 \eta^2 S_F^2 + 2 \lambda \eta S_F + c_0 + 2 c_1) \mathbb{E}  \|\mathbf{v}_{t} - \mathbf{J v}_{t} \|^2 + 3 \lambda^2 N L_f^2L_g^2 \nonumber\\
+ 2 L_f^2 L_g^2  &+ (18 + \frac{2}{c_1} ) \lambda^2 S_F^2 \mathbb{E} \| \mathbf{x}_{t} - \mathbf{J x}_{t} \|^2 + (6 + \frac{1}{c_0}) \lambda^2 \eta^2 S_F^2 N \mathbb{E} \|\bar{\mathbf{u}}_t \|^2.
\end{align}
We set $c_0 = \frac{1 - \lambda^2}{6 \lambda^2}$ and $c_1 = \frac{1 - \lambda^2}{12 \lambda^2}$. When $0 < \eta \leq \min \{\frac{1 - \lambda^2}{24 \lambda^2 S_F}, \frac{1}{6 S_F}\}$, we have $\lambda^2 (1 + 12 \lambda^2 \eta^2 S_F^2 + 2 \lambda \eta S_F + c_0 + 2 c_1) \leq \frac{1 + \lambda^2}{2}$, and the fact 
\begin{align}
\mathbb{E} \|\bar{\mathbf{u}}_t\| =&\mathbb{E} \|\frac{1}{N} \sum_{n=1}^N \nabla \hat{F}_n(\mathbf{x}_{n, t})\|^2 +  \mathbb{E} \|\frac{1}{N} \sum_{n=1}^N \nabla \hat{F}_n(\mathbf{x}_{n, t}) - \bar{\mathbf{u}}_{t}\|^2 \nonumber\\
\leq& 2\mathbb{E} \|\frac{1}{N} \sum_{n=1}^N \nabla \hat{F}_n(\mathbf{x}_{n, t}) - \frac{1}{N} \sum_{n=1}^N \nabla \hat{F}_n(\mathbf{\bar{x}}_{t})\|^2 + 2\mathbb{E} \| \frac{1}{N} \sum_{n=1}^N \nabla \hat{F}_n(\mathbf{\bar{x}}_{t})\|^2 +  \frac{ L_f^2 L_g^2}{N} \nonumber\\
\leq& \frac{2S_F^2}{N} \sum_{n=1}^N\mathbb{E} \| \mathbf{x}_{n, t} - \mathbf{\bar{x}}_{t}\|^2 + 2\mathbb{E} \| \nabla \hat{F}(\mathbf{\bar{x}}_{t})\|^2 +  \frac{ L_f^2 L_g^2}{N} \nonumber
\end{align}
then we have
\begin{align}
\mathbb{E}\left[\left\|\mathbf{v}_{t + 1} - \mathbf{J} \mathbf{v}_{t + 1}\right\|^2\right] 
\leq & \frac{1 +\lambda^2}{2} \mathbb{E}\left [\left\|\mathbf{v}_{t} - \mathbf{J} \mathbf{v}_{t}\right\|^2\right] + \left(\frac{6 \lambda^2 \eta^2 S_F^2 N}{1 - \lambda^2} + 3 N + 2\right) L_f^2 L_g^2 \nonumber\\
&+ \frac{36 \lambda^2 S_F^2}{1 - \lambda^2} \mathbb{E}\left[\left\|\mathbf{x}_t - \mathbf{J} \mathbf{x}_t\right\|^2\right] + \frac{12 \lambda^2 \eta^2 S_F^2 N}{1 - \lambda^2} \mathbb{E} \left\|\nabla \hat{F}(\bar{\mathbf{x}}_{ t})\right\|^2 
\end{align}
where 

\end{proof}

\subsection{Proofs of Theorem }
Based on previous lemmas, we start to prove the convergence of Theorem. 
\begin{proof}
Recall \Cref{lem:A2},  we have 
\begin{align} \label{eq:38}
\left\|\mathbf{x}_{t+1}-\mathbf{J} \mathbf{x}_{t+1}\right\|^2 \leq & \frac{1+\lambda^2}{2}\left\|\mathbf{x}_t-\mathbf{J} \mathbf{x}_t\right\|^2 +\frac{2 \eta^2 \lambda^2}{1-\lambda^2}\left\|\mathbf{v}_{t}-\mathbf{J} \mathbf{v}_{t}\right\|^2
\end{align}

Putting \eqref{eq:38} and \cref{lem:7} into \cref{lem:A3}, then we have 
\begin{align} \label{eq:39}
\sum_{t=0}^{T}\|\mathbf{x}_t - \mathbf{J x}_{t} \|^2 \leq \frac{4\lambda^2 \eta^2}{(1- \lambda^2)^2} \sum_{t=0}^{T}\|\mathbf{v}_t - \mathbf{J v}_{t} \|^2
\end{align}

\begin{align} \label{eq:40}
&\sum_{t=0}^{T} \|\mathbf{v}_t - \mathbf{J v}_{t} \|^2 \leq  \frac{2}{1 - \lambda^2} \|\mathbf{v}_0 - \mathbf{J v}_{0} \|^2 + \left(\frac{6 \lambda^2 \eta^2 S_F^2}{1 - \lambda^2} + 5 N\right) \frac{2 L_f^2 L_g^2 T}{1 - \lambda^2 } \nonumber\\
&+ \frac{72 \lambda^2 S_F^2}{(1 - \lambda^2)^2} \sum_{t=0}^{ T} \mathbb{E}\left[\left\|\mathbf{x}_t - \mathbf{J} \mathbf{x}_t\right\|^2\right] + \frac{24 \lambda^2 \eta^2 S_F^2 N}{(1 - \lambda^2)^2} \sum_{t=0}^{t = T}\mathbb{E}\left\| \nabla \hat{F}(\bar{\mathbf{x}}_{ t}) \right\|^2
\end{align}

Then putting \eqref{eq:39} into \eqref{eq:40}, we get 
\begin{align}
&\sum_{t=0}^{T}\|\mathbf{x}_t - \mathbf{J x}_{t} \|^2 
\leq \frac{4\lambda^2 \eta^2}{(1- \lambda^2)^2}\frac{2}{1 - \lambda^2} \|\mathbf{v}_0 - \mathbf{J v}_{0} \|^2 + \frac{96 \lambda^4 \eta^4 S_F^2 N}{(1 - \lambda^2)^4} \sum_{t=0}^{t = T}\mathbb{E}\left\|\nabla \hat{F}(\bar{\mathbf{x}}_{ t})\right\|^2  \nonumber\\
&+ \frac{288 \lambda^4 \eta^2 S_F^2}{(1 - \lambda^2)^4} \sum_{t=0}^{ T} \mathbb{E} \left\|\mathbf{x}_t - \mathbf{J} \mathbf{x}_t\right\|^2 + \left(\frac{6 \lambda^2 \eta^2 S_F^2}{1 - \lambda^2} + 5 N\right) \frac{8 \lambda^2 \eta^2 L_f^2 L_g^2 T}{(1 - \lambda^2)^3} \nonumber
\end{align}

Then we have 
\begin{align}
&[1 - \frac{288 \lambda^4 \eta^2 S_F^2}{(1 - \lambda^2)^4} ]\sum_{t=0}^{T}\|\mathbf{x}_t - \mathbf{J x}_{t} \|^2 
\leq  \left(\frac{6 \lambda^2 \eta^2 S_F^2}{1 - \lambda^2} + 5 N\right) \frac{8 \lambda^2 \eta^2 L_f^2 L_g^2 T}{(1 - \lambda^2)^3} + \frac{96 \lambda^4 \eta^4 S_F^2 N}{(1 - \lambda^2)^4} \sum_{t=0}^{t = T}\mathbb{E}\left\|\nabla \hat{F}(\bar{\mathbf{x}}_{ t}) \right\|^2
\end{align}
When $0 < \eta \leq  \min \{\frac{1 - \lambda^2}{24 \lambda^2 S_F}, \frac{1}{6 S_F}\}$, we have $[1 - \frac{192 \lambda^4 \eta^2 S_F^2}{(1 - \lambda^2)^4} ] \geq \frac{1}{2}$,
Therefore, we have 
\begin{align} \label{eq:43}
\sum_{t=0}^{T}\|\mathbf{x}_t - \mathbf{J x}_{t} \|^2 \leq  
\left(\frac{1}{ \lambda^2} + 5 N\right) \frac{16 \lambda^2 \eta^2 L_f^2 L_g^2 T}{(1 - \lambda^2)^2}  + \frac{192 \lambda^4 \eta^4 S_F^2 N}{(1 - \lambda^2)^4} \sum_{t=0}^{t = T}\mathbb{E}\left\|\nabla \hat{F}_t\right\|^2 
\end{align}
Putting \eqref{eq:43} into \cref{lem:B1}, we have
\begin{align}
&\frac{1}{T} \sum_{t=0}^{T - 1} \mathbb{E} \|\nabla F(\bar{\mathbf{x}}_{t})\|^2 
\leq \frac{2  \mathbb{E} [F(\bar{\mathbf{x}}_{0}) - F(\bar{\mathbf{x}}_{T})] }{\eta T}  + \frac{2 \eta S_F L_f^2 L_g^2}{N} 
\nonumber\\
&+ \frac{2 S_F^2}{N T } \sum_{t=0}^{T - 1} \sum_{n=1}^{N}\|x_{n,t} - \bar{\mathbf{x}}_t\|^2 + \frac{2 L_g^2 S_f^2 \sigma_g^2}{m} - \frac{(1 - 2 \eta S_F )}{T} \sum_{t=0}^{T - 1} \mathbb{E}\|\nabla \hat{F}(\bar{\mathbf{x}}_t) \|^2 \nonumber\\
&\leq \frac{2  \mathbb{E} [F(\bar{\mathbf{x}}_{0}) - F(\bar{\mathbf{x}}_{T})] }{\eta T}
- \frac{1}{T}[1 - 2 \eta S_F - \frac{384 \lambda^4 \eta^4 S_F^4 }{(1 - \lambda^2)^4}] \sum_{t = 0}^{T - 1} \|\nabla \hat{F}(\bar{\mathbf{x}}_t) \|^2 \nonumber\\
&+  \left(\frac{1}{ \lambda^2} + 5 N\right) \frac{32 \lambda^2 \eta^2 L_f^2 L_g^2 S_F^2}{(1 - \lambda^2)^2 N}  + \frac{2 L_g^2 S_f^2 \sigma_g^2}{m} + \frac{2 \eta S_F L_f^2 L_g^2}{N}  \nonumber\\
&\leq \frac{2  \mathbb{E} [F(\bar{\mathbf{x}}_{0}) - F(\bar{\mathbf{x}}_{T})] }{\eta T}
+  \left(\frac{1}{ \lambda^2} + 5 N\right) \frac{32 \lambda^2 \eta^2 L_f^2 L_g^2 S_F^2}{(1 - \lambda^2)^2 N}  + \frac{2 L_g^2 S_f^2 \sigma_g^2}{m} + \frac{2 \eta S_F L_f^2 L_g^2}{N} \nonumber
\end{align}
\end{proof}

Then we discuss the convergence rate of the \cref{alg:1}. Setting $\eta = O(\sqrt{\frac{N}{T}})$ , we have final
\begin{align}
\frac{1}{T} \sum_{t=0}^{T - 1} \mathbb{E} \|\nabla F(\bar{\mathbf{x}}_{t})\|^2 \leq O(\frac{\mathbb{E} [F(\bar{\mathbf{x}}_{0}) - F(\bar{\mathbf{x}}_{T})] }{(N T)^{1/2}})
+  \left(\frac{1}{ \lambda^2} + 5 N\right) \frac{32 \lambda^2 L_f^2 L_g^2 S_F^2}{(1 - \lambda^2)^2} O(\frac{1}{T})  + \frac{2 L_g^2 S_f^2 \sigma_g^2}{m} + O(\frac{2 S_F L_f^2 L_g^2}{(NT)^{1/2}})
\end{align}

Let $m = O(\varepsilon^{-2})$ and $\sqrt{T} > N$, we know to make $\frac{1}{T} \sum_{t=0}^{T - 1} \mathbb{E} \|\nabla F(\bar{\mathbf{x}}_{t})\|^2 \leq \varepsilon^{2}$, we have $T \leq N^{-1} \varepsilon^{-4} $. 

\section{Proof of SLATE-M Algorithm}

\subsection{Proofs of the Intermediate Lemmas}
\begin{lemma} \label{lem:C1}
Suppose the sequence $\{x_t\}_0^{T}$ are generated from SLATE-M in \cref{alg:2}, we have
\begin{align}
F(\bar{\mathbf{x}}_{t+1}) &  \leq F(\bar{\mathbf{x}}_{t}) - (\frac{\eta}{2} - \frac{\eta^2 S_F }{2} ) \|\bar{\mathbf{u}}_{t} \|^2 - \frac{\eta}{2} \| \nabla F(\bar{\mathbf{x}}_{t})\|^2 + \frac{3\eta S_F^2}{2N} \|\mathbf{x}_{t} - \bar{\mathbf{x}}_t\|^2 \nonumber\\ 
&+ \frac{3\eta L_g^2 S_f^2 \sigma_g^2}{2m} + \frac{3\eta}{2}\|\frac{1}{N} \sum_{n=1}^N \nabla \hat{F}_n(\mathbf{x}_{n, t}) - \bar{\mathbf{u}}_{t}\|^2
\end{align}
\end{lemma}
\begin{proof}
\begin{align}
F(\bar{\mathbf{x}}_{t+1}) &\stackrel{(a)}{\leq} F(\bar{\mathbf{x}}_{t}) + \langle\nabla F(\bar{\mathbf{x}}_{t}), \bar{\mathbf{x}}_{t+1}-\bar{\mathbf{x}}_{t}\rangle + \frac{S_F}{2}\left\|\bar{\mathbf{x}}_{t+1} - \bar{\mathbf{x}}_{t}\right\|^{2}  \nonumber\\
&\stackrel{(b)}{=}F(\bar{\mathbf{x}}_{t}) - \eta \langle\nabla F(\bar{\mathbf{x}}_{t}), \bar{\mathbf{u}}_{t}\rangle + \frac{\eta^2 S_F }{2} \|\bar{\mathbf{u}}_{t}\|^{2} \nonumber\\
&\stackrel{(c)}{=}F(\bar{\mathbf{x}}_{t}) - (\frac{\eta}{2} - \frac{\eta^2 S_F }{2} ) \|\bar{\mathbf{u}}_{t} \|^2 - \frac{\eta}{2} \| \nabla F(\bar{\mathbf{x}}_{t})\|^2 + \frac{\eta}{2} \|\nabla F(\bar{\mathbf{x}}_{t})- \bar{\mathbf{u}}_{t}\|^2 \nonumber\\
&\leq  F(\bar{\mathbf{x}}_{t}) - (\frac{\eta}{2} - \frac{\eta^2 S_F }{2} ) \|\bar{\mathbf{u}}_{t} \|^2 - \frac{\eta}{2} \| \nabla F(\bar{\mathbf{x}}_{t})\|^2 + \frac{3\eta}{2} \|\nabla F(\bar{\mathbf{x}}_{t})- \frac{1}{N} \sum_{n=1}^N \nabla F_n(\mathbf{x}_{n, t})\|^2 \nonumber\\ 
&+ \frac{3\eta}{2} \|\frac{1}{N} \sum_{n=1}^N \nabla F(\mathbf{x}_{n, t}) - \frac{1}{N} \sum_{n=1}^N \nabla \hat{F}_n(\mathbf{x}_{n, t})\|^2 + \frac{3\eta}{2} \|\frac{1}{N} \sum_{n=1}^N \nabla \hat{F}_n (\mathbf{x}_{n, t}) - \bar{\mathbf{u}}_{t}\|^2
\end{align}
where inequality (a) holds by the smoothness of F(x); equality (b) follows from update step in Step 9 of Algorithm \ref{alg:1}; (c) uses the fact that $ \langle a, b\rangle = \frac{1}{2}[ \|a\|^2 +\|a\|^2 - \|a - b\|^2]$.  Taking expectation on both sides and considering the last third term
\begin{align}
\mathbb{E}\|\nabla F(\bar{\mathbf{x}}_{t}) - \frac{1}{N} \sum_{n=1}^N \nabla F_n(\mathbf{x}_{n, t})\|^{2} \leq \frac{1}{N} \sum_{n=1}^{N} \mathbb{E}\|\nabla F_n(\bar{\mathbf{x}}_{t}) - \nabla F_n(\mathbf{x}_{n, t})\|^{2} \leq \frac{S_F^{2}}{N} \sum_{n=1}^{N} \|x_{n,t} - \bar{\mathbf{x}}_{t}\|^{2} 
\end{align}
Considering the last second term, we have
\begin{align}
\mathbb{E}\|\frac{1}{N} \sum_{n=1}^N \nabla F_n (\mathbf{x}_{t}) - \frac{1}{N} \sum_{n=1}^N \nabla \hat{F}_n(\mathbf{x}_{t})\|^2 \leq \frac{1}{N} \sum_{n=1}^{N} \mathbb{E} \| \nabla F_n (\mathbf{x}_{n, t}) - \nabla \hat{F}_n(\mathbf{x}_{n, t})\|^2 \leq \frac{L_g^2 S_f^2 \sigma_g^2 }{m}
\end{align}

Therefore, we obtain
\begin{align}
F(\bar{\mathbf{x}}_{t+1}) &  \leq F(\bar{\mathbf{x}}_{t}) - (\frac{\eta}{2} - \frac{\eta^2 S_F }{2} ) \|\bar{\mathbf{u}}_{t} \|^2 - \frac{\eta}{2} \| \nabla F(\bar{\mathbf{x}}_{t})\|^2 + \frac{3\eta S_F^2}{2N} \sum_{n=1}^{N} \|x_{n,t} - \bar{\mathbf{x}}_{t}\|^{2}  \nonumber\\ 
&+ \frac{3\eta L_g^2 S_f^2 \sigma_g^2}{2m} + \frac{3\eta}{2}\|\frac{1}{N} \sum_{n=1}^N \nabla \hat{F}_n(\mathbf{x}_{n, t}) - \bar{\mathbf{u}}_{t}\|^2
\end{align}
\end{proof}

\begin{lemma} \label{lem:A6}
Assume that the stochastic partial derivatives $u_{t}$  be generated from SLATE-M in Algorithm \ref{alg:2}, we have
\begin{align}
&\mathbb{E}\|\bar{\mathbf{u}}_{t + 1} - \frac{1}{N} \sum_{n=1}^N \nabla \hat{F}_n(\mathbf{x}_{n, t + 1})\|^2 
\leq (1 \alpha)^{2}\mathbb{E}\|\bar{\mathbf{u}}_{t} - \frac{1}{N} \sum_{n=1}^N \nabla \hat{F}_n(\mathbf{x}_{n, t})\|^{2} + \frac{2(1 -  \alpha)^2 S_F^2}{N^2b} \mathbb{E}\|\mathbf{x}_{t + 1} - \mathbf{x}_{t}\|^2 + \frac{2 \alpha^2 L_g^2 L_f^2}{Nb}    \nonumber\\
&\mathbb{E}\|u_{n, t + 1} - \nabla \hat{F}_n(\mathbf{x}_{n, t + 1})\|^2 \leq (1-\alpha)^{2}\mathbb{E}\|\mathbf{u}_{n, t} - \nabla \hat{F}(\mathbf{x}_{n, t})\|^{2} + \frac{2(1 -  \alpha)^2 S_F^2}{b} \mathbb{E}\|x_{n, t + 1} - \mathbf{x}_{n, t}\|^2 + \frac{2 \alpha^2 L_g^2 L_f^2}{b} 
\end{align}
\end{lemma}
\begin{proof}
Recall that $ \bar{\mathbf{u}}_{t + 1} = \frac{1}{N}\sum_{n=1}^{N} [ \frac{1}{b} \sum_{i=1}^{b} \nabla \hat{F}_n(\mathbf{x}_{n, t+1}; \xi^i_{n, t+1},\mathcal{B}_{n, t + 1}) + (1 - \alpha)(u_{n,t} -  \frac{1}{b} \sum_{i=1}^{b} \nabla \hat{F}_n(\mathbf{x}_{n,t}; \xi^i_{n, t + 1}, \mathcal{B}_{n, t + 1}))]$,\\ and $\nabla \hat{F}_n(x_{n,t}) = \mathbb{E}\nabla\hat{F}_n(x_{n,t}; \xi_{n,t}, \mathcal{B}_{n,t})$, we have
\begin{align}
&\mathbb{E}\|\bar{\mathbf{u}}_{t + 1} - \frac{1}{N} \sum_{n=1}^N \nabla \hat{F}_n(\mathbf{x}_{n, t + 1})\|^2 \nonumber\\
=& \mathbb{E}\|\frac{1}{N}\sum_{n=1}^{N}[\frac{1}{b} \sum_{i=1}^{b} \nabla \hat{F}_n(\mathbf{x}_{n, t+1}; \xi^i_{n, t + 1},\mathcal{B}_{n, t + 1}) + (1- \alpha)(u_{n,t} -  \frac{1}{b} \sum_{i=1}^{b} \nabla \hat{F}_n(\mathbf{x}_{n,t}; \xi^i_{n, t + 1},\mathcal{B}_{n, t + 1})) - \nabla \hat{F}_{n}(\mathbf{x}_{n, t+1}) ] \|^{2} \nonumber\\
=& \mathbb{E}\|\frac{1}{N}\sum_{i=1}^{N} [(\frac{1}{b} \sum_{i=1}^{b} \nabla \hat{F}_n(\mathbf{x}_{n, t+1}; \xi^i_{n, t + 1},\mathcal{B}_{n, t + 1}) - \nabla \hat{F}_n(\mathbf{x}_{n, t+1})) - (1 - \alpha)(\frac{1}{b} \sum_{i=1}^{b} \nabla \hat{F}_n(\mathbf{x}_{n,t}; \xi^i_{n, t + 1},\mathcal{B}_{n, t + 1}) - \nabla \hat{F}_n (\mathbf{x}_{n, t})) + (1 - \alpha)(\bar{\mathbf{u}}_{t} - \nabla \hat{F}(\mathbf{x}_{n, t})]\|^2 \nonumber\\
\stackrel{(a)}{=}& (1 - \alpha)^{2} \mathbb{E}\|\bar{\mathbf{u}}_{t} - \frac{1}{N} \sum_{n=1}^N \nabla \hat{F}_n(\mathbf{x}_{n, t})\|^{2} + \frac{1}{N^2} \sum_{i=1}^{N}\mathbb{E} \|(\frac{1}{b} \sum_{i=1}^{b} \nabla \hat{F}_n(\mathbf{x}_{n, t+1}; \xi^i_{n, t + 1},\mathcal{B}_{n, t + 1}) - \nabla \hat{F}_n(\mathbf{x}_{n, t+1})) \nonumber\\
&-(1- \alpha)(\frac{1}{b} \sum_{i=1}^{b} \nabla \hat{F}_n(\mathbf{x}_{n,t}; \xi^i_{n, t + 1},\mathcal{B}_{n, t + 1}) - \nabla \hat{F}_n(\mathbf{x}_{n, t}))\|^2 \nonumber\\
=& (1-\alpha)^{2}\mathbb{E}\|\bar{\mathbf{u}}_{t} -  \frac{1}{N} \sum_{n=1}^N \nabla \hat{F}_n(\mathbf{x}_{n, t}) \|^{2} + \frac{1}{N^2} \sum_{i=1}^{N} \mathbb{E}\|(1 - \alpha)[(\frac{1}{b} \sum_{i=1}^{b} \nabla \hat{F}_n(\mathbf{x}_{n, t+1}; \xi^i_{n, t + 1},\mathcal{B}_{n, t + 1}) - \nabla \hat{F}_n(\mathbf{x}_{n, t+1})) \nonumber\\
&- (\frac{1}{b} \sum_{i=1}^{b} \nabla \hat{F}_n(\mathbf{x}_{n,t}; \xi^i_{n, t + 1},\mathcal{B}_{n, t + 1}) - \nabla \hat{F}_n(\mathbf{x}_{n, t}))] +  \alpha(\frac{1}{b} \sum_{i=1}^{b} \nabla \hat{F}_n(\mathbf{x}_{n, t + 1}; \xi^i_{n, t + 1},\mathcal{B}_{n, t + 1}) - \nabla \hat{F}_n(\mathbf{x}_{n, t + 1})) \|^2 \nonumber\\
\leq& (1-\alpha)^{2}\mathbb{E}\|\bar{\mathbf{u}}_{t} - \frac{1}{N} \sum_{n=1}^N \nabla \hat{F}_n(\mathbf{x}_{n, t}) \|^{2} + \frac{2(1 -  \alpha)^2}{N^2} \sum_{i=1}^{N} \mathbb{E}\|(\frac{1}{b} \sum_{i=1}^{b} \nabla \hat{F}_n(\mathbf{x}_{n, t + 1}; \xi^i_{n, t + 1},\mathcal{B}_{n, t + 1}) - \nabla \hat{F}_n(\mathbf{x}_{n, t + 1})) \nonumber\\
&- (\frac{1}{b} \sum_{i=1}^{b} \nabla \hat{F}_n(\mathbf{x}_{n,t}; \xi^i_{n, t + 1},\mathcal{B}_{n, t + 1}) - \nabla \hat{F}_n(\mathbf{x}_{n, t + 1}))\|^2 +  \frac{2 \alpha^2}{N^2}\sum_{i=1}^{N} \mathbb{E}\|(\frac{1}{b} \sum_{i=1}^{b} \nabla \hat{F}_n(\mathbf{x}_{n, t + 1}; \xi^i_{n, t + 1},\mathcal{B}_{n, t + 1}) - \nabla \hat{F}_n(\mathbf{x}_{n, t + 1}))\|^2 \nonumber\\
\stackrel{(b)}{\leq}& (1-\alpha)^{2}\mathbb{E}\|\bar{\mathbf{u}}_{t} - \frac{1}{N} \sum_{n=1}^N \nabla \hat{F}_n(\mathbf{x}_{n, t}))\|^{2} +  \frac{2(1 -  \alpha)^2 S_F^2}{N^2b} \mathbb{E}\|\mathbf{x}_{t + 1} - \mathbf{x}_{t}\|^2 + \frac{2 \alpha^2 L_g^2 L_f^2}{Nb} 
\end{align}
where (a) holds due to $\mathbb{E} [(\frac{1}{b} \sum_{i=1}^{b} \nabla \hat{F}_n(\mathbf{x}_{n, t+1}; \xi^i_{n, t + 1},\mathcal{B}_{n, t + 1}) - \nabla \hat{F}_n(\mathbf{x}_{n, t+1})) - (1 - \alpha)(\frac{1}{b} \sum_{i=1}^{b} \nabla \hat{F}_n(\mathbf{x}_{n,t}; \xi^i_{n, t + 1},\mathcal{B}_{n, t + 1}) - \nabla \hat{F}_n (\mathbf{x}_{n, t}))] = 0$ and (b) is due to the  \cref{lem:A1} (b).Similarly, we have 
\begin{align}
& \mathbb{E}\|u_{n, t + 1} - \nabla \hat{F}_n(\mathbf{x}_{n, t + 1})\|^2 \nonumber\\
=& \mathbb{E}\|[\frac{1}{b} \sum_{i=1}^{b} \nabla \hat{F}_n(\mathbf{x}_{n, t + 1}; \xi^i_{n, t + 1},\mathcal{B}_{n, t + 1}) - \nabla \hat{F}_n(\mathbf{x}_{n, t + 1})] \nonumber\\
&- (1 -  \alpha)(\frac{1}{b} \sum_{i=1}^{b} \nabla \hat{F}_n(\mathbf{x}_{n, t + 1}; \xi^i_{n, t + 1},\mathcal{B}_{n, t + 1}) - \nabla \hat{F}_n (x_{n, t + 1})) + (1 - \alpha)(u_{n, t} - \nabla \hat{F}_n(\mathbf{x}_{n, t})\|^2 \nonumber\\
=& (1 - \alpha)^{2}\mathbb{E}\|u_{n, t} - \nabla \hat{F}_n (\mathbf{x}_{n, t})\|^{2} + \mathbb{E} \|(\frac{1}{b} \sum_{i=1}^{b} \nabla \hat{F}_n(\mathbf{x}_{n, t + 1}; \xi^i_{n, t + 1},\mathcal{B}_{n, t + 1}) - \nabla \hat{F}_n(\mathbf{x}_{n, t + 1})) \nonumber\\
& - (1 - \alpha)(\frac{1}{b} \sum_{i=1}^{b} \nabla \hat{F}_n(\mathbf{x}_{n,t}; \xi^i_{n, t + 1},\mathcal{B}_{n, t + 1}) - \nabla \hat{F}_n(\mathbf{x}_{n, t}))\|^2 \nonumber\\
\leq& (1 - \alpha)^{2}\mathbb{E}\|u_{n, t} - \nabla \hat{F}_n (\mathbf{x}_{n, t})\|^{2} + 2(1 - \alpha)^2 \mathbb{E}\|(\frac{1}{b} \sum_{i=1}^{b} \nabla \hat{F}_n(\mathbf{x}_{n, t + 1}; \xi^i_{n, t + 1},\mathcal{B}_{n, t + 1}) - \nabla \hat{F}_n(\mathbf{x}_{n, t + 1})) \nonumber\\
&- (\frac{1}{b} \sum_{i=1}^{b} \nabla \hat{F}_n(\mathbf{x}_{n,t}; \xi^i_{n, t + 1},\mathcal{B}_{n, t + 1}) - \nabla \hat{F}_n(\mathbf{x}_{n, t}))\|^2 +  2 \alpha^2 \mathbb{E}\|(\frac{1}{b} \sum_{i=1}^{b} \nabla \hat{F}_n(\mathbf{x}_{n, t + 1}; \xi^i_{n, t + 1},\mathcal{B}_{n, t + 1}) - \nabla \hat{F}_n(\mathbf{x}_{n, t + 1}))\|^2 \nonumber\\
\leq& (1-\alpha)^{2}\mathbb{E}\|\mathbf{u}_{n, t} - \nabla \hat{F}(\mathbf{x}_{n, t})\|^{2} + \frac{2(1 -  \alpha)^2 S_F^2}{b} \mathbb{E}\|x_{n, t + 1} - \mathbf{x}_{n, t}\|^2 + \frac{2 \alpha^2 L_g^2 L_f^2}{b} 
\end{align}
\end{proof}
\begin{lemma}
Suppose sequence $\mathbf{v}_{t}$ are generated by \Cref{alg:2} and if $0<\eta \leq \frac{1 - \lambda^2}{2 \sqrt{24} \lambda^2 S_F}$, we have
\begin{align}
&\mathbb{E}\left\|\mathbf{v}_{t+1} - \mathbf{J} \mathbf{v}_{t+1}\right\|^2
\leq  \frac{3 + \lambda^2}{4} \mathbb{E}\left\|\mathbf{v}_t - \mathbf{J} \mathbf{v}_t\right\|^2 + \frac{21 \lambda^2 N S_F^2 \eta^2}{1 - \lambda^2} \mathbb{E}\left\|\bar{\mathbf{u}}_{t}\right\|^2  \nonumber\\
&+\frac{63 \lambda^2 S_F^2}{1 - \lambda^2} \mathbb{E}\left\|\mathbf{x}_{t} - \mathbf{J} \mathbf{x}_{t}\right\|^2 + \frac{7 \lambda^2 \alpha^2}{1-\lambda^2} \mathbb{E}\left\|\mathbf{u}_{t} - \nabla \hat{\mathbf{F}}_{t} \right\|^2 + 3 \lambda^2 N \alpha^2 L_f^2 L_g^2 \nonumber
\end{align}
\begin{align}
\mathbb{E}\left\|\mathbf{v}_0 - \mathbf{J v}_0\right\|^2
\leq \lambda^2 \mathbb{E} \left\|\mathbf{u}_0 - \nabla \hat{\mathbf{F}}_0 \right\|^2 + \lambda^2 \mathbb{E} \left\| \nabla \hat{\mathbf{F}}_0 \right\|^2 \nonumber
\end{align}
\end{lemma}

\begin{proof}
Similar to \eqref{eq:22}, we have
\begin{align} \label{eq:49}
\left\|\mathbf{v}_{t+1} - \mathbf{J} \mathbf{v}_{t+1}\right\|^2 \leq \lambda^2\left\|\mathbf{v}_t - \mathbf{J} \mathbf{v}_t\right\|^2 + \lambda^2\left\|\mathbf{u}_{t + 1} - \mathbf{u}_{t}\right\|^2 +  2\left\langle (\mathbf{W} - \mathbf{J}) \mathbf{v}_t, (\mathbf{W} - \mathbf{J})\left(\mathbf{u}_{t + 1} - \mathbf{u}_{t}\right)\right\rangle
\end{align}
To bound the above terms, we recall the update of each local stochastic gradient estimator $\mathbf{u}_{n,t}$ in Algorithm \ref{alg:2}: $\forall t \geq 0$,
$$
u_{n, t + 1} = \frac{1}{b} \sum_{i=1}^{b} \nabla \hat{F}_n(\mathbf{x}_{n, t+1}; \xi^i_{n, t + 1},\mathcal{B}_{n, t + 1}) + (1-\alpha)(u_{n,t} -  \frac{1}{b} \sum_{i=1}^{b} \nabla \hat{F}_n(\mathbf{x}_{n, t}; \xi^i_{n, t + 1}, \mathcal{B}_{n, t + 1}))
$$
Firstly, we consider the second term in \eqref{eq:49}, We have that $\forall t \geq 0$ and $\forall n \in [N]$
\begin{align} \label{eq:50}
\mathbf{u}_{n, t + 1} - \mathbf{u}_{n,t} &= \frac{1}{b} \sum_{i=1}^{b} \nabla \hat{F}_n(\mathbf{x}_{n, t+1}; \xi^i_{n, t + 1},\mathcal{B}_{n, t + 1}) - \alpha u_{n,t} - (1-\alpha) \frac{1}{b} \sum_{i=1}^{b} \nabla \hat{F}_n(\mathbf{x}_{n,t}; \xi^i_{n, t + 1},\mathcal{B}_{n, t + 1})) \nonumber\\ 
&= \frac{1}{b} \sum_{i=1}^{b} \nabla \hat{F}_n(\mathbf{x}_{n, t+1}; \xi^i_{n, t + 1},\mathcal{B}_{n, t + 1}) - \frac{1}{b} \sum_{i=1}^{b} \nabla \hat{F}_n(\mathbf{x}_{n,t}; \xi^i_{n, t + 1},\mathcal{B}_{n, t + 1}) - \alpha (\mathbf{u}_{n,t} - \nabla \hat{F}_n(\mathbf{x}_{n,t})) \nonumber\\ 
&+ \alpha (\frac{1}{b} \sum_{i=1}^{b} \nabla \hat{F}_n(\mathbf{x}_{n,t}; \xi^i_{n, t + 1},\mathcal{B}_{n, t + 1}) - \nabla \hat{F}_n (\mathbf{x}_{n,t}))  
\end{align}

We take the expectation to obtain: $\forall t \geq 1$ and $\forall i \in [1, N]$,
\begin{align} \label{eq:51}
\mathbb{E}\left\|\mathbf{u}_{n, t + 1} - \mathbf{u}_{n,t}\right\|^2 \leq & 3 \mathbb{E}\|\frac{1}{b} \sum_{i=1}^{b} \nabla \hat{F}_n(\mathbf{x}_{n, t+1}; \xi^i_{n, t + 1},\mathcal{B}_{n, t + 1}) - \frac{1}{b} \sum_{i=1}^{b} \nabla \hat{F}_n(\mathbf{x}_{n,t}; \xi^i_{n, t + 1},\mathcal{B}_{n, t + 1}) \|^2 \nonumber \\
+ 3 \alpha^2 \mathbb{E}\| u_{n,t} &- \nabla \hat{F}_n(\mathbf{x}_{n,t}) \|^2 + 3 \alpha^2 \mathbb{E}\left\|\frac{1}{b} \sum_{i=1}^{b} \nabla \hat{F}_n(\mathbf{x}_{n,t}; \xi^i_{n, t + 1},\mathcal{B}_{n, t + 1}) - \nabla \hat{F}_n(\mathbf{x}_{n,t})\right\|^2 \nonumber\\
\leq & 3 S_F^2 \mathbb{E} \left\|\mathbf{x}_{n, t + 1} - \mathbf{x}_{n,t}\right\|^2 + 3 \alpha^2 \mathbb{E}\left\|u_{n,t} - \nabla \hat{F}_n(\mathbf{x}_{n,t}) \right\|^2 + 3 \alpha^2 L_g^2 L_f^2
\end{align}

Next, towards the last term in \eqref{eq:49}. Then we take expectation on the both sides of \eqref{eq:50}, we have that $\forall t \geq 1$,
\begin{align} \label{eq:52}
\mathbb{E}\left[\mathbf{u}_{t + 1} - \mathbf{u}_{t} \mid \mathcal{F}_{t + 1}\right] = \nabla \hat{\mathbf{F}}_{t + 1} - \nabla \hat{\mathbf{F}}_t - \alpha (\mathbf{u}_{t} - \nabla \hat{\mathbf{F}}_{t})
\end{align}
Then we discuss the bound of the third term with $\mathcal{F}_{t + 1}$-measurability of $\mathbf{v}_{t + 1}$, we have: $\forall t \geq 1$,
\begin{align} \label{eq:53}
& 2\left\langle (\mathbf{W} - \mathbf{J}) \mathbf{v}_t, (\mathbf{W}-\mathbf{J}) \mathbb{E}\left[\mathbf{u}_{t + 1} - \mathbf{u}_{t} \mid \mathcal{F}_{t + 1}\right]\right\rangle \nonumber\\
\stackrel{(a)}{=}& 2 \left\langle (\mathbf{W} - \mathbf{J}) \mathbf{v}_t, (\mathbf{W}-\mathbf{J}) \left(\nabla \hat{\mathbf{F}}_{t + 1} - \nabla \hat{\mathbf{F}}_{t} - \alpha (\mathbf{u}_{t} - \nabla \hat{\mathbf{F}}_{t})\right)\right\rangle  \nonumber\\
\stackrel{(b)}{\leq}& 2 \lambda\left\|\mathbf{v}_t - \mathbf{J} \mathbf{v}_t\right\| \cdot \lambda\left\|\nabla \hat{\mathbf{F}}_{t + 1} - \nabla \hat{\mathbf{F}}_{t} - \alpha (\mathbf{u}_{t} - \nabla \hat{\mathbf{F}}_{t})\right\| \nonumber\\
\stackrel{(c)}{\leq} & \frac{1 - \lambda^2}{2}\left\|\mathbf{v}_t - \mathbf{J} \mathbf{v}_t\right\|^2 + \frac{2 \lambda^4}{1-\lambda^2}\left\|\nabla \hat{\mathbf{F}}_{t + 1} - \nabla \hat{\mathbf{F}}_{t} - \alpha (\mathbf{u}_{t} - \nabla \hat{\mathbf{F}}_{t}) \right\|^2, \nonumber\\
\stackrel{(d)}{\leq} & \frac{1 - \lambda^2}{2}\left\|\mathbf{v}_t - \mathbf{J} \mathbf{v}_t\right\|^2 + \frac{4 \lambda^4 S_F^2}{1-\lambda^2}\|\mathbf{x}_{t + 1} - \mathbf{x}_{t} \|^2 + \frac{4 \lambda^4 \alpha^2}{1-\lambda^2}\| \mathbf{u}_{t} - \nabla \hat{\mathbf{F}}_{t} \|^2
\end{align}

where (a) uses \eqref{eq:52}, (b) is due to the Cauchy-Schwarz inequality and $\|\mathbf{W}-\mathbf{J}\|= \lambda$, (c) uses the elementary inequality that $2 a b \leq c_0 a^2 + b^2 / c_0$, with $c_0 = \frac{1-\lambda^2}{2 \lambda^2}$ for any $a, b \in \mathbb{R}$,  and $(d)$ holds since each $\hat{F}(x)$ is $L$-smooth. 

Putting \eqref{eq:51} and \eqref{eq:53} into to \eqref{eq:49} obtain: $\forall t \geq 1$,
\begin{align}
\mathbb{E}\left\|\mathbf{v}_{t+1} - \mathbf{J} \mathbf{v}_{t+1}\right\|^2 \leq & \frac{1 + \lambda^2}{2} \mathbb{E}\left\|\mathbf{v}_t - \mathbf{J} \mathbf{v}_t\right\|^2 + \frac{7 \lambda^2 S_F^2}{1 - \lambda^2} \mathbb{E}\left\|\mathbf{x}_{t + 1} - \mathbf{x}_{t}\right\|^2 + \frac{7 \lambda^2 \alpha^2}{1-\lambda^2} \mathbb{E}\left\|\mathbf{u}_{t} - \nabla \hat{\mathbf{F}}_{t} \right\|^2 + 3 \lambda^2 N \alpha^2 L_f^2 L_g^2 \nonumber\\
\stackrel{(a)}{\leq} &\left(\frac{1 + \lambda^2}{2} + \frac{42 \lambda^4 S_F^2 \eta^2}{1 - \lambda^2}\right) \mathbb{E}\left\|\mathbf{v}_t - \mathbf{J} \mathbf{v}_t\right\|^2 + \frac{21 \lambda^2 N S_F^2 \eta^2}{1 - \lambda^2} \mathbb{E}\left\|\bar{\mathbf{u}}_{t}\right\|^2  \nonumber\\
&+\frac{63 \lambda^2 S_F^2}{1 - \lambda^2} \mathbb{E}\left\|\mathbf{x}_{t} - \mathbf{J} \mathbf{x}_{t}\right\|^2 + \frac{7 \lambda^2 \alpha^2}{1-\lambda^2} \mathbb{E}\left\|\mathbf{u}_{t} - \nabla \hat{\mathbf{F}}_{t} \right\|^2 + 3 \lambda^2 N \alpha^2 L_f^2 L_g^2\nonumber\\
\stackrel{(b)}{\leq} & \frac{3 + \lambda^2}{4} \mathbb{E}\left\|\mathbf{v}_t - \mathbf{J} \mathbf{v}_t\right\|^2 + \frac{21 \lambda^2 N S_F^2 \eta^2}{1 - \lambda^2} \mathbb{E}\left\|\bar{\mathbf{u}}_{t}\right\|^2  \nonumber\\
&+\frac{63 \lambda^2 S_F^2}{1 - \lambda^2} \mathbb{E}\left\|\mathbf{x}_{t} - \mathbf{J} \mathbf{x}_{t}\right\|^2 + \frac{7 \lambda^2 \alpha^2}{1-\lambda^2} \mathbb{E}\left\|\mathbf{u}_{t} - \nabla \hat{\mathbf{F}}_{t} \right\|^2 + 3 \lambda^2 N \alpha^2 L_f^2 L_g^2
\end{align}
where (a) follows the \eqref{eq:24} and (b) holds due to $\frac{1 + \lambda^2}{2} + \frac{42 \lambda^4 S_F^2 \eta^2}{1 - \lambda^2} \leq \frac{3 + \lambda^2}{4}$ if $0 < \alpha \leq \frac{1 - \lambda^2}{2 \sqrt{42} \lambda^2 S_F}$. 
In addition, 
\begin{align} \mathbb{E}\left\|\mathbf{v}_0 - \mathbf{J v}_0\right\|^2
&=\mathbb{E} \left\|\mathbf{W}\left(\mathbf{u}_{0}\right) - \mathbf{J} \mathbf{W}\left( \mathbf{u}_{0} \right)\right\|^2  = \mathbb{E}\left\|(\mathbf{W}-\mathbf{J}) \mathbf{u}_1\right\|^2 \nonumber\\ 
& \leq\lambda^2 \mathbb{E} \left\|\mathbf{u}_0 - \nabla \hat{\mathbf{F}}(\mathbf{x}_0) + \nabla \hat{\mathbf{F}}(\mathbf{x}_0) \right\|^2 \nonumber\\ 
& \leq \lambda^2 \mathbb{E} \left\|\mathbf{u}_0 - \nabla \hat{\mathbf{F}}_0 \right\|^2 + \lambda^2 \mathbb{E} \left\| \nabla \hat{\mathbf{F}}_0 \right\|^2 \nonumber
\end{align}
\end{proof}
\subsection{proof of Theorem }
Then we start the proof of Theorem
\ref{thm:2}.
\begin{proof}
Recall that 
\begin{align}
&\mathbb{E}\|\bar{\mathbf{u}}_{t + 1} - \frac{1}{N} \sum_{n=1}^N \nabla \hat{F}(x_{n, t + 1})\|^2 \leq (1-\alpha)^{2}\mathbb{E}\|\bar{\mathbf{u}}_{t} - \frac{1}{N} \sum_{n=1}^N  \nabla \hat{F}(\mathbf{x}_{n, t})\|^{2} +  \frac{2(1 -  \alpha)^2 S_F^2}{N^2b} \mathbb{E}\|\mathbf{x}_{t + 1} - \mathbf{x}_{t}\|^2 + \frac{2 \alpha^2 L_g^2 L_f^2}{Nb} 
\end{align}
We know that $\frac{1}{1-(1-\alpha)^2} \leq \frac{1}{\alpha}$ for $\alpha \in (0, 1)$. Based on \Cref{lem:A3}, we have: $\forall T \geq 2$,
\begin{align} \label{eq:61}
&\sum_{t=0}^{T-1} \mathbb{E}\left\|\bar{\mathbf{u}}_{t} - \frac{1}{N} \sum_{n=1}^N \nabla \hat{F}(x_{n, 0})\right\|^2 \\
\leq& \frac{\mathbb{E}\left\|\bar{\mathbf{u}}_{0} - \frac{1}{N} \sum_{n=1}^N \nabla \hat{F}(x_{n, 0})\right\|^2}{\alpha} + \sum_{t = 0}^{T - 2} \frac{2 S_F^2}{N^2 \alpha b} \mathbb{E}\|\mathbf{x}_{t + 1} - \mathbf{x}_{t}\|^2 + \frac{2 \alpha L_g^2 L_f^2}{Nb} T \nonumber \\
\leq& \frac{\mathbb{E}\left\|\bar{\mathbf{u}}_{0} - \frac{1}{N} \sum_{n=1}^N \nabla \hat{F}(x_{n, 0})\right\|^2}{\alpha} + \frac{6 S_F^2}{N^2 \alpha b} \sum_{t = 0}^{T - 2} [ \mathbb{E}\|\mathbf{x}_{t + 1} - \mathbf{Jx}_{t + 1}\|^2 + \|\mathbf{Jx}_{t + 1} - \mathbf{Jx}_{t} \|^2 + \| \mathbf{x}_t - \mathbf{J} \mathbf{x}_{t}\|^2 ] + \frac{2 \alpha L_g^2 L_f^2}{Nb} T \nonumber \\
\leq& \frac{\mathbb{E}\left\|\bar{\mathbf{u}}_{0} - \frac{1}{N} \sum_{n=1}^N \nabla \hat{F}(x_{n, 0})\right\|^2}{\alpha} + \frac{12 S_F^2}{N^2 \alpha b} \sum_{t=0}^{T - 1}  \mathbb{E}\|\mathbf{x}_{t} - \mathbf{Jx}_{t}\|^2 +  \frac{6 \eta^2  S_F^2 }{N \alpha b} \sum_{t=0}^{T - 1}  \| \bar{\mathbf{u}}_t \|^2 + \frac{2 \alpha L_g^2 L_f^2}{Nb} T  \nonumber\\
\leq& \frac{L_g^2 L_f^2 }{\alpha N B} + \frac{12 S_F^2}{N^2 \alpha b} \sum_{t = 0}^{T - 1}  \mathbb{E}\|\mathbf{x}_{t} - \mathbf{Jx}_{t}\|^2 +  \frac{6 \eta^2  S_F^2 }{N \alpha b} \sum_{t = 0}^{T - 1}  \| \bar{\mathbf{u}}_t \|^2 + \frac{2 \alpha L_g^2 L_f^2}{Nb} T \nonumber
\end{align}
where $B$ is the initial batch size. 
Recall that 
\begin{align}
\mathbb{E}\|u_{n, t + 1} - \nabla \hat{F}_n(\mathbf{x}_{n, t + 1})\|^2 \leq (1-\alpha)^{2}\mathbb{E}\|\mathbf{u}_{n, t} - \nabla \hat{F}(\mathbf{x}_{n, t})\|^{2} + \frac{2(1 -  \alpha)^2 S_F^2}{b} \mathbb{E}\|x_{n, t + 1} - \mathbf{x}_{n, t}\|^2 + \frac{2 \alpha^2 L_g^2 L_f^2}{b}  
\end{align}
Similarly, we have the following: $\forall T \geq 2$,
\begin{align} \label{eq:63}
&\sum_{n=1}^{N} \sum_{t=0}^{T - 1} \mathbb{E}\left\|\mathbf{u}_{n, t} - \nabla \hat{F}_n(\mathbf{x}_{n, t})\right\|^2 \\
\leq & \sum_{n=1}^{N}\frac{\mathbb{E}\left\|\mathbf{u}_{n, 0} - \nabla \hat{F}_n(x_{n, 0 })\right\|^2}{\alpha} + \frac{2 S_F^2}{\alpha b} \sum_{t = 0}^{T - 2} \mathbb{E}\|x_{t + 1} - x_{t}\|^2 + \frac{2 \alpha N L_g^2 L_f^2}{b } T \nonumber\\
\leq & \sum_{n=1}^{N} \frac{\mathbb{E}\left\|\mathbf{u}_{n, 0} - \nabla \hat{F}_n (x_{n, 0})\right\|^2}{\alpha} + \frac{6 N S_F^2 \eta^2}{\alpha} \sum_{t=1}^{T-1} \mathbb{E}\left\|\bar{\mathbf{u}}_t \right\|^2 + \frac{12 S_F^2}{\alpha b} \sum_{t=1}^T \mathbb{E} \left\|\mathbf{x}_t-\mathbf{J} \mathbf{x}_t\right\|^2  + \frac{2 \alpha N L_g^2 L_f^2}{b } T \nonumber\\
\leq & \frac{N L_f^2 L_g^2}{\alpha B} + \frac{6 N S_F^2 \eta^2}{\alpha} \sum_{t=0}^{T-2} \mathbb{E}\left\|\bar{\mathbf{u}}_t \right\|^2 + \frac{12 S_F^2}{\alpha b} \sum_{t=0}^{T- 1} \mathbb{E} \left\|\mathbf{x}_t-\mathbf{J} \mathbf{x}_t\right\|^2  + \frac{2 \alpha N L_g^2 L_f^2}{b } T  \nonumber
\end{align}

where the last inequality follows that 
\begin{align}
\mathbb{E}\left\|\mathbf{u}_{n,0} - \nabla \hat{F}(\mathbf{x}_{n,0})\right\|^2 &= \mathbb{E} \left\|\frac{1}{B} \sum_{i=1}^{B} \nabla \hat{F}_n(\mathbf{x}_{n, 0}; \xi^i_{n, 0},\mathcal{B}_{n, 0})) - \nabla \hat{F}_{n}(\mathbf{x}_{n, 0})\right\|^2  \\
& \stackrel{(a)}{=} \frac{1}{B^2} \sum_{i=1}^{B} \mathbb{E}\left\| \nabla \hat{F}_n(\mathbf{x}_{n, 0}; \xi^i_{n, 0},\mathcal{B}_{n, 0})) - \nabla \hat{F}_{n}(\mathbf{x}_{n, 0})\right\|^2  \leq \frac{L_g^2 L_f^2}{B},
\end{align}

where $(a)$ follows from $\mathbb{E}[\frac{1}{b} \sum_{i=1}^{b} \nabla \hat{F}_n(\mathbf{x}_{n, 0}; \xi^i_{n, 0},\mathcal{B}_{n, 0})) - \nabla \hat{F}_{n}(\mathbf{x}_{n, 0})] = 0$. 


To further bound $\sum_{t=0}^{T - 1}\left\|\mathbf{v}_t - \mathbf{J y}_t\right\|^2$, we obtain: if $0<\alpha \leq \frac{1 - \lambda^2}{2 \sqrt{24} \lambda^2 S_F}$, then $\forall T \geq 2$,
\begin{align}
& \sum_{t=0}^{T - 1} \mathbb{E}\left[\left\|\mathbf{v}_t - \mathbf{J} \mathbf{v}_t\right\|^2\right] \nonumber\\
\leq & \frac{4 \mathbb{E} \left\|\mathbf{v}_0 - \mathbf{J} \mathbf{v}_0 \right\|^2 }{1 - \lambda^2} + \frac{84 \lambda^2 N S_F^2 \eta^2}{\left(1 - \lambda^2\right)^2} \sum_{t = 0}^{T - 2} \mathbb{E}\left\|\bar{\mathbf{u}}_t\right\|^2 +\frac{252 \lambda^2 S_F^2}{\left(1 - \lambda^2\right)^2} \sum_{t = 0}^{T - 2} \mathbb{E} \left\|\mathbf{x}_t-\mathbf{J} \mathbf{x}_t\right\|^2 + \frac{28 \lambda^2 \alpha^2}{\left(1 - \lambda^2\right)^2} \sum_{t = 1}^{T - 1} \mathbb{E} \left\|\mathbf{u}_t - \nabla \hat{\mathbf{F}}\left(\mathbf{x}_t\right)\right\|^2 + \frac{12 \lambda^2 N \alpha^2 L_f^2 L_g^2 T}{1 - \lambda^2} \nonumber\\
\leq & \frac{4 \lambda^2 \mathbb{E}\left\|\nabla \hat{\mathbf{F}}_0\right\|^2}{1 - \lambda^2} + \frac{4 \lambda^2 N L_f^2 L_g^2}{\left(1 - \lambda^2\right) B} + \frac{84 \lambda^2 N S_F^2 \eta^2}{\left(1 - \lambda^2\right)^2} \sum_{t = 0}^{T - 1} \mathbb{E} \left\|\bar{\mathbf{u}}_t\right\|^2 + \frac{252 \lambda^2 S_F^2}{\left(1 - \lambda^2\right)^2} \sum_{t = 0}^{T - 1} \mathbb{E}\left\|\mathbf{x}_t - \mathbf{J} \mathbf{x}_t\right\|^2  + \frac{12 \lambda^2 N \alpha^2 L_f^2 L_g^2 T}{1 - \lambda^2}  + \frac{28 \lambda^2 \alpha^2}{\left(1 - \lambda^2\right)^2} \sum_{t = 0}^{T - 1} \sum_{n=1}^{N} \mathbb{E} \left\|\mathbf{u}_{n, t} - \nabla \hat{\mathbf{F}}_n\left(\mathbf{x}_{n, t}\right)\right\|^2 \nonumber
\end{align}

Furthermore, we use \eqref{eq:63} and if $0<\eta \leq \frac{1-\lambda^2}{2 \sqrt{42} \lambda^2 L}$ and $\alpha \in(0,1)$, then $\forall T \geq 2$,
\begin{align}
\sum_{t=0}^{T-1} \mathbb{E} \left\|\mathbf{v}_t - \mathbf{J} \mathbf{v}_t \right\|^2 \leq & \frac{252 \lambda^2 N S_F^2 \eta^2}{\left(1 - \lambda^2\right)^2} \sum_{t = 0}^{T - 1} \mathbb{E} \left\|\bar{\mathbf{u}}_t\right\|^2 + \frac{588 \lambda^2 S_F^2}{\left(1 - \lambda^2\right)^2} \sum_{t = 0}^{T-1} \mathbb{E}\left[\left\|\mathbf{x}_t-\mathbf{J} \mathbf{x}_t\right\|^2\right] \nonumber\\
&+ \frac{28 \lambda^2 N \alpha L_g^2 L_f^2}{\left(1 - \lambda^2\right)^2 B} + \frac{56 \lambda^2 N \alpha^3 L_f^2 L_F^2 T}{\left(1 - \lambda^2\right)^2} + \frac{12 \lambda^2 N \alpha^2 L_f^2 L_g^2 T}{1 - \lambda^2} + \frac{4 \lambda^2 \mathbb{E} \left\|\nabla \hat{\mathbf{F}}_0\right\|^2}{1 - \lambda^2} + \frac{4 \lambda^2 N L_f^2 L_g^2}{\left(1 - \lambda^2\right) B}
\end{align}

Finally, we use \Cref{lem:A3} in \eqref{eq:16} to obtain: $\forall T \geq 2$,
\begin{align}
\sum_{t = 0}^{T-1} \mathbb{E} \left\|\mathbf{x}_t - \mathbf{J} \mathbf{x}_t\right\|^2 \leq& \frac{4 \lambda^2 \eta^2}{(1 - \lambda^2)^2} \sum_{t = 0}^{T-2} \|\mathbf{v}_t - \mathbf{J} \mathbf{y}_t\|\nonumber\\
\leq & \frac{1008 \lambda^4 N S_F^2 \eta^4}{\left(1 - \lambda^2\right)^4} \sum_{t = 0}^{T - 1} \mathbb{E} \left\|\bar{\mathbf{u}}_t\right\|^2 + \frac{2352 \lambda^4 S_F^2 \eta^2}{\left(1 - \lambda^2\right)^4} \sum_{t = 0}^{T-1} \mathbb{E}\left\|\mathbf{x}_t-\mathbf{J} \mathbf{x}_t\right\|^2 \nonumber \\
&+ (\frac{7\alpha}{1 - \lambda^2} + 1)\frac{16 \lambda^4 N L_g^2 L_f^2 \eta^2}{\left(1 - \lambda^2\right)^3 B} + (\frac{14 \alpha}{1 - \lambda^2} + 3) \frac{16 \lambda^2 N \alpha^2 L_f^2 L_g^2 T \eta^2}{(1 - \lambda^2)^3} + \frac{16 \lambda^4 \eta^2 \left\|\nabla \hat{\mathbf{F}}_0 \right\|^2}{(1 - \lambda^2)^3}
\end{align}

which may be written equivalently as
\begin{align} \label{eq:68}
\left(1 - \frac{2352 \lambda^4 S_F^2 \eta^2}{\left(1 - \lambda^2\right)^4}\right) \sum_{t = 1}^T \mathbb{E}\left[\left\|\mathbf{x}_t-\mathbf{J} \mathbf{x}_t\right\|^2\right] \leq & \frac{1008 \lambda^4 N S_F^2 \eta^4}{\left(1 - \lambda^2\right)^4} \sum_{t = 0}^{T - 1} \mathbb{E} \left\|\bar{\mathbf{u}}_t\right\|^2 + (\frac{7\alpha}{1 - \lambda^2} + 1)\frac{16 \lambda^4 N L_g^2 L_f^2 \eta^2}{\left(1 - \lambda^2 \right)^3 B} \nonumber\\
&+ (\frac{14 \alpha}{1 - \lambda^2} + 3) \frac{16 \lambda^2 N \alpha^2 L_f^2 L_g^2 T \eta^2}{(1 - \lambda^2)^3} + \frac{16 \lambda^4 \eta^2 \mathbb{E} \left\|\nabla \hat{\mathbf{F}}\left(\mathbf{x}_0\right)\right\|^2}{(1 - \lambda^2)^3}
\end{align}

We observe in \eqref{eq:68} that $\frac{2352 \lambda^4 S_F^2 \alpha^2}{\left(1-\lambda^2\right)^4} \leq \frac{1}{2}$ if $0 < \eta \leq \frac{\left(1 - \lambda^2\right)^2}{90 \lambda^2 S_F}$. 
Based on \Cref{lem:C1}, we have
\begin{align}
\frac{1}{T}\sum_{t = 0}^{T - 1} \mathbb{E} \| \nabla \mathbf{F} (\bar{\mathbf{x}}_{t})\|^2 
\leq& \frac{2(\mathbf{F} (\bar{\mathbf{x}}_{0}) - \mathbf{F} (\bar{\mathbf{x}}_{T}))}{\eta T} - (1 - \eta S_F ) \frac{1}{T} \sum_{t = 0}^{T - 1}\|\bar{\mathbf{u}}_{t} \|^2  + \frac{3 S_F^2}{N T} \sum_{t = 0}^{T - 1} \sum_{n=1}^N\|\mathbf{x}_{n, t} - \bar{\mathbf{x}}_t\|^2 \nonumber\\
&+ \frac{3 L_g^2 S_f^2 \sigma_g^2}{m} + 3 \frac{1}{T} \sum_{t = 0}^{T - 1}\|\frac{1}{N} \sum_{n=1}^N \nabla \hat{F}_n(\mathbf{x}_{n, t}) - \bar{\mathbf{u}}_{t}\|^2 \nonumber\\
&\stackrel{(a)}{\leq }\frac{2(\mathbf{F}(\bar{\mathbf{x}}_{0}) - \mathbf{F} ( \bar{\mathbf{x}}_{T}))}{\eta T} - (1 - \eta S_F) \frac{1}{T} \sum_{t = 0}^{T - 1}\|\bar{\mathbf{u}}_{t} \|^2  + \frac{3 S_F^2}{NT} \sum_{t = 0}^{T - 1} \sum_{n=1}^N \|\mathbf{x}_{n, t} - \bar{\mathbf{x}}_t \|^2 \nonumber\\
&+ \frac{3 L_g^2 S_f^2 \sigma_g^2}{m} + 3 \frac{L_g^2 L_f^2 }{\alpha N B T} + \frac{36 S_F^2}{N^2 \alpha b T} \sum_{t = 0}^{T - 1} \mathbb{E}\|\mathbf{x}_{t} - \mathbf{Jx}_{t}\|^2 + \frac{18 \eta^2  S_F^2 }{N \alpha b T} \sum_{t = 0}^{T - 1}  \| \bar{\mathbf{u}}_t \|^2 + \frac{6 \alpha L_g^2 L_f^2}{Nb} \nonumber\\
\stackrel{(b)}{\leq}& \frac{2(\mathbf{F}(\bar{\mathbf{x}}_{0}) - \mathbf{F} ( \bar{\mathbf{x}}_{T}))}{\eta T} - (1 - \eta S_F - \frac{18 \eta^2  S_F^2 }{N \alpha b} ) \frac{1}{T}\sum_{t = 0}^{T - 1}\|\bar{\mathbf{u}}_{t} \|^2  + \frac{2}{N \eta^2 T} \sum_{t = 0}^{T - 1}  \mathbb{E}\|\mathbf{x}_{t} - \mathbf{Jx}_{t}\|^2 \nonumber\\
&+ \frac{3 L_g^2 S_f^2 \sigma_g^2}{m} + 3 \frac{L_g^2 L_f^2 }{\alpha N B T}  + \frac{6 \alpha L_g^2 L_f^2}{Nb} \nonumber\\
\stackrel{(c)}{\leq}& \frac{2(\mathbf{F}(\bar{\mathbf{x}}_{0}) - \mathbf{F} ( \bar{\mathbf{x}}_{T}))}{\eta T} - (1 - \eta S_F - \frac{18 \eta^2  S_F^2 }{N \alpha b}) \frac{1}{T} \sum_{t = 0}^{T - 1}\|\bar{\mathbf{u}}_{t} \|^2   + \frac{3 L_g^2 S_f^2 \sigma_g^2}{m} + 3 \frac{L_g^2 L_f^2 }{\alpha N B T}  + \frac{6 \alpha L_g^2 L_f^2}{Nb} \nonumber\\
&+ \frac{4032 \lambda^4 S_F^2 \eta^2}{\left(1 - \lambda^2\right)^4 T} \sum_{t = 0}^{T - 1} \mathbb{E} \left\|\bar{\mathbf{u}}_t\right\|^2 + (\frac{7\alpha}{1 - \lambda^2} + 1)\frac{64 \lambda^4 L_g^2 L_f^2}{\left(1 - \lambda^2 \right)^3 B T} + (\frac{14 \alpha}{1 - \lambda^2} + 3) \frac{64 \lambda^2 \alpha^2 L_f^2 L_g^2}{(1 - \lambda^2)^3} + \frac{64 \lambda^4 \mathbb{E} \left\|\nabla \hat{\mathbf{F}}_0\right\|^2}{(1 - \lambda^2)^3 N T} \nonumber\\
=& \frac{2(\mathbf{F}(\bar{\mathbf{x}}_{0}) - \mathbf{F} ( \bar{\mathbf{x}}_{T}))}{\eta T} - (1 - \eta S_F - \frac{18 \eta^2  S_F^2 }{N \alpha b} - \frac{4032 \lambda^4 S_F^2 \eta^2}{\left(1 - \lambda^2\right)^4} ) \frac{1}{T} \sum_{t = 0}^{T - 1}\|\bar{\mathbf{u}}_{t} \|^2   + \frac{3 L_g^2 S_f^2 \sigma_g^2}{m} + 3 \frac{L_g^2 L_f^2 }{\alpha N B T}  + \frac{6 \alpha L_g^2 L_f^2}{Nb} \nonumber\\
&+ (\frac{7\alpha}{1 - \lambda^2} + 1)\frac{64 \lambda^4 L_g^2 L_f^2}{\left(1 - \lambda^2 \right)^3 B T} + (\frac{14 \alpha}{1 - \lambda^2} + 3) \frac{64 \lambda^2 \alpha^2 L_f^2 L_g^2}{(1 - \lambda^2)^3} + \frac{64 \lambda^4 \mathbb{E} \left\|\nabla \hat{\mathbf{F}}_0\right\|^2}{(1 - \lambda^2)^3 N T} \nonumber\\
\stackrel{(d)}{\leq}&  \frac{2(\mathbf{F}(\bar{\mathbf{x}}_{0}) - \mathbf{F} ( \bar{\mathbf{x}}_{T}))}{\eta T}   + \frac{3 L_g^2 S_f^2 \sigma_g^2}{m} + 3 \frac{L_g^2 L_f^2 }{\alpha N B T}  + \frac{6 \alpha L_g^2 L_f^2}{Nb} + \frac{96 \lambda^2 L_g^2 L_f^2}{\left(1 - \lambda^2 \right)^3 B T} + \frac{256 \lambda^2 \alpha^2 L_f^2 L_g^2}{(1 - \lambda^2)^3} + \frac{64 \lambda^4 \mathbb{E} \left\|\nabla \hat{\mathbf{F}}_0\right\|^2}{(1 - \lambda^2)^3 N T}  \nonumber
\end{align}
where (a) holds due to \eqref{eq:61}; (b) uses the $\alpha = \frac{72 S^2_F \eta^2 }{Nb}$ and $S_F \eta \leq \frac{1}{4} \leq \frac{1}{2}$; (c) follows the \eqref{eq:68} and (d) holds due to the fact that $1 - \eta S_F - \frac{18 \eta^2  S_F^2 }{N \alpha b} - \frac{4032 \lambda^4 S_F^2 \eta^2}{\left(1 - \lambda^2\right)^4} \geq 0$ if $0 < \eta \leq min \{\frac{1}{4}, \frac{\left(1 - \lambda^2\right)^2}{90 \lambda^2} \} \frac{1}{S_F}$, and $\alpha = \frac{72 S^2_F \eta^2 }{Nb} \leq \frac{1 - \lambda^2}{14 \lambda^2}$ if $\eta \leq \frac{\sqrt{1 - \lambda^2}}{12 \sqrt{7} \lambda S_F}$
\end{proof}

Then, we choose $b = O(1), \eta = O(\frac{N^{2/3}}{T^{1/3}}), \alpha = \frac{N^{1/3}}{T^{2/3}}, B = \frac{T^{1/3}}{N^{2/3}} $
\begin{align}
&\frac{1}{T}\sum_{t = 0}^{T - 1} \mathbb{E} \| \nabla \mathbf{F} (\bar{\mathbf{x}}_{t})\|^2 \leq O(\frac{2(\mathbf{F}(\bar{\mathbf{x}}_{0}) - \mathbf{F} ( \bar{\mathbf{x}}_{T}))}{(N T)^{2/3}} + \frac{3 L_g^2 S_f^2 \sigma_g^2}{m} \nonumber\\
&+ O(\frac{3 L_g^2 L_f^2 }{(NT)^{2/3}})  + O(\frac{6 L_g^2 L_f^2}{(NT)^{2/3}}) + \frac{352 \lambda^2 L_f^2 L_g^2}{(1 - \lambda^2)^3} O(\frac{N^{2/3}}{T^{4/3}}) + \frac{64 \lambda^4 \mathbb{E} \left\|\nabla \hat{\mathbf{F}}_0\right\|^2}{(1 - \lambda^2)^3 N T}  \nonumber
\end{align}

\end{document}